\documentclass[prodmode,acmtecs]{acmsmall} 

\usepackage[ruled]{algorithm2e}
\renewcommand{\algorithmcfname}{ALGORITHM}
\SetAlFnt{\small}
\SetAlCapFnt{\small}
\SetAlCapNameFnt{\small}
\SetAlCapHSkip{0pt}
\IncMargin{-\parindent}

\usepackage{multirow}
\usepackage{rotating}
\usepackage{enumerate}
\usepackage{amsfonts}
\usepackage{amsmath}
\usepackage{algorithmic}
\usepackage{amssymb}
\usepackage{graphicx}
\usepackage{subfigure}

\usepackage{comment}

\usepackage{color}
\usepackage{float}
\usepackage{comment}
\usepackage{flushend}
\usepackage{hyperref}

\usepackage{xr-hyper}
\usepackage{zref-base}
\usepackage{atveryend}

\newcommand{\jn}[1]{\textcolor{red}{JN: #1}}

\usepackage{subfigure}
\usepackage{float}
\usepackage{flushend}
\usepackage{amsmath}

\usepackage{multirow}
\usepackage{rotating}
\usepackage{enumerate}
\usepackage{amsfonts}
\usepackage{amssymb}
\usepackage{graphicx}
\usepackage{subfigure}

\usepackage{subfigure}
\usepackage{caption,setspace}

\usepackage{pdfpages}

\newtheorem{claim}{Claim}

\begin{document}





\author{TIMOTHY LA FOND
\affil{Purdue University}
JENNIFER NEVILLE
\affil{Purdue University}
BRIAN GALLAGHER
\affil{Lawrence Livermore National Laboratory}
}

\title{Size-Consistent Statistics for Anomaly Detection in Dynamic Networks}

\maketitle


\begin{abstract}

An important task in network analysis is the detection of anomalous events in a network time series.  These events could merely be times of interest in the network timeline or they could be examples of malicious activity or network malfunction.  Hypothesis testing using network statistics to summarize the behavior of the network provides a robust framework for the anomaly detection decision process.  Unfortunately, choosing network statistics that are dependent on confounding factors like the total number of nodes or edges can lead to incorrect conclusions (e.g., false positives and false negatives).  In this dissertation we describe the challenges that face anomaly detection in dynamic network streams regarding confounding factors.  We also provide two solutions to avoiding error due to confounding factors: the first is a randomization testing method that controls for confounding factors, and the second is a set of size-consistent network statistics which avoid confounding due to the most common factors, edge count and node count.  

\end{abstract}



{\def\arraystretch{1.5}\tabcolsep=2pt
\begin{table}
\begin{footnotesize}
\caption{{Glossary of terms}}
\label{glossary}
\begin{tabular}{| l | l |}
\hline
$G_t$ & Observed graph at time $t$ \\
$V_t$ & Set of vertices in graph $G_t$, size $N$ \\
$W_t$ & Weighted adjacency matrix of $G_t$ \\
$| W_t |$ & Total weight of $W_t$ \\
$P^*_t$ & True distribution of edge weights in the underlying model, size $|V^*| x N^*$ \\
$A^*_t$ & Adjacency matrix of $P^*_t$; i.e. $a_{ij,t} = I[ p^*_{ij,t} > 0 ]$ \\
$V^*$ & True vertex set of underlying model, $V_t \subset V^*$ \\
$P_t$ & Renormalized distribution of edge weight on vertex set $V_t$, used to sample $G_t$ \\
$A_t$ & Adjacency matrix of $P_t$; i.e. $a_{ij,t} = I[ p_{ij,t} > 0 ]$ \\
$| A_t |$ & Number of nonzero cells in adjacency matrix \\
$\widehat{P}_t$ & Approximate distribution of edge weights estimated from $G_t$: $\widehat{p}_{ij,t} = \frac{w_{ij,t}}{| W_t |_1}$\\
$\widehat{A}_t$ & Adjacency matrix of $G_t$; i.e. $a_{ij,t} = I[ w_{ij,t} > 0 ] $\\
$\overline{p^*}_t$ & Mean value of any nonzero cell in $P^*_t$ \\
$\overline{p}_t$ & Mean value of any nonzero cell in $P_t$ \\
$\overline{\widehat{p}}_t$ & Mean value of any nonzero cell in $\widehat{P}_t$ \\
$\overline{p^*}_t \bigg| V_t$ & Mean value of the $P^*_t$ cells that belong to vertex subset $V_t$ \\
$w_{row_i,2}$ & Total weight in row $i$ of $W_t$ \\
$\overline{p^*}_{row,t}$ & Expected mass in any row of $P^*_t$ \\
$\overline{p}_{row,t}$ & Expected mass in any row of $P_t$ \\
$\overline{\widehat{p}}_{row,t}$ & Expected mass in any row of $\widehat{P}_t$ \\
$\overline{p^*}_{row,t} | V_t$ & Expected mass of rows in $P^*_t$, excluding any rows or cells that do not belong to $V_t$\\
\hline 
\end{tabular}
\end{footnotesize}

\end{table}

}

\section{Introduction}



In this paper, we will focus on the task of anomaly detection in a dynamic network where the structure of the network is changing over time.  For example, each time step could represent one day's worth of activity on an e-mail network or communications of a computer network.  The goal is then to identify any time steps where the pattern of those communications seems abnormal compared to those of other time steps.  


We will be approaching this problem as a hypothesis testing task - the null hypothesis is that a time step under scrutiny represents normal behavior of the network while the alternative hypothesis is that it is anomalous.  The null distribution will be constructed from graph examples observed in the past, and the test statistics will be various network statistics.  Whenever the null hypothesis is rejected for a time step, we will flag the tested time step as an anomaly.

A typical real-world network experiences many changes in the course of its natural behavior, changes which are not examples of anomalous events.  The most common of these is variation in the volume of edges.  In the case of an e-mail network where the edges represent messages, the network could be growing in size over time or there could be random variance in the number of messages sent each day.  The statistics used to measure the network properties are usually intended to capture some other effect of the network than simply the volume of edges: for example, the common clustering coefficient is a measure of transitivity which is the propensity for triangular interactions in the network. 
However, statistics such as the clustering coefficient are \textit{Statistically Inconsistent} as the size of the network changes - more or fewer edges/nodes change the output of the statistic even when the transitivity property is constant making graph size a confounding factor.  Statistical consistency and inconsistency are described in more detail in Section 6.3.  Even on an Erd{\"o}s-R{\'e}nyi network, which does not explicitly capture transitive relationships through a network property, the clustering coefficient will be greater as the number of edges in the network increases as more triangles will be closed due to random chance.  When statistics vary with the number of edges in the network, it is not valid to compare different network time steps using those statistics unless the number of edges is constant in each time step.  The flowchart in Figure \ref{control-statistic} outlines the detection approach: unless the statistic is carefully defined to be robust to confounding factors, it is impossible to determine which factor that generated the graph is responsible for detected anomalies.

Table \ref{glossary} shows a glossary of terms that will be used throughout this Chapter.  Some, like the terms $G_t$, $V_t$, and $W_t$, are from the dynamic graph definitions used previously.  The other terms will be explained as they are used throughout the Chapter. 


Figure \ref{fig:compare} shows the effect of statistical (in)consistency.  During the experiment pairs of graphs were generated using a Chung-Lu generative model (described in Section 6.6) with a certain number of total edges.  Subfigure (a) shows the values of a \textit{Size Consistent Statistic} called \textit{Probability Mass Shift} (described in Section 6.4) calculated on pairs of graphs, while Subfigure (b) shows the same for the Netsimile statistic (described in a previous Chapter).  Each black point shows the average value of 100 generated graph pairs while the red points are the minimum and maximum of these pairs.  As the edge weight increases (x-axis) the statically consistent Mass Shift (\ref{fig:compare}a) maintains a consistent mean, whereas the statistically inconsistent Netsimile (\ref{fig:compare}b) varies wildly, even though all graphs are generated from the same underlying model (Chung-Lu \cite{chunglu} with a power law degree distribution). 

\begin{figure}[h]
\begin{center}
\subfigure[]{\includegraphics[trim=0 15 0 55, clip=true, width=.39\columnwidth,natwidth=610,natheight=642]{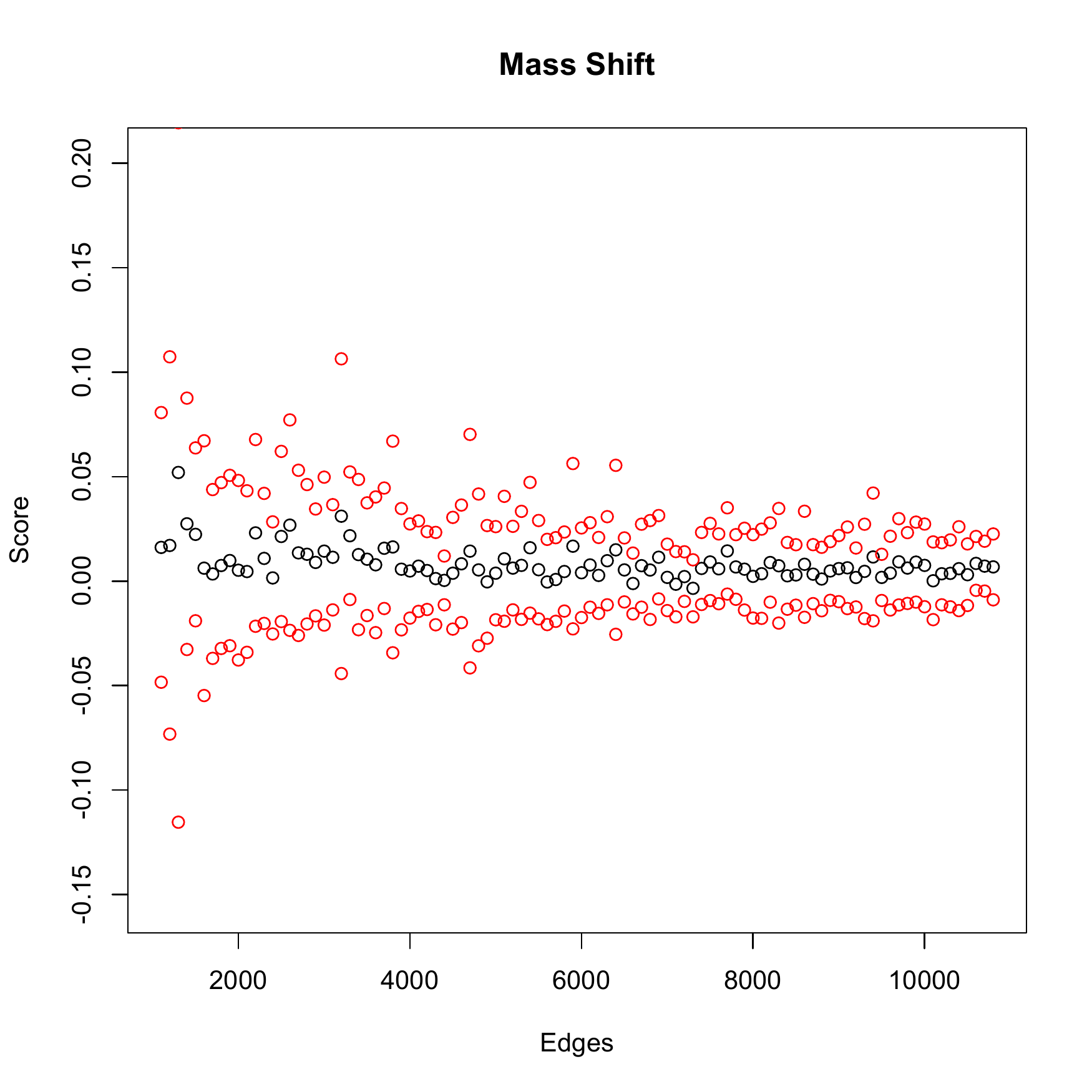}}
\subfigure[]{\includegraphics[trim=0 15 0 55, clip=true, width=.39\columnwidth,natwidth=610,natheight=642]{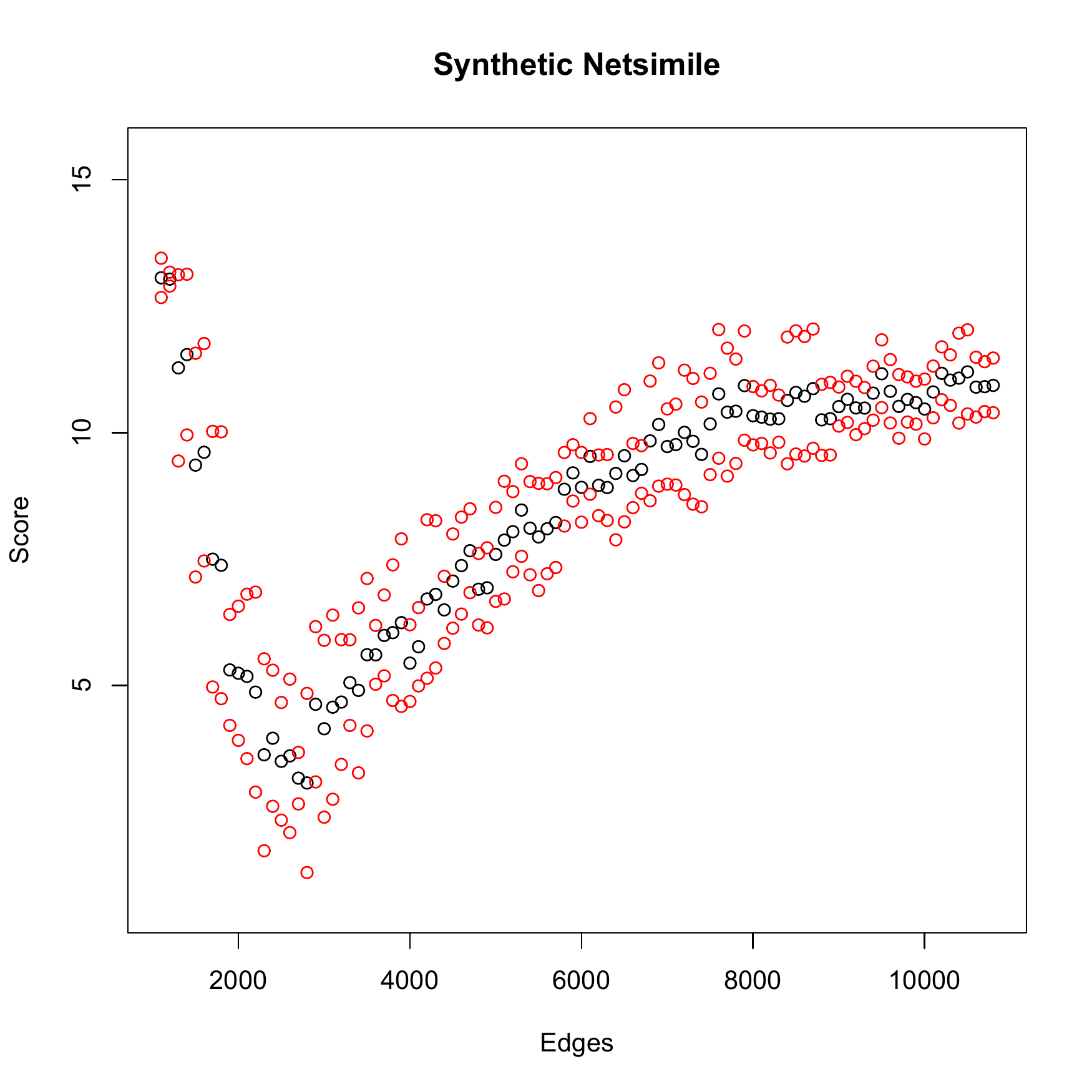}}
\end{center}

\caption{Statistic values network data generated from same model, but with increasing size.  Behavior of (a) Consistent Statistic; (b) Inconsistent Statistic. Black pts: avg of 100 trials, red pts: [min, max].}
\label{fig:compare}
\vspace{-6mm}
\end{figure}

In this work, we will analytically characterize statistics by their sensitivity to network size, and offer principled alternatives that are \textbf{ consistent estimators} of network behavior, which empirically give more accurate results when finding anomalies in dynamic networks with varying sizes.  In terms of confounding effects this approach eliminates confounding by ensuring that the test statistics used do not vary when the confounding network properties change, bringing the statistics closer to the ideal one-to-one relationship with network properties.  


\begin{figure}[h!]
\begin{centering}
\includegraphics[width=.75\columnwidth,natwidth=610,natheight=642]{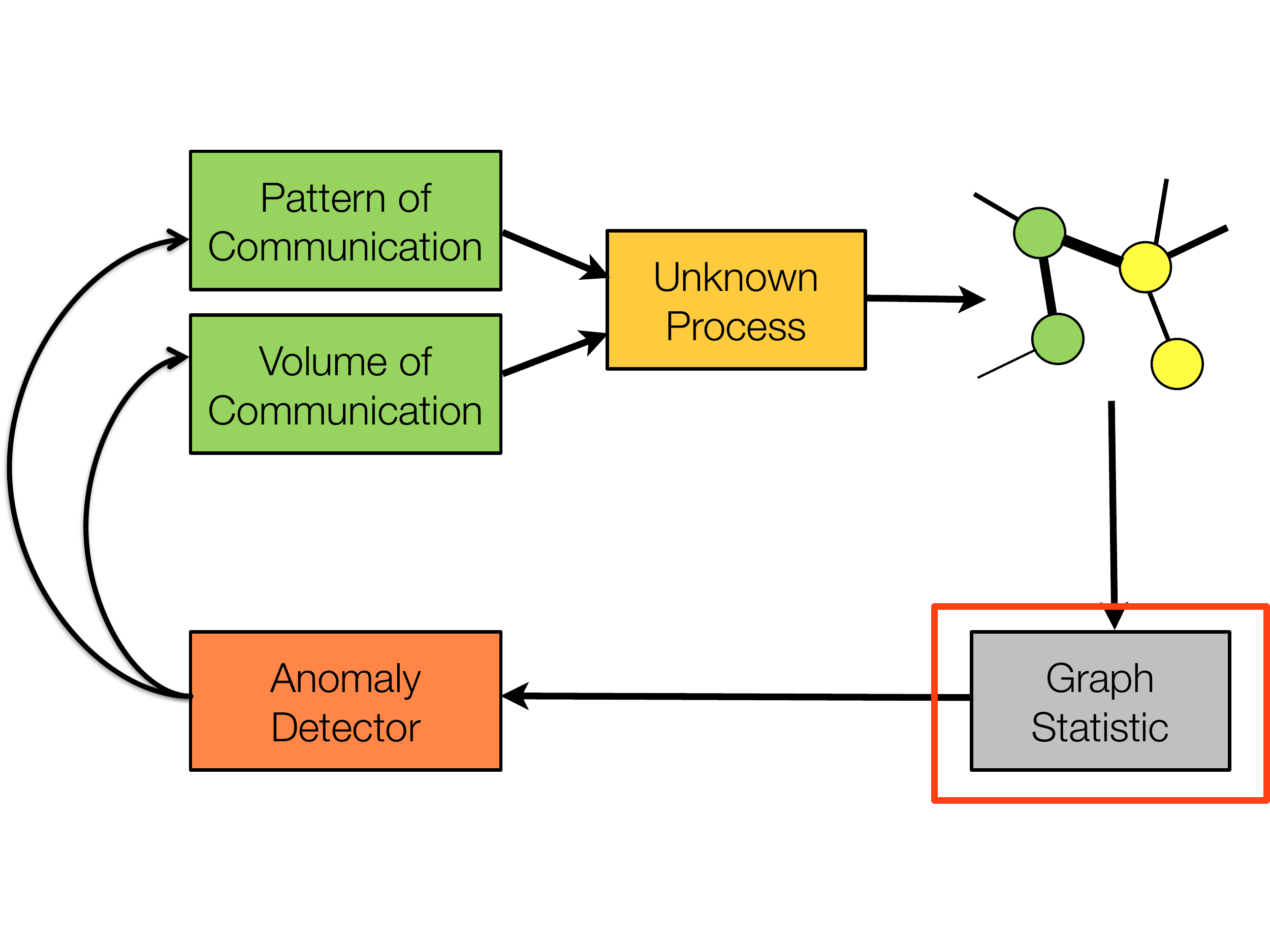}
\vspace{-5mm}
\caption{Controlling for confounding effects through careful definition of the test statistics.}
\label{control-statistic}
\end{centering}
\end{figure}




The major contributions of this work are:

\begin{itemize}
\item We define \textbf{Size Inconsistent} and \textbf{Size Consistent} properties of network statistics and show that Size Consistent statistics have fewer false positives and false negatives that Inconsistent statistics.
\item We prove that several commonly used network statistics are Size Inconsistent and fail to capture the network behavior with varying network densities.
\item We introduce provably Size Consistent statistics which measure changes to the network structure regardless of the total number of observed edges.  
\item We demonstrate that our proposed statistics converge quickly and have superior ROC performance compared to conventional statistics.
\end{itemize}

\section{Problem Definition and Data Model}




Let $G =  \{ V, W \}$ be a weighted graph that represents a network, where $V$ is the node set and $W$ is the weighted adjacency matrix representing messages or some other interaction, with $w_{ij}$  the number of messages between nodes $i$ and $j$. Let $|V|$ and $|W|$ refer to the number of nodes, and total weight of the edges, in $G$ respectively. A dynamic network is simply a set of graphs $\{ G_1, G_2, ... G_T  \}$ where each graph represents network activity within a consistent-width time step (e.g. one step per day).  



\textbf{Problem definition}: Given a stream of graph examples $\{G_{1}, G_{2}, ..., G_{{t-1}}\}$ drawn from a {\em normal} model $M^n$, and a graph $G_{t}$ drawn from an unknown model, determine if $G_{t}$ was drawn from $M^n$ or some alternative model $M^a$. \\



Given an observed graph, we wish to decide if this graph exhibits the same behavior (network properties) as past graph examples or if is likely the product of some different, anomalous properties.  We will be solving this problem with hypothesis tests utilizing network statistics as the test statistics.  If $S_k(G)$ is some network statistic designed to measure a network property $k$, then the set of statistics calculated on the normal examples $\{S_k(G_{1}), S_k(G_{2}), ..., S_k(G_{{t-1})}\}$ forms the empirical null distribution, and the value $S_k(G_{{t}})$ is the test point.

For this work we will use a two-tailed test with $p$-values of $\alpha = 0.05$.  Anomalous test cases where the null hypothesis is rejected correspond to {\em true positives}; normal cases where the null hypothesis is rejected correspond to {\em false positives}.  Likewise anomalous cases where the null is not rejected correspond to {\em false negatives} and normal cases where the null is not rejected correspond to {\em true negatives}. 

The anomaly detection procedure is summarized in Figure \ref{fig:anomalydetectiontask} from model down to null distribution and test point.



If all the graph examples have the same number of edges and nodes then graph size cannot be a confounding factor regardless of the choice of test statistic - those properties are naturally controlled in the data.  However, if $M^n$ and $M^a$ produce graphs with a variable number of edges and nodes then any test statistic needs to be robust to changes in the graph size.  Ideally, if $G_x, G_y \sim M$ but $| V_x | \neq | V_y |, | W_x | \neq | W_y |$ we would still want $S_k(G_x) \approx S_k(G_y)$ to be true. 

To accommodate the observations of graphs of varying size, let us assume the models that generated the observed graphs are hidden but take the form of a multinomial sampling procedure.  Let $P^*$ be a $| V^* | \times | V^* |$ matrix where the rows and columns represent a node set $V^*$ and the sum of all cells is equal to 1.  Here $V^*$ represents a large set of {\em possible} nodes, i.e., larger than the set we may see in any one graph $G$.  The entry $p^*_{ij,t}$ specifies the probability that a randomly sampled message at time $t$ is between $i$ and $j$.  Let $| V |$ and $| W |$ be drawn from distributions $M_V$ and $M_W$.  

\vspace{10mm}

\begin{figure}[h!]
\begin{centering}
\includegraphics[width=.75\columnwidth,natwidth=610,natheight=642]{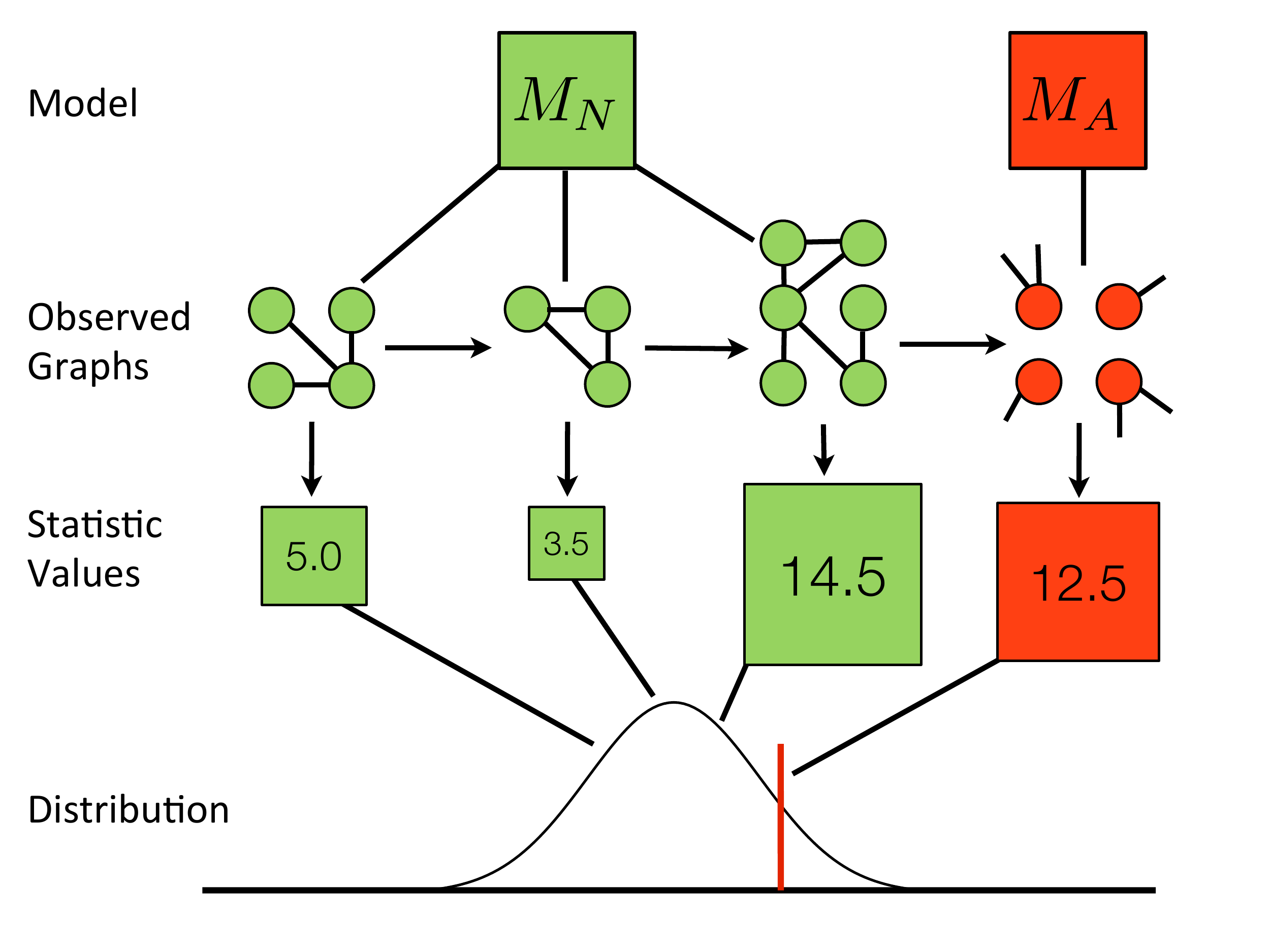}
\vspace{-1mm}
\caption{Dynamic network anomaly detection task.  Given past instances of graphs created by the typical model of behavior, identify any new graph instance created by an alternative anomalous model.}
\label{fig:anomalydetectiontask}
\end{centering}
\end{figure}

\vspace{2mm}

\clearpage

The full generative process for all graphs is then: 

\begin{itemize}
\item  Draw $|V| \sim M_V$. Select $V$ from $V^*$ uniformly at random. 
\item  Construct $P$ by selecting the rows/columns from $P^*$ that correspond to $V$ and normalize the probabilities to sum to 1 (i.e., $p_{ij} = \frac{1}{Z} p^*_{ij}$, where $Z = \sum_{ij \in V} p^*_{ij}$). 
\item Draw $| W | \sim M_W$. Sample $| W |$ messages using probabilities $P$. 
\item  Construct the graph $G=(V,W)$ from the sampled messages.
\end{itemize}

\noindent $G$ is the output of a multinomial sampling procedure on $P$, with each independent message sample increasing the weight of one cell in $W$.  $P$ itself is a set of probabilities obtained by sampling $V$ from $V^*$.  This graph generation process is summarized in figure \ref{sampling-complete}.

\begin{figure}[]
\begin{centering}
\includegraphics[width=.75\columnwidth,natwidth=610,natheight=642]{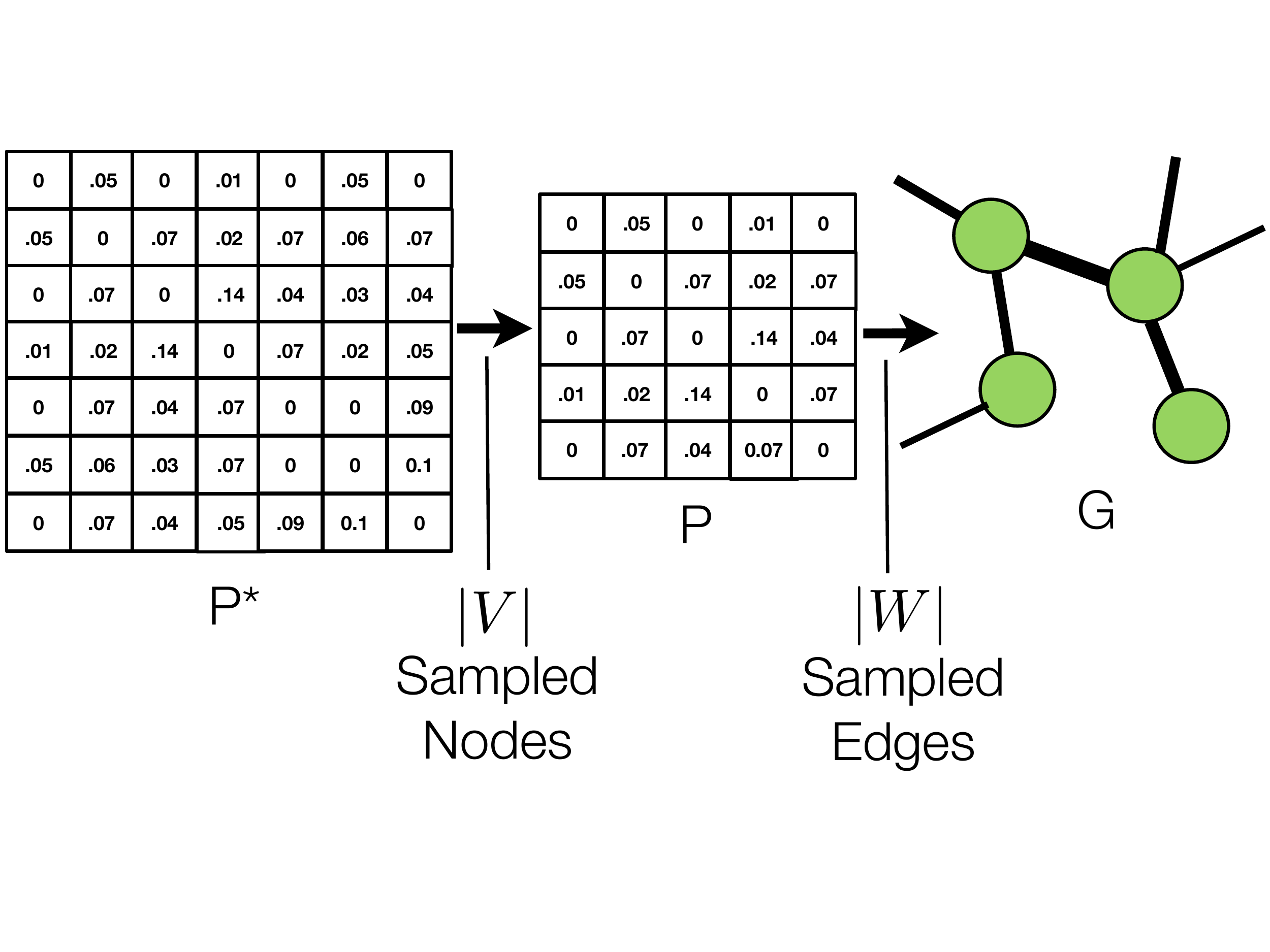}
\vspace{-15mm}
\caption{Graph generation process.  The matrix $P^*$ represents all possible nodes and their interaction probabilities.  By sampling $|V|$ nodes and $|W|$ edges the observed graph $G$ is obtained. }
\label{sampling-complete}
\end{centering}
\end{figure}

Given this generative process, the difference between normal and anomalous graphs is characterized by differences in their underlying models.  Let the normal model be represented by $P^{*n}, M^{n}_V$ and $M^{n}_W$ and let the anomalous model be represented by $P^{*a}, M^{a}_V$ and $M^{a}_W$.  Finding instances where $M_V$ or $M_W$ are anomalous is trivial since we can use the count of nodes or messages as the test statistic.  Finding instances of graphs drawn from $P^{*a}$ is more difficult as our choice of network statistics affects whether we are sensitive to changes in $|W|$ or $|V|$.


If we redefine our network statistics to be functions over $P^*$ instead of $G$ we avoid the problem of graph size as a confounding factor as $P^*$ is independent of $M_V$ or $M_W$.  However, since $P^*$ is unobservable, there is no way to calculate $S_k(P^*)$ directly. Instead we can only calculate $\widehat{S}_k(G)$ from the observed graph $G$.  $\widehat{S}_k(G)$ is an estimate of $S_k(P)$ using the sampled messages $W$ to estimate the underlying probabilities, and $S_k(P)$ itself is an estimate of the true $S_k(P^*)$ on a subset $V$ of the total nodes.  So just as the sampling procedure follows $P^* \rightarrow P \rightarrow G$, the estimation procedure follows the inverse steps $\widehat{S}_k(G) \rightarrow S_k(P) \rightarrow S_k(P^*)$.


Delta statistics like Graph Edit Distance can also be used for anomaly detection.  In this case the empirical statistic will be $\widehat{S}_k(G_1,G_2)$, where $G_1$ and $G_2$ are generated using the graph generation procedure described above, and the true value of the statistic is $S_k(P^*_1,P^*_2)$.  In order to be consistent Delta statistics should not change when either graph changes in size.  

Ideally, $\widehat{S}_k(G) = S_k(P) = S_k(P^*)$ and we would have the same output regardless of $W$ and $V$, being sensitive only to changes in the model.  However this is typically not attainable in practice as it is difficult to estimate the true statistic value from graphs that are extremely small - few edges and nodes provides less evidence of the underlying properties.  In addition, an unbiased statistic with extremely high variance is also a poor test statistic.  In many scenarios the best statistics are ones which converge to the value of $S_k(P^*)$ as $|V|, |W|$ increase, a property that we will refer to as {\em Size Consistency}.  In the next section we will formally define the properties of Size Consistent and Size Inconsistent statistics and show how they affect the accuracy of hypothesis tests.

\section{Properties of the Test Statistic}



\subsection{Size Consistency} 




As described previously, a statistic $S_k(P^*)$ depends on the properties of the procedure that generated the graph instance and is a measure of the graph properties independent of the exact number of edges and nodes in the graph.  Although the empirical statistic $\widehat{S}_k(G)$ may not be independent of the edge and node count, if it converges to $S_k(P^*)$ as $|V|$ and $|W|$ increase it is a reasonable approximation as long as $|V|$ and $|W|$ are large enough.  The bias of the empirical statistic due to graph size is $abs( S_k(P^*) - \widehat{S}_k(G) )$; if this bias converges to 0 as  $|V|$ and $|W|$ increase then $\widehat{S}_k(G)$ is \textit{Size Consistent}.


\begin{definition}
A statistic $\widehat{S}_k$ is {\em Size Consistent} w.r.t. $S_k$ if:

\vspace{-6mm}
\begin{small}
\begin{align*} 
&\lim_{|W| \rightarrow \infty} \widehat{S}_k(G) - S_k(P) = 0 \\
\mbox{ AND } \;\;\;  &\lim_{|V| \rightarrow |V^*|} S_k(P) - S_k(P^*) = 0 
\end{align*}
\end{small}
\vspace{-2mm}
\end{definition}






Delta statistics have the same requirements for consistency as standard statistic except that the edge and node count of both graphs must be increasing. 

\begin{definition}
A delta statistic $\widehat{S}_k$ is {\em Size Consistent} w.r.t. $S_k$ if:

\vspace{-6mm}
\begin{small}
\begin{align*} 
&\lim_{|V_1|, |V_2| \rightarrow |V^*|} \widehat{S}_k(G_1,G_2) = S_k(P_1,P_2) \\
\mbox{ AND } \;\;\;  &\lim_{|W_1|, |W_2| \rightarrow \infty} S_k(P_1,P_2) = S_k(P^*_1,P^*_2)
\end{align*}
\end{small}
\vspace{-2mm}
\end{definition}

\begin{theorem} \textbf{False Positive Rate for Size Consistent Statistics} \\
Let $ \{G_1... G_{{x}} \}$ be a finite set of ``normal'' graphs drawn from $P^*$, $M^{n}_W$ and $M^{n}_V$ and let $G_{{test}}$ be a test graph drawn from $P^*$, $M^{a}_W$ and $M^{a}_V$. 
Let $| W_{min} |$ be the minimum edge count in both $\{ G_{{x}} \}$ and $G_{{test}}$ and $|V_{min}|$ be the minimum node count.  For a hypothesis test using a Size Consistent test statistic $S_k$ and a p-value of $\alpha$, as $| W_{min} |\rightarrow \infty$ and $|V_{min}| \rightarrow |V^*|$ the probability of identifying $G_{{test}}$ as a false positive approaches $\alpha$. 
\label{convergetrueneg}
\end{theorem}

\begin{proof}
If $\widehat{S}_k(G)$ is a consistent estimator of $S_k(P)$, then 
as $| W_{min} | \rightarrow \infty$, $\widehat{S}_k(G) \rightarrow S_k(P)$ and if $S_k(P)$ is a consistent estimator of $S_k(P^*)$ then as $|V_{min}| \rightarrow |V^*|$, $S_k(P) \rightarrow S_k(P^*)$.  If $\{G_1...G_x\}$ and $G_{test}$ are drawn from $P^{*n}$, then $\widehat{S}_k(G_1)...\widehat{S}_k(G_x)$ and $\widehat{S}_k(G_{test}$ are converging to the same distribution of values and the hypothesis test will reject with the p-value probability of $\alpha$.
\end{proof}

As the number of edges and nodes drawn for the null distribution and test instance increase, the bias $abs( S_k(P^*) - \widehat{S}_k(G) )$ of the statistic calculated on those networks converges to zero.  This means that $\widehat{S}_k(G)$ effectively becomes equal to $S_k(P^*)$, and the outcome of the hypothesis test is only dependent on whether the test instance and null examples were both drawn from $P^{*n}$ or if the test instance was drawn from $P^{*a}$.  Even if the test case has an unusual number of edges or nodes, as long as the number of edges and nodes is not too small there will not be a false positive.  


Size consistency is also beneficial in the case of false negatives.  A statistic which is sensitive to changes in the edge or node count will produce a null distribution with high variance if $M_V^n$ or $M_W^n$ have high variance, which increases the chance of producing false negatives.  A size consistent statistic will have less variance as $|V|$ and $|W|$ increase, so as long as the minimum outputs of $M_V^n$ or $M_W^n$ are not too small the variance will be negligible.

\begin{theorem} \textbf{False Negative Rate for Size Consistent Statistics} \\
Let $G_{{test}}$ be a network that is anomalous (i.e., drawn from $P^{*a}$, $M^n_V$, $M^n_W$) with respect to property $k$, and $\{G_1...G_x \}$ be graph examples drawn from $P^{*n}$, $M^n_V$, $M^n_W$.  Let $| W_{min} |$ be the minimum of $M^n_W$ and $|V_{min}|$ be the minimum of $M^n_V$.  As $| W_{min} | \rightarrow \infty$ and $|V_{min}| \rightarrow |V^*|$ the probability of failing to reject $G_{test}$ approach 0 if $P^{*n} \neq P^{*a}$.
\label{convergetruepos}
\end{theorem}

\begin{proof}
If $\widehat{S}_k(G)$ is a consistent estimator of $S_k(P)$, then 
as $| W_{min} | \rightarrow \infty$, $\widehat{S}_k(G) \rightarrow S_k(P)$ and if $S_k(P)$ is a consistent estimator of $S_k(P^*)$ then as $|V_{min}| \rightarrow |V^*|$, $S_k(P) \rightarrow S_k(P^*)$.  If $S_k(P^{*a}) \neq S_k(P^{*n})$, then as $| W_{min} | \rightarrow \infty$ and $|V_{min}| \rightarrow |V^*|$ the statistic $\widehat{S}_k(G_{test})$ and the set of statistics $ \widehat{S}_k(G_{1}), \widehat{S}_k(G_{2}), ... \widehat{S}_k(G_{x})$ converge to different values and $G_{test}$ will be flagged as an anomaly with probability 1.
\end{proof}

Now that we have investigated the effects of size consistency, we must look at the effects of its inverse.

\subsection{Size Inconsistency}



{\em Size Inconsistency} is the inverse of size consistency: if a statistic is not size consistent, then it is size inconsistent.  

\begin{definition}
A statistic $\widehat{S}_k$ is {\em Size Inconsistent} w.r.t. $S_k$ if: 

\vspace{-6mm}
\begin{small}
\begin{align*} 
&\lim_{|W| \rightarrow \infty} \widehat{S}_k(G) - S_k(P) \neq 0  \\
\mbox{ OR } \;\;\;  &\lim_{|V| \rightarrow |V^*|} S_k(P) - S_k(P^*) \neq 0 
\end{align*}
\end{small}
\vspace{-2mm}
\end{definition}

\noindent Where $S_k$ is a nontrivial function (a trivial function being one that is a constant, $\infty$, or $-\infty$ for all input values).  This definition also applies to the delta statistic case.

\begin{theorem} \textbf{False Positives for Size Divergent Statistics} \\
Let $\{G_1... G_{{x}} \}$ be a finite set of ``normal'' graphs drawn from $P^*$, $M^{n}_W$, and $M^{n}_V$ and $G_{{test}}$ be a test graph drawn from $P^*$.  If $\widehat{S}_k(G)$ diverges with increasing $|W|$ or $|V|$ and $M^{n}_W$, $M^{n}_V$ have finite bounds, there is some $|V|$ or $|W|$ for which a hypothesis test using $S_k(G)$ as the test statistic will incorrectly flag $G_{test}$ as an anomaly. 
\label{falsepos}
\end{theorem}

\begin{proof}
When the set of graphs $\{G_1...G_{{x}}\}$ are used to estimate an empirical distribution of the null $\widehat{S_k}$, the distribution is bounded by $max[{S}_k(\{G_1...G_{{x}}\})]$ and $min[{S}_k(\{G_1...G_{x}\})]$, so the critical points $\phi_{lower}$ and $\phi_{upper}$ of a hypothesis test using this set of graphs will be within these bounds.  Since an increasing $| W_{{test}} | $ or $|V_{test}|$ implies ${S}_k(G_{{test}})$ diverges, then there exists a $| W_{{test}} |$ or $| V_{test}|$ such that ${S}_k(G_{{test}})$ is not within $\phi_{lower}$ and $\phi_{upper}$ and will be rejected by the test.
\end{proof}

Size Inconsistency generally occurs when the value of a statistic is a linear function of the edge weight or the number of nodes in the graph: when the edge weight or the number of nodes goes to infinity, the output of the statistic also diverges.  If a statistic is dependent on the size of a graph, then two graphs both drawn from $P^{*n}$ may produce entirely different values and a false positive will occur.

A second problem occurs when the edge counts in the estimation set have high variance.  If the statistic is dependent on the number of edges, noise in the edge counts translates to noise in the statistic values which lowers the statistical power (i.e. the percentage of true anomalies detected) of the test.  With a sufficient amount of edge count noise, the signal is completely drowned out and the statistical power of the anomaly detector drops to zero. 

\begin{theorem} \textbf{False Negatives for Size Divergent Statistics} \\
Let $G_{{test}}$ be a network that is anomalous (i.e., drawn from $P^{*a}$) with respect to property $k$.  If ${S}_k(G)$ diverges with increasing $| W|$ or $|V|$ there exists some $M^{n}_W$ or $M^{n}_V$ with sufficient variance such that a hypothesis test with p-value $\alpha$ and empirical null distribution $\widehat{S_k}$ will fail to detect ${S}_k(G_{{test}})$ as an anomaly with probability $1 - \alpha$. 
\label{falseneg}
\end{theorem}

\begin{proof}
Let $G_{{test}}$ be the test network drawn from $P^{*a}$, $M^{n}_W$, and $M^{n}_V$, and $\{G_1...G_{{x}}\}$ be the set of null distribution graphs drawn from $P^{*n}$, $M^{n}_W$ and $M^{n}_V$.  If ${S}_k(G)$ is a divergent function of $| W|$ or $|V|$, then the variance of of the null distribution $\widehat{S}_k$ estimated from $\{G_1...G_{{x}}\}$ is dependent on the variance of $| W|$ and $|V|$.  If the variance of sampled $| W|$ or $|V|$ is sufficiently large, the variance of $\widehat{S}_k$ will increase to cover all possible ${S}_k(G)$ values, and the test instance will fail to be flagged as an anomaly with probability $1 - \alpha$. 
\end{proof}

\noindent With a sufficient amount of edge count noise, the statistical power of the anomaly detector drops to zero.  Regardless of whether a time step is an example of an anomaly or not, if the variance of the test statistic is dominated by random noise in the edge count the time step will only be flagged due to random chance.

These theorems show that divergence with number of edges or nodes can lead to both false positives and false negatives from anomaly detectors that look for unusual network properties other than edge count. %
These theorems have been defined using a statistic calculated on a single network, but some statistics are delta measures which are measured on two networks.  In these cases, the edge counts of either or both of the networks can cause edge dependency issues.




\section{Network Statistics}
\label{Network Statistics}
In this section we introduce our set of proposed size consistent statistics, as well as analyze multiple existing statistics to determine if they are size consistent or inconsistent.  These properties are summarized in Table \ref{tab:statistics}; \textit{Fast Convergence} indicates fewer necessary edge/node observations to obtain a high level of accuracy. 


\subsection{Conventional Statistics}

\vspace{-2mm}
\paragraph{Graph Edit Distance}
\noindent The graph edit distance (GED) \cite{gao} is often used in anomaly detection tasks. 
GED on a weighted graph is typically defined as:

\vspace{-3mm}
\begin{small}
\begin{align} \label{eq:ged}
GED(G_{1}, G_{2}) &= | V_{1} | + | V_{2} | - 2  | V_{1} \cap V_{2} | \nonumber \\ & + \sum_{ij \in V_1 \cup V_2} abs( w_{ij,1} - w_{ij,2} ) 
\end{align}
\end{small}
\vspace{-3mm}

\begin{claim} 
GED is a Size Inconsistent statistic.
\end{claim}



\noindent Consider the case where $G_1$ and $G_2$ are both drawn from $P^*$.  Let $|W_2| = |W_1| + W_\Delta$ where $W_\Delta$ is some constant value.  The expected difference in weights between two nodes $i,j$ in $G_1$ versus $G_2$ is:

\vspace{-4mm}
\begin{small}
\begin{align*}
 E[ w_{ij,1} - w_{ij,2} ] &= |W_1| p_{ij} - (|W_1| + W_{\Delta}) p_{ij} \nonumber = W_{\Delta} p_{ij}  
 \end{align*}
\end{small}
\vspace{-2mm}

\noindent Then, the limit as $|W_1|, |W_2|$ increase is

\vspace{-4mm}
\begin{small}
\begin{align*}
\lim_{|W_1|,|W_2| \rightarrow \infty} &GED(G_{1}, G_{2}) \nonumber \\ &= \lim_{|W_1|,|W_2| \rightarrow \infty} |V| + |V| - 2 |V \cap V| + \sum_{ij \in V} W_{\Delta} p_{ij} \nonumber \\ &= W_{\Delta}
 \end{align*}
\end{small}
\vspace{-2mm}

\noindent As $GED(G_{1}, G_{2})$ is converging to a constant value, it is not converging to a nontrivial $S_k$ and the first condition of Size Consistency is violated. $\square$



\paragraph{Degree Distribution and Degree Dist. Difference}

\noindent As defined before the Degree Distribution of a graph is the distribution of node degrees.  In this task we will find the difference between the degree distributions of two graphs using a delta statistic. 
Define the delta statistic Degree Distribution Difference $DD(G_1,G_2)$ between two graphs as:

\vspace{-3mm}
\begin{small}
\begin{align} \label{eq:dd}
 {DD}(G_1, G_2) =& \sum_{bin_k \in Bins} ( \sum_{i \in V_1} I[ D_i(G_1) \in bin_k ] \nonumber \\ &-  \sum_{i \in V_2} I[ D_i(G_2) \in bin_k ] )^2
\end{align}
\end{small}
\vspace{-2mm}

\noindent where $Bins$ is a consecutive sequence of equal size bins which encompass all degree values in both graphs.  Note that this value is an approximation of the Cram\'{e}r von-Mises criterion between the two empirical degree distributions.
Let the probabilistic degree of node $i$ be $D_i(P^*) = \sum_{j \neq i \in V^*} p^*_{ij}$.  Then let the value of $DD(P^*_1,P^*_2)$ be:

\vspace{-3mm}
\begin{small}
\begin{align} \label{eq:edd}
 {DD}(P^*_1,P^*_2) = &\sum_{bin_k \in Bins} ( \sum_{i \in V^*} I[ D_i(P^*_1) \in bin_k ] \nonumber \\ &-  \sum_{i \in V^*} I[ D_i(P^*_2) \in bin_k ] )^2
\end{align}
\end{small}
\vspace{-2mm}

%
%
%
%

\begin{table}[t]
\centering
\caption{Statistical properties of previous network statistics and our proposed alternatives.}
\label{tab:statistics}
{
\begin{tabular}{l | c | c | c | c |}
& Inconsist. & Consist. & Fast   \\
&&& Convergence \\
\hline
Mass Shift && \checkmark & \checkmark  \\
Probabilistic Degree && \checkmark & \checkmark  \\
Triangle Probability && \checkmark &\checkmark \\
\hline
Graph Edit Distance  & \checkmark & & \\
Degree Distribution  & \checkmark & & \\
Barrat Clustering & &\checkmark  &     \\
\hline
Netsimile &  \checkmark & & \\
Deltacon & & \checkmark & \\
\end{tabular}
}
\end{table}

\begin{claim} 
$DD(G_1,G_2)$ is a Size Inconsistent statistic.
\end{claim}
%
%


\noindent Let $G_1, G_2$ be drawn from $P^*$ using the same node set $V$ and let $|W_2| = |W_1| + W_\Delta$ for some constant $W_\Delta$.  As $|W_1|,|W_2|$ increase the value of $D_i(G_2) - D_i(G_1)$ converges to 

\vspace{-2mm}
\begin{small}
\begin{align}
 (|W_1| + W_\Delta) \sum_{j \neq i \in V} p^*_{ij} -  |W_1| \sum_{j \neq i \in V} p^*_{ij} = W_\Delta \sum_{j \neq i \in V} p^*_{ij}
\end{align}
\end{small}
\vspace{-1mm}

\noindent So for sufficiently large $W_\Delta$, at least one node will be placed into a higher bin for $G_2$ versus $G_1$, and the limit $\lim_{|W_1|,|W_2|\rightarrow \infty} DD(G_1,G_2)$ is not equal to $DD(P^*,P^*)$.  This violates the first condition of Size Consistency and therefore the Degree Distribution Difference is Size Inconsistent. $\square$

Other measures create aggregates using the degrees of multiple nodes \cite{priebe, berlingerio} but as the degree is size inconsistent these aggregates tend to be so as well. 

\paragraph{Weighted Clustering Coefficient}
\noindent Clustering coefficient is a measure of the transitivity, the propensity to form triangular relationships in a network.
As the standard clustering coefficient is not designed for weighted graphs we will be analyzing a weighted clustering coefficient, specifically the Barrat weighted clustering coefficient (CB)\cite{saramaki}:

\vspace{-3mm}
\begin{small}
\begin{align} \label{eq:wcc}
CB(G) = \sum_{i} \frac{1}{|V|(k_i-1)D_i(G)} \sum_{j,k} \frac{w_{ij}+w_{ik}}{2} \widehat{a}_{ij}\widehat{a}_{ik}\widehat{a}_{jk} 
\end{align}
\end{small}
\vspace{-2mm}

\noindent where $\widehat{a}_{ij} = I[w_{ij} > 0]$, $D_i(G) = \sum_j w_{ij}$, and $k_i = \sum_j a_{ij}$.  Other weighted clustering coefficients exist but they behave similarly to the Barrat coefficient.

\begin{claim} 
CB is a Size Consistent statistic that converges to 
\begin{small}
\begin{align}
CB(P^*) &= \sum_{i} \! \frac{1}{|V^*| ({a}^*_i-1) \sum_{j} \! p^*_{ij}}  \sum_{j,k}   \frac{ (p^*_{ij}+p^*_{ik})}{2} {a}^*_{ij}{a}^*_{ik}{a}^*_{jk}.
\end{align}
\end{small}
\end{claim}

%

\noindent First we will find $CB(P)$ by taking the limit as $|W| \rightarrow \infty$:

\vspace{-3mm}
\allowdisplaybreaks{
\begin{align*}
\lim_{|W| \rightarrow \infty} &CB(G) \nonumber \\ &= \lim_{|W| \rightarrow \infty} \sum_{i} \frac{1}{|V| ( \widehat{a}_i-1) D_i(G)} \sum_{j,k} \frac{w_{ij}+w_{ik}}{2} \widehat{a}_{ij}\widehat{a}_{ik}\widehat{a}_{jk} \nonumber \\
&= \sum_{i} \frac{1}{|V| ({a}_i-1) |W| \sum_{j} p_{ij}} \sum_{j,k} \frac{|W| (p_{ij}+p_{ik})}{2} {a}_{ij}{a}_{ik}{a}_{jk} \nonumber \\
&= \sum_{i} \frac{1}{|V| ({a}_i-1)  \sum_{j} p_{ij}} \sum_{j,k} \frac{ (p_{ij}+p_{ik})}{2} {a}_{ij}{a}_{ik}{a}_{jk} \nonumber \\
&= CB(P)
\end{align*}
}
\vspace{-3mm}

\noindent Now we take the limit $\lim_{|V| \rightarrow |V^*|} CB(P)$:

\vspace{-3mm}
\allowdisplaybreaks{
\begin{align*}
& \lim_{|V| \rightarrow |V^*|} CB(P) \nonumber \\ &= \lim_{|V| \rightarrow |V^*|} \sum_{i} \frac{1}{|V| ({a}_i-1)  \sum_{j} p_{ij}} \sum_{j,k} \frac{ (p_{ij}+p_{ik})}{2} {a}_{ij}{a}_{ik}{a}_{jk} \nonumber \\
&= \lim_{|V| \rightarrow |V^*|} \sum_{i} \frac{1}{|V| ({a}_i-1) \frac{1}{\sum_{ij \in V} p^*_{ij}} \sum_{j} p^*_{ij}} \nonumber \\ & \sum_{j,k}  \frac{1}{\sum_{ij \in V}p^*_{ij}} \frac{ (p^*_{ij}+p^*_{ik})}{2} {a}_{ij}{a}_{ik}{a}_{jk} \nonumber \\
& =  \sum_{i} \frac{1}{|V^*| ({a}^*_i-1) \sum_{j} p^*_{ij}}  \sum_{j,k}  \frac{ (p^*_{ij}+p^*_{ik})}{2} {a}^*_{ij}{a}^*_{ik}{a}^*_{jk} \nonumber \\
\hfill \square
\end{align*}
}
\vspace{-4mm}


Other weighted clustering coefficients such as those proposed by Onnela et. al. \cite{onnela} and Holme et. al. \cite{holme} behave in a similar manner.

\paragraph{Deltacon}
\noindent The core element of the Deltacon statistic is the Affinity Matrix which is a measure of the closeness (in terms of random walk distance) between all nodes in a graph.  Pairs of graphs with similar Affinity Matrices are scored as being more likely to be from the same distribution.  

\begin{claim} 
Deltacon is a Size Consistent statistic.
\end{claim}

\noindent The Affinity Matrix $S$ is approximated with Fast Belief Propagation and is estimated with $S \approx I + \epsilon*A + \epsilon^2*A^2$ where $A$ is the adjacency matrix and $\epsilon$ is the coefficient of attenuating neighbor influence.  As $|W| \rightarrow \infty$ and $|V| \rightarrow |V^*|$ the adjacency matrix $A$ approaches $A^*$ which is the adjacency matrix of $P^*$, so the statistic does converge to the value given by the Affinity Matrix difference calculated on the true $P^*$ graphs.  However, this convergence will be slow in practice as a small difference between $A$ and $A^*$ can cause large changes in the path lengths between nodes if the missing edges are critical bridges between graph regions.


\paragraph{Netsimile}
\noindent Netsimile is an aggregate statistic using multiple simple statistics to form descriptive vectors.  These statistics include number of neighbors, clustering coefficient (unweighted), two-hop neighborhood size, clustering coefficient of neighbors, and total edges, outgoing edges, and neighbors of the egonet.  

\begin{claim} 
Netsimile is a Size Inconsistent statistic.
\end{claim}

\noindent Statistics that use the raw edge count such as $D_i(G)$ will not be consistent as shown earlier, so aggregates that use these types of statistics will also be inconsistent.  The statistic uses the Canberra distance ($\frac{abs(S_k(G_1) - S_k(G_2))}{S_k(G_1) + S_k(G_2)}$) for each component statistic $S_k$ as a form of normalization, but as the component statistics diverge to infinity the Canberra distance converges to 0 and the normalization is still inconsistent.


\subsection{Proposed Size Consistent Statistics}

We will now define a set of Size Consistent statistics designed to measure network properties similar to the previously described dependent statistics, but without the sensitivity to total network edge count.  They will also be designed such that the empirical estimations converge to the true values as quickly as possible.


\begin{figure}[h!]
\begin{centering}
\includegraphics[width=.75\columnwidth,natwidth=610,natheight=642,trim={0mm 30mm 0mm 30mm},clip]{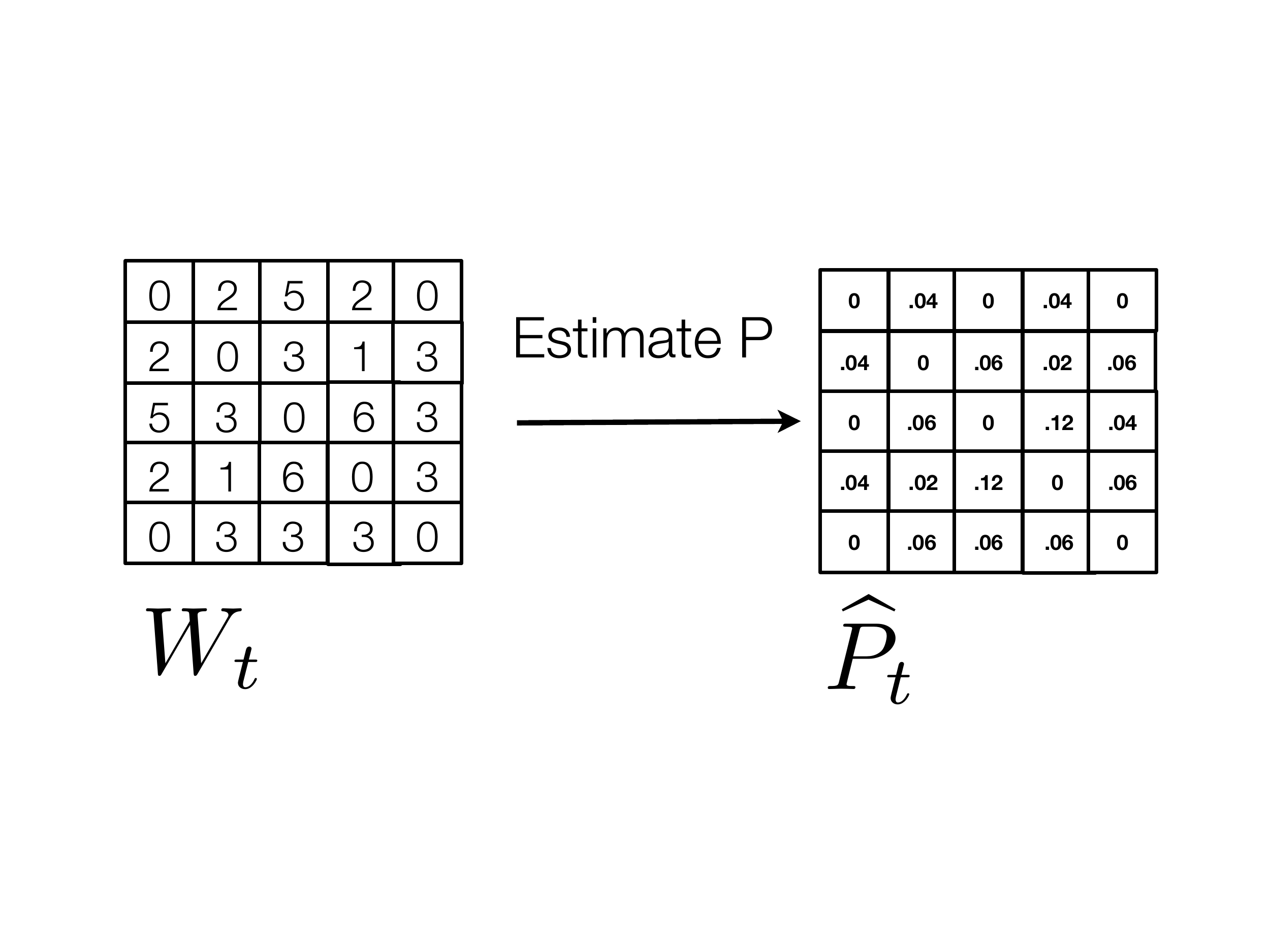}
\vspace{-5mm}
\caption{Estimation of the $\widehat{P_t}$ matrix from the observed $W$ weights.}
\label{P-estimation}
\end{centering}
\end{figure}

\vspace{4mm}

These statistics use a matrix $\widehat{P}_t$ where $\widehat{p}_{ij,t} = \frac{w_{ij,t}}{|W_t|}$ which is empirical estimate of $P_t$ obtained by normalizing the matrix as shown in figure \ref{P-estimation}.  Obtaining this matrix can be thought of as a reversal of the sampling process shown in Figure \ref{sampling-complete}.  Although $\widehat{P}_t$ is only an estimate of $P_t$, it is an unbiased one, and given an increasingly large $|W_t|$ it will be eventually exactly equal to $P_t$.  Therefore, empirical statistics which use $\widehat{P}_t$ in place of $P_t$ as their input will converge to the true statistic calculated on $P_t$ and the statistic converges w.r.t. $|W_t|$.

However, this does not guarantee that $\widehat{P}_t$ will converge in probability to $P^*_t$ as the number of nodes in $V_t$ increases.  In fact, the value of any cell $\widehat{p}_{ij,t}$ is inversely proportionate to $|V_t|$: as both $\widehat{P}_t$ and $P^*_t$ are probability distributions which sum to 1, the more cells in either matrix the lower the probability mass in each cell on average.  This concentrating effect as $V_t$ is sampled from $V^*_t$ is demonstrated in figure \ref{mass-concentration}.  



\begin{figure}[h!]
\begin{centering}
\includegraphics[width=.75\columnwidth,natwidth=610,natheight=642, trim={0mm 20mm 0mm 10mm},clip]{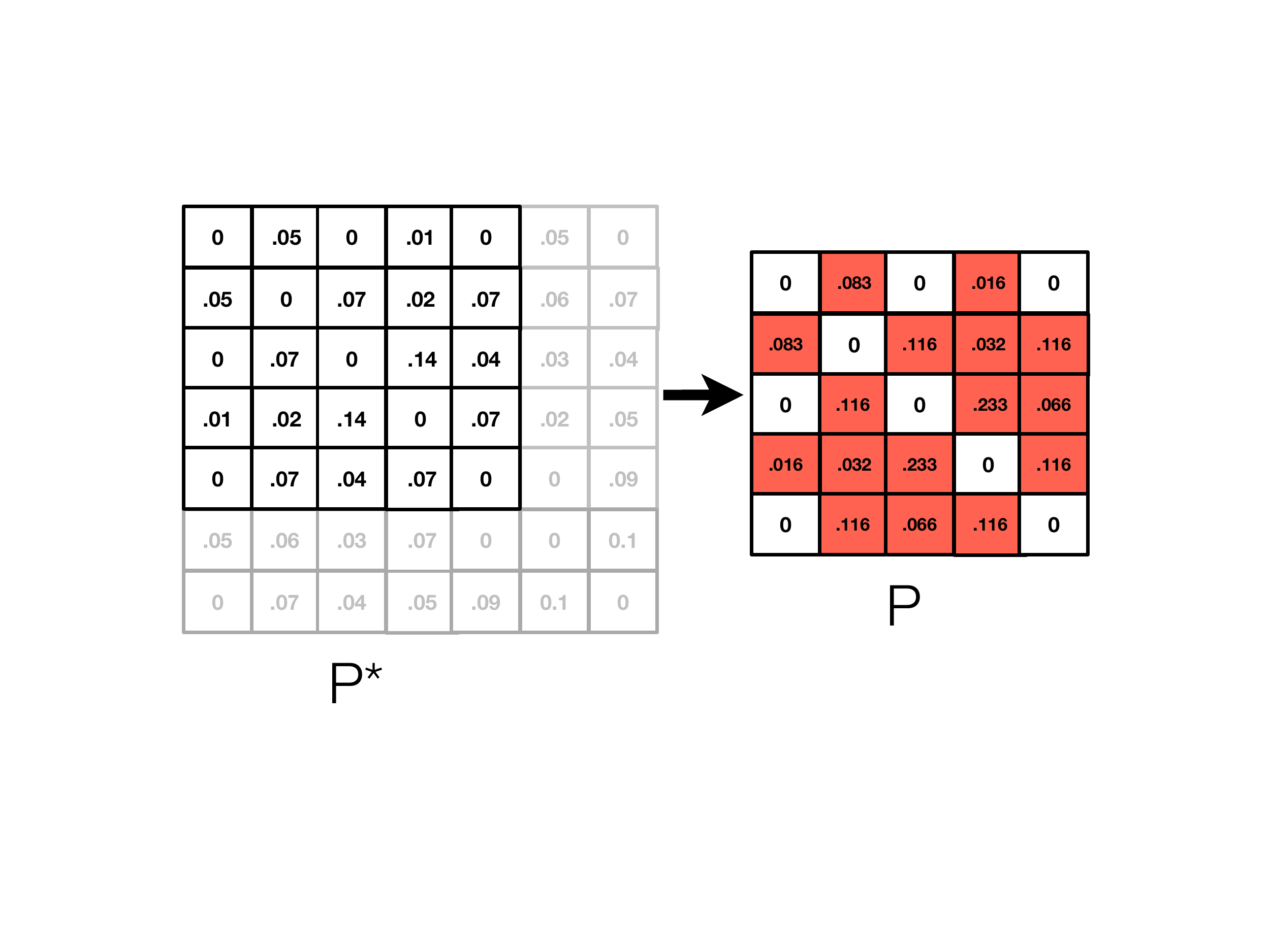}
\vspace{-8mm}
\caption{Mass of the cells in $P$ increase as $|V|$ decreases.}
\label{mass-concentration}
\end{centering}
\end{figure}


The solution to avoiding this concentration of probability mass is to introduce normalizing terms which negates the effect. These terms are $\overline{p^*}_t = \frac{1}{| A_t^* |} \sum_{ij \in V_t^*} p^*_{ij,t}$ for $S_k(P^*_t)$ and the empirical term $\overline{p}_t = \frac{1}{| A_t |} \sum_{ij \in V} p_{ij,t}$ for $S_k(P_t)$, where $|A_t^*|$ and $|A_t|$ are the number of nonzero cells in $P_t^*$ and $P_t$ respectively.  Replacing each $p^*_{ij,t}$ and $p_{ij,t}$ term in a statistic with $\frac{p^*_{ij,t}}{\overline{p^*_t}}$ and $\frac{p_{ij,t}}{\overline{p_t}}$ ensures that the statistic also converges as $|V_t|$ increases.

The utility of the $\overline{p^*}_t$ and $\overline{p}_t$ terms is to normalize the probability mass concentration effect when the size of $|V|$ changes.  As $P$ is a proper probability distribution and sums to a total of 1, decreasing $|V|$ also causes the cells in $P$ to decrease and the probability mass in each cell to rise (illustrated in figure \ref{mass-concentration}). Normalizing by the mean of each nonzero cell $\overline{p}$ allows the terms of the consistent statistics to converge as $|V|$ increases and ensures that the bias remains small.  Another way to consider this term is that $\overline{p}$ is a renormalization of $\overline{p^*}_V$ where $\overline{p^*}_V$ is the mean of the subset of $P^*$ cells that belong to $V$.  As $\overline{p^*}_V$ is the sample mean approximating $\overline{p^*}$ it is an unbiased estimator of $\overline{p^*}$ and the inverse $\frac{1}{\overline{p^*}_V}$ is a consistent estimator of $\frac{1}{\overline{p^*}}$ due to Slutsky's theorem.  The regions spanned by each of these terms are shown in figure \ref{region-visualization}.



\begin{figure}[h!]
\begin{centering}
\includegraphics[width=.75\columnwidth,natwidth=610,natheight=642,trim={0mm 10mm 0mm 10mm}]{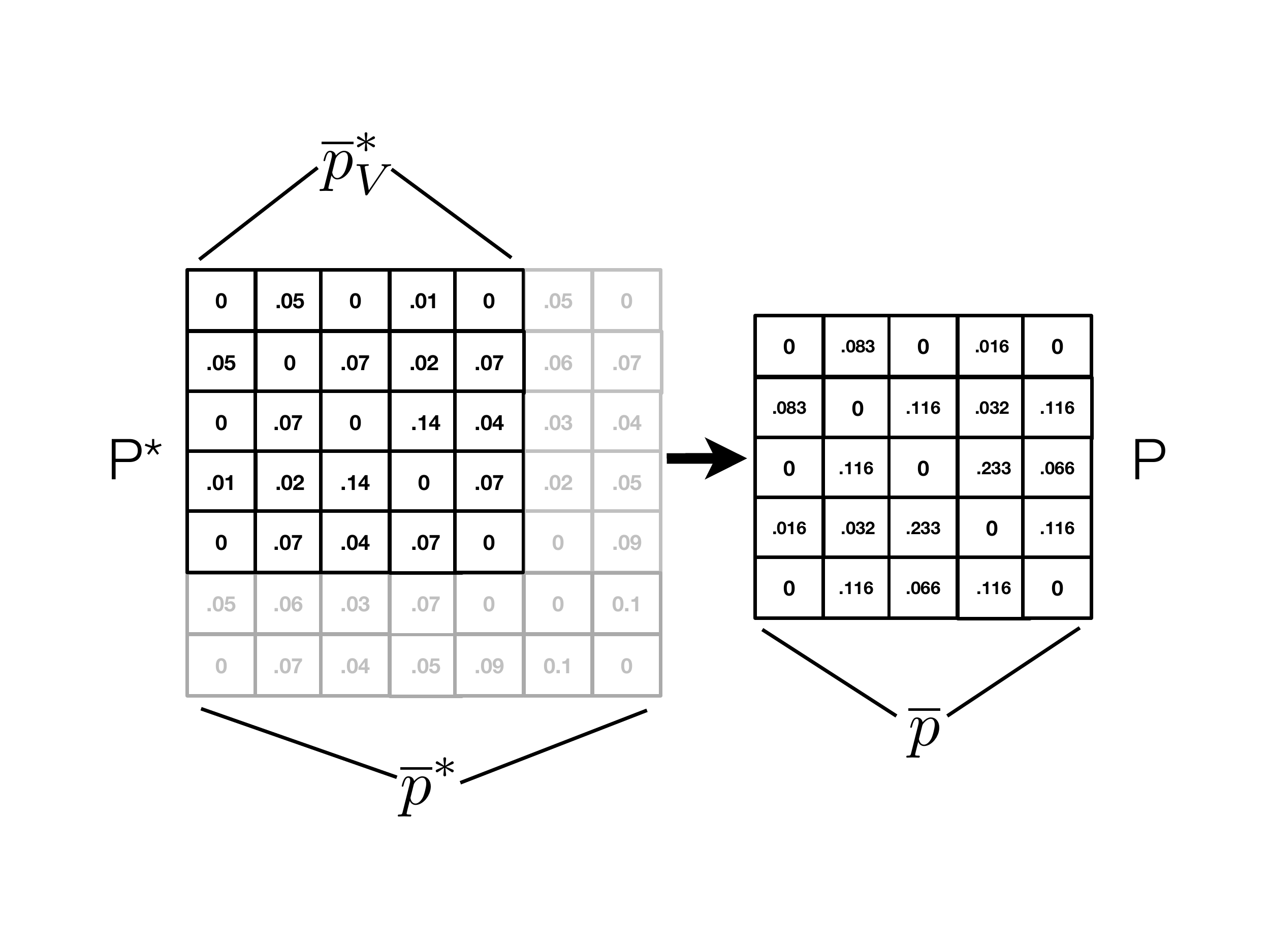}
\vspace{-5mm}
\caption{Visualization of the regions averaged to obtain $\overline{p^*}, \overline{p},$ and $ \overline{p^*}_V$.}
\label{region-visualization}
\end{centering}
\end{figure}


\paragraph{Probability Mass Shift}

\noindent We will now introduce a new consistent statistic called Probability Mass Shift which is a measure of change in the underlying $P^*$ matrices which produced two graphs.  Similar to graph edit distance when used on a dynamic network it is a measure of the rate of change the network is experiencing; unlike graph edit distance it is consistent with respect to the size of the input graphs.

\noindent Let $P^*_1$ and $P^*_2$ be probability distributions over a node set $V^*$.  Define the Probability Mass Shift between $P^*_1$ and $P^*_2$ to be:

\vspace{-4mm}
\begin{small}
\begin{align} \label{eq:ms}
{MS}(P^*_1, P^*_2) &=  \frac{1}{Z_{V^*}} \sum_{ij \in V^*} ( \frac{ {p^*}_{ij,1} }{\overline{p^*_1}} - \frac{ {p^*}_{ij,2} }{\overline{p^*_2}} )^2
\end{align}
\end{small}
\vspace{-2mm}

\noindent where the term $\overline{p^*_x} = \frac{1}{| A^* |} \sum_{ij \in V^*} p^*_{ij}$ refers to the average value of nonzero cells in $P^*_x$ and $Z_{V^*} = {| V^* | \choose 2}$.  

Let the Probability Mass Shift between node subsets $V_1, V_2 \subset V^*$ be:

\vspace{-4mm}
\begin{small}
\begin{align} \label{eq:ms}
{MS}(P_1, P_2) &=  \frac{1}{Z_{V_{\cap}}} \sum_{ij \in V_{\cap}} ( \frac{ {p}_{ij,1} }{\overline{p_1}} - \frac{ {p}_{ij,2} }{\overline{p_2}} )^2
\end{align}
\end{small}
\vspace{-2mm}

\noindent where $V_{\cap}$ is the intersection of $V_1$ and $V_2$, $p_{ij,x} = \frac{p^*_{ij,x}}{\sum_{ij \in V_{\cap}} p^*_{ij,x}}$, and $\overline{p_x} = \frac{1}{| A |} \sum_{ij \in V_{\cap}} p_{ij}$.

Now define the empirical Probability Mass Shift $\widehat{MS}(G_1,G_2)$ to be:


\vspace{-3mm}
\begin{small}
\begin{align} \label{eq:ems}
\widehat{MS}(G_1, G_2) &=  \frac{1}{Z_{V_{\cap}}} \sum_{ij \in V_{\cap}} ( \frac{ \widehat{p}_{ij,1} }{\overline{\widehat{p}_1}} - \frac{ \widehat{p}_{ij,2} }{\overline{\widehat{p}_2}} )^2
\end{align}
\end{small}
\vspace{-2mm}

\noindent where $\overline{\widehat{p}_x} = \frac{1}{| \widehat{A_x} |} \sum_{ij \in V_{\cap}} \widehat{p}_{ij}$, $| \widehat{A_x} | = \sum_{ij \in V_{\cap}} I[ w_{ij,x} > 0]$, and $\widehat{p}_{ij,x} = \frac{w_{ij,x}}{\sum_{ij \in V_{\cap}} w_{ij,x} }$.

\vspace{4mm}
\begin{theorem} 
$\widehat{MS}(G_1,G_2)$ is a size consistent statistic which \\ converges to $MS(P^*_1, P^*_2)$. 
\end{theorem}
\vspace{-1mm}




\vspace{-3mm}
\begin{small}
\allowdisplaybreaks{
\begin{align}
& \lim_{|V_1|,|V_2| \rightarrow |V^*|} MS(P_1,P_2) \nonumber \\ &= \lim_{|V_1|,|V_2| \rightarrow |V^*|} \frac{1}{Z_P} \sum_{ij \in V_{\cap}} ( {| A_1 |}  {p}_{ij,1} - {| A_2 |}   {p}_{ij,2})^2 \nonumber \\ 
&= \lim_{|V_1|,|V_2| \rightarrow |V^*|} \frac{1}{Z_P} \sum_{ij \in V_{\cap}} (\frac{| A_1 |}{\sum_{ij \in V_{\cap} } p^*_{ij,1}  } p^*_{ij,1} - \frac{| A_2 |}{\sum_{ij \in V_{\cap}}p^*_{ij,2} } p^*_{ij,2})^2 \nonumber \\ 
\end{align}
}
\end{small}
\vspace{-4mm}

\noindent As $\frac{| A_x |}{\sum_{ij \in V_{\cap}} p^*_{ij,x}} = \frac{ \sum_{ij \in V_{\cap}} I [ p^*_{ij,x} > 0 ] }{\sum_{ij \in V_{\cap}} p^*_{ij,x}} \approx  \frac{1}{ (\overline{ p^*_x } ) }$, this is an approximation calculated over only a subset of the nodes $V_{\cap}$, denoted as $\frac{1}{\overline{p^*_1} | V_{\cap}}$.  

\vspace{-3mm}
\begin{small}
\allowdisplaybreaks{
\begin{align}
\lim_{|V_1|,|V_2| \rightarrow |V^*|} MS(P_1,P_2) \nonumber \\
&= \lim_{|V_1|,|V_2| \rightarrow |V^*|} \frac{1}{Z_P} \sum_{ij \in V_{\cap}} (\frac{1}{\overline{p^*_1} | V_{\cap}} p^*_{ij,1} - \frac{1}{\overline{p^*_2} | V_{\cap}} p^*_{ij,2})^2 \nonumber \\
&=  \frac{1}{Z_{V^*}} \sum_{ij \in V^*}  \lim_{|V_1|,|V_2| \rightarrow |V^*|} \big( \frac{1}{(\overline{p^*_1} | V_{\cap})^2 } (p^*_{ij,1})^2  +  \frac{1}{(\overline{p^*_2} | V_{\cap})^2 }  (p^*_{ij,2})^2 \nonumber \\ &- 2  \frac{1}{(\overline{p^*_1} | V_{\cap}) }  \frac{1}{(\overline{p^*_2} | V_{\cap}) } p^*_{ij,1} p^*_{ij,2}  \big) 
\end{align}
}
\end{small}
\vspace{-2mm}

\noindent $\overline{p^*_x} | V_{\cap}$ and $(\overline{p^*_x} | V_{\cap})^2$ are the sample mean and square of the sample mean of the value of the $P^*$ cells in $V_{\cap}$ respectively, and according to Slutsky's Theorem their inverses $\frac{1}{\overline{p^*_x} | V_{\cap}}$ and $\frac{1}{(\overline{p^*_x} | V_{\cap})^2}$ converge in probability to $\frac{1}{(\overline{p^*_x})}$ and $\frac{1}{(\overline{p^*_x})^2}$ as $| V_{\cap} | \rightarrow | V^* |$.  

\vspace{-3mm}
\allowdisplaybreaks{
\begin{align}
 \lim_{|V_1|,|V_2| \rightarrow |V^*|} MS(P_1,P_2) 
& = \frac{1}{Z_{V^*}} \sum_{ij \in V^*}  \big( \frac{1}{(\overline{p^*_1} )^2 } (p^*_{ij,1})^2  \nonumber \\ +  \frac{1}{(\overline{p^*_2})^2 }  (p^*_{ij,2})^2 -& 2  \frac{1}{(\overline{p^*_1} ) }  \frac{1}{(\overline{p^*_2} ) } p^*_{ij,1} p^*_{ij,2}  \big) = MS(P^*_1,P^*_2) 
\end{align}
}
\vspace{-3mm}


\vspace{-2mm}
\allowdisplaybreaks{
\begin{align}
  \lim_{| W_1 |, | W_2 | \rightarrow \infty} &\widehat{MS}(G_1, G_2)   \nonumber \\
&= \lim_{| W_1 |, | W_2 | \rightarrow \infty} \frac{1}{Z_P} \sum_{ij \in V_{\cap}}   ( \widehat{| A_1 |}  \widehat{p}_{ij,1} - \widehat{| A_2 |}  \widehat{p}_{ij,2})^2   \nonumber \\
&= \lim_{| W_1 |, | W_2 | \rightarrow \infty} \frac{1}{Z_P} \sum_{ij \in V}  \left(  \widehat{| A_1 |}  \frac{ {w}_{ij,1} }{| W_1 |} - \widehat{| A_2 |}  \frac{ {w}_{ij,2} }{| W_2 |}\right)^2  \nonumber \\
&= \lim_{| W_1 |, | W_2 | \rightarrow \infty} \frac{1}{Z_P} \sum_{ij \in V}  \frac{\widehat{| A_1 |}^2}{| W_1 |^2}  {w}_{ij,1}^2 \nonumber \\ &+  \frac{\widehat{| A_2 |}^2}{| W_2 |^2}  {w}_{ij,2}^2 - 2\frac{ \widehat{| A_1 |} }{| W_1 |}\frac{ \widehat{| A_2 |} }{| W_2 |}{w}_{ij,1}{w}_{ij,2} 
\end{align}
}
\vspace{-2mm}

%
%

Let the minimum value of any nonzero cell in $P_x$ be a finite $\epsilon$.  The probability of any node pair not being sampled from $P_x$ is $(1 - \epsilon)^{| W_x |}$, which is converging to 0.  Once every nonzero cell has been sampled, $|\widehat{A_x}| = |A_x|$, so this term is converging and can be replaced: 

\vspace{-3mm}
\allowdisplaybreaks{
\begin{align}
&=  \lim_{| W_1 |, | W_2 | \rightarrow \infty} \frac{1}{Z_P} \sum_{ij \in V}  \frac{{| A_1 |}^2}{| W_1 |^2} {w}_{ij,1}^2 +  \frac{{| A_2 |}^2}{| W_2 |^2}  {w}_{ij,2}^2  \nonumber \\ &- 2\frac{ {| A_1 |} }{| W_1 |}\frac{ {| A_2 |} }{| W_2 |} {w}_{ij,1}  {w}_{ij,2} \nonumber \\
&= \frac{1}{Z_P} \sum_{ij \in V}  {| A_1 |}^2 {p}_{ij,1}^2 + {| A_2 |}^2 {p}_{ij,2}^2 - 2 {| A_1 |} {| A_2 |}  {p}_{ij,1}  {p}_{ij,2} \nonumber \\
&= MS(P_1,P_2) \square
\end{align}
}
\vspace{-2mm}


\noindent We can improve upon the empirical version of the statistic by calculating the amount of bias for $|W_1|,|W_2|$ values and compensating.  As the expectation of $w_{ij,x}^2$ for any node pair $i,j$ given $ | W_x |$ can be written as:

\vspace{-3mm}
\begin{align}
E_{W_x }[ w^2_{ij,x} ] &= Var(w_{ij,x}) + E_{W_x }[ w_{ij,x} ]^2 \nonumber \\&= Var( Bin(| W_x |, {p}_{ij,x}) ) + E_{W_x }[ Bin(| W_x |, {p}_{ij,x}) ]^2 \nonumber  \\
&= | W_x |{p}_{ij,x}(1-{p}_{ij,x}) + | W_x |^2{p}^2_{ij,x} 
\end{align}
\vspace{-2mm}

\noindent We can rewrite the expectation of the empirical mass shift: 

\vspace{-3mm}
\allowdisplaybreaks{
\begin{align}
&\frac{1}{Z_P} \sum_{ij \in V}  \frac{{(| A_1 |)}^2}{| W_1 |^2}  ( | W_1 |{p}_{ij,1}(1-{p}_{ij,1}) + | W_1 |^2{p}^2_{ij,1} ) \nonumber \\ & + \frac{{(| A_2 |)}^2}{| W_2 |^2}  ( | W_2 |{p}_{ij,2}(1-{p}_{ij,2}) + | W_2 |^2{p}^2_{ij,2} )  \nonumber \\ &- 2\frac{ {| A_1 |} }{| W_1 |}\frac{ {| A_2 |} }{| W_2 |}| W_1 |p_{ij,1}  | W_2 |p_{ij,2} \nonumber \\
 &= \frac{1}{Z_P} \sum_{ij \in V} (| A_1 |p_{ij,1} - | A_2 |p_{ij,2})^2 \nonumber \\ &+ \frac{1}{| W_1 |}   | A_1 |^2p_{ij,1}(1-p_{ij,1}) +  \frac{1}{| W_2 |}   | A_2 |^2p_{ij,2}(1-p_{ij,2}) \nonumber \\
&= MS(P_1, P_2) + \frac{1}{| W_1 |}   | A_1 |^2 p_{ij,1}(1-p_{ij,1}) \nonumber \\ &+  \frac{1}{| W_2 |}   | A_2 |^2 p_{ij,2} (1-p_{ij,2})
\end{align}
}
\vspace{-2mm}

\noindent which is equal to $MS(P_1, P_2)$ plus a bias term.  

%


Although we have shown the empirical mass shift to be size consistent, we can improve the rate of convergence by subtracting our estimate of the bias:

\vspace{-3mm}
\begin{small}
\allowdisplaybreaks{
\begin{align}
\widehat{MS}(G_1, G_2) &= ( | \widehat{A_1} |  \widehat{p}_{ij,1} -  | \widehat{A_2} |  \widehat{p}_{ij,2})^2 - \frac{1}{| W_1 |} | \widehat{A_1} |^2 \widehat{p}_{ij,1}(1-\widehat{p}_{ij,1})  \nonumber \\ &- \frac{1}{| W_2 |}  | \widehat{A_2} |^2 \widehat{p}_{ij,2}(1-\widehat{p}_{ij,2}) 
\end{align}
}
\end{small}
\vspace{-2mm}

\paragraph{Probabilistic Degree Distance} 


The Probabilistic Degree Distance is a delta statistic that measures the difference between the degree distributions of two graphs in a size-consistent manner.  It is defined as:

\vspace{-3mm}
\begin{small}
\begin{align} 
& {PDD}(P^*_1, P^*_2) = \sum_{bin_k \in Bins} (\frac{1}{{|V^*|}} \sum_{i \in V^*} I[ PD_i(P^*_1) \in bin_k ] \nonumber \\ &- \frac{1}{{|V^*|}} \sum_{i \in V^*} I[  PD_i(P^*_2) \in bin_k ] )^2
\end{align}
\end{small}
\vspace{-2mm}

\noindent where $Bins$ is a consecutive sequence of equal size bins which encompass all $PD_i$ values, $PD_i(P^*) = \frac{1}{\overline{p^*}}\sum_{j \in V^*} p^*_{ij}$ is the \textit{Probabilistic Degree} of a node $i$, and $\overline{p^*} = \frac{1}{|V^*|}$ is the average probability mass per node.  We can rewrite the probabilistic degree as $PD_i(P^*) = |V^*| \sum_{j \in V^*} p^*_{ij}$.   

As the name suggests the probabilistic degree is a normalized version of node degree, and the distribution of probabilistic degrees replaces the standard degree distribution.  Before we can begin our proofs about the consistency of the PDD, we must first analyze the behavior of this probabilistic degree distribution.

The probabilistic degree of a node can be represented as the mean of the masses in the cells of that node in $P^*_k$:

\begin{small}
\begin{align}
F^*_{row,k}(x) &= \sum_{i \in V^*_k} I[ x > \frac{ \sum_{j \neq i \in V^*_k} p^*_{ij,k}}{(\overline{p^*}_{row,k})} ]
\end{align}
\end{small}

Where $\overline{p^*}_{row,k} = \frac{\sum_{ij \in V^*_k} p^*_{ij,k}}{N^*} =  \frac{1}{N^*}$.  We can rewrite the CDF as 

\begin{small}
\begin{align}
F^*_{row,k}(x) &= \sum_{i \in V^*_k} I[ x > N^* * \sum_{j \neq i \in V^*_k} {p^*_{ij,k}} ]
\end{align}
\end{small}







%
%
%
%
%
%

As before let us investigate the effect of node sampling by calculating the value of the statistic using the normalized probabilities $P_1, P_2$ on node subsets $V_1, V_2$:

\begin{small}
\begin{align}
F_{row,k}(x) &= {\sum_{i \in V_k} I[ x > \sum_{j \neq i \in V_k} \frac{p_{ij,k}}{(\overline{p}_{row,k})} ]}
\end{align}
\end{small}




Where $\overline{p}_{row,k} = \frac{1}{N}$.

Since $p_{ij,k} = \frac{p^*_{ij,k}}{\sum_{ij \in V_k} p^*_{ij,k}}$, we can rewrite $F_{row,k}(x)$ as

\begin{small}
\begin{align}
F_{row,k}(x) &= {\sum_{i \in V_k} I[ x >  N * \frac{1}{\sum_{jl \in V_k} p^*_{jl,k}}*{\sum_{j \neq i \in V_k} p^*_{ij,k}} ]} \nonumber \\
&= \sum_{i \in V_k} I[ x > \frac{ 1 }{\overline{p^*}_{row,k} | V_k} * \sum_{j \neq i \in V_k} p^*_{ij,k} ]
\end{align}
\end{small}

Where $\overline{p^*}_{row,k} | V_k$ is the mean probability mass per row in $P^*$ excluding any cells/rows that don't belong in the set $V_k$.  


If we take the expectation of the PDF for a particular $i$ with respect to the node sample $V_k$:

\begin{small}
\begin{align} 
&E_{V_k} [  \frac{1}{\overline{p^*}_{row,k} | V_k}*{\sum_{j \neq i \in V_k} p^*_{ij,k}} ] \nonumber \\
= &E_{N} \bigg[    E_{V_k \big| N} [ \frac{1}{\overline{p^*}_{row,k} | V_k}*{\sum_{j \neq i \in V_k} p^*_{ij,k}} ] \bigg] \nonumber \\
= &E_{N} \bigg[    E_{V_k \big| N} [ \frac{1}{\overline{p^*}_{row,k} | V_k}]*E_{V_k \big| N} [ {\sum_{j \neq i \in V_k} p^*_{ij,k}} ] \bigg] \nonumber \\
\end{align}
\end{small}

If we apply Wald's equation to $E_{V_k \big| N} [ {\sum_{j \neq i \in V_k} p^*_{ij,k}} ]$ we obtain

\begin{small}
\begin{align} 
E_{V_k \big| N} [ {\sum_{j \neq i \in V_k} p^*_{ij,k}} ] = E_{V_k \big| N} [ | A_{row_i,k} | ] * \overline{p^*_{ij,k}}
\end{align}
\end{small}

If we assume that probability mass in row $i$ is evenly distributed amongst the columns, then the fraction of row mass in $V_k$ versus $V^*_k$ is equal to the fraction of their sizes:

\begin{small}
\begin{align} 
 E_{V_k \big| N} [ | A_{row_i,k} | ] * \overline{p^*_{ij,k}} = \frac{N}{N^*} * | A^*_{row_i,k} | * \overline{p^*_{ij,k}} = \frac{N}{N^*} * {\sum_{j \neq i \in V^*_k} p^*_{ij,k}}
\end{align}
\end{small}

Now if we approximate $E_{V_k \big| N} [ \frac{1}{\sum_{jl \in V_k} p^*_{jl,k}}]$ with a taylor expansion we obtain:

\begin{small}
\begin{align} 
&E_{V_k \big| N} [ \frac{1}{\overline{p^*}_{row,k} | V_k}] = \nonumber \\
&\frac{1}{E_{V_k \big| N} [\overline{p^*}_{row,k} | V_k]}  + \frac{1}{(E_{V_k \big| N} [\overline{p^*}_{row,k} | V_k])^3}*Var(\overline{p^*}_{row,k} | V_k)
\end{align}
\end{small}

If we make the same assumption that row mass is roughly evenly distributed across the columns of the matrix, \[ E_{V_k \big| N} [\overline{p^*}_{row,k} | V_k] = \frac{N}{N^*} * \overline{p^*}_{row,k}\].  We can also rewrite the variance term as \[ Var( \frac{\sum_{ij \in V_k} p^*_{ij,k}}{N}) = \frac{1}{N^2} * Var(\sum_{ij \in V_k} p^*_{ij,k}) \]  

Putting this together we have

\begin{small}
\begin{align} 
&E_{V_k \big| N} [ \frac{1}{\overline{p^*}_{row,k} | V_k}] = \nonumber \\
&\frac{N^*}{N * \overline{p^*}_{row,k} }   + \frac{N^{*3}}{N^3 * \overline{p^*}_{row,k}^3} *\frac{1}{N^2} * Var(\sum_{ij \in V_k} p^*_{ij,k})
\end{align}
\end{small}


A typical degree distribution of a social network tends to be a power-law in type, which means that a handful of nodes have a large degree and most have a very small degree.  Again we will assume that covariance between edge probabilities are limited to within row/column pairs.  If we assume that the majority of nodes have a sub $O(N^{1/2})$ number of neighbors then the $1/N^2$ term will be greater than the number of covariance terms and the bias from these nodes will converge to 0.  Likewise, if the handful of high degree nodes have a sub $O(N)$ number of neighbors their covariance terms will also be less than $1/N^2$ and the bias will also converge to 0.


\begin{small}
\begin{align} 
\frac{N^*}{N * \overline{p^*}_{row,k} }*\frac{N}{N^*} * {\sum_{j \neq i \in V^*_k} p^*_{ij,k}} =  \frac{\sum_{j \neq i \in V^*_k} p^*_{ij,k}}{\overline{p^*}_{row,k}}
\end{align}
\end{small}

Which is the true PDF calculated on $P^*_k$.  This means that the CDF of the row masses converges to the correct distribution as $N$ approaches $N^*$. $\square$

%

\vspace{5mm}

\textbf{Degree Distribution Edge Bias}

Now we will consider the CDF of the empirical row mass calculated from an edge sampling $W_k$:

\begin{small}
\begin{align} 
\widehat{F}_{row,k}(x) &= \sum_{i \in V} I[ x > \frac{\sum_{j \neq i \in \widehat{V}_k} \widehat{p}_{ij,k}}{\overline{\widehat{p}_{row,k}} } ]
\end{align}
\end{small}

Where $\overline{\widehat{p}_{row,k}} = \frac{1}{\widehat{N}_k}$ and $\widehat{V}_k$ is the set of nodes that have at least one edge in $W_k$.  

If we take the expectation of the PDF with respect to $W_k$ we obtain

\begin{small}
\begin{align} 
E_{W_k} [ \widehat{N}_k * \sum_{j \neq i \in \widehat{V}_k} \widehat{p}_{ij,k} ]
\end{align}
\end{small}

As all rows in $P_k$ have at least one cell with nonzero probability, as $| W_k | \uparrow$, $\widehat{V}_k \rightarrow V_k$ as the probability of sampling at least one edge from every node approaches 1.  So if we take the limit as $| W_k |$ increases:

\begin{small}
\begin{align} 
&\lim_{| W_k | \rightarrow \infty}E_{W_k} [ \widehat{N}_k * \sum_{j \neq i \in \widehat{V}_k} \widehat{p}_{ij,k} ] \nonumber \\
= &E_{W_k} [ {N}_k * \sum_{j \neq i \in {V}_k} \widehat{p}_{ij,k} ] \nonumber \\
= & {N}_k * \sum_{j \neq i \in {V}_k} E_{W_k} [ \widehat{p}_{ij,k} ] \nonumber \\
= & {N}_k * \sum_{j \neq i \in {V}_k} p_{ij,k} 
\end{align}
\end{small}

So

\begin{small}
\begin{align} 
&\lim_{| W_k | \rightarrow \infty}E_{W_k} [ \widehat{F}_{row,k}(x) ] = F_{row,k}(x)
\end{align}
\end{small}
%
%
%


Now let us define the empirical probabilistic degree distance and analyze its behavior.  The empirical probabilistic degree is $\widehat{PD}_i(G) = {{|V|}}  \sum_{j \in V} {\widehat{p}}_{ij}$ and the empirical version of the delta statistic on $G_1,G_2$ is:

\vspace{-3mm}
\begin{small}
\begin{align} 
& \widehat{PDD}(G_1, G_2) = \sum_{bin_k \in Bins} (\frac{1}{{|V|}} \sum_{i \in V} I[ \widehat{PD}_i(G_1) \in bin_k ] \nonumber \\ &- \frac{1}{{|V|}} \sum_{i \in V^*} I[  \widehat{PD}_i(G_2) \in bin_k ] )^2
\end{align}
\end{small}
\vspace{6mm}

\begin{theorem} $\widehat{PDD}(G_1, G_2)$ is a size consistent statistic \\ which converges to ${PDD}(P^*_1, P^*_2)$. \end{theorem}

\noindent First take the limit of the Probabilistic Degree for a node as $|W|$ increases: 

\vspace{-3mm}
\begin{small}
\allowdisplaybreaks{
\begin{align} 
\lim_{|W| \rightarrow \infty}  \widehat{PD}_i(G) &= \lim_{|W| \rightarrow \infty}  {{|V|}}  \sum_{j \in V} \widehat{p}_{ij} \nonumber \\
= \lim_{|W| \rightarrow \infty}  {|V|}  \sum_{j \in V} \frac{{w}_{ij}}{|W|}
&= |V| \sum_{j \in V} p_{ij} = PD_i(P)
\end{align}
}
\end{small}
\vspace{6mm}

\noindent If we take the same limit over the Probabilistic Degree Difference we obtain:

\vspace{-3mm}
\begin{small}
\begin{align} 
\lim_{|W| \rightarrow \infty} & \widehat{PDD}(G_1, G_2) = \sum_{bin_k \in Bins} (\frac{1}{{|V_1|}} \sum_{i \in V_1} I[ {PD}_i(P_1) \in bin_k ] \nonumber \\ &- \frac{1}{{|V_2|}} \sum_{i \in V_2} I[  {PD}_i(P_2) \in bin_k ] )^2 = PDD(P_1,P_2)
\end{align}
\end{small}
\vspace{-2mm}

\noindent If we take the limit as $|V| \rightarrow |V^*|$ of $PD_i(P)$:

\vspace{-3mm}
\begin{small}
\begin{align} 
\lim_{|V| \rightarrow |V^*|} PD_i(P) &= \lim_{|V| \rightarrow |V^*|}  |V| \sum_{j \in V} p_{ij}  = \lim_{|V| \rightarrow |V^*|} \frac{|V|}{\sum_{ij \in V} p^*_{ij} } \sum_{j \in V} p^*_{ij}
\end{align}
\end{small}
\vspace{-2mm}

\noindent $\frac{|V|}{\sum_{ij \in V} p^*_{ij} }$ can be rewritten as $\frac{1}{\bar{p^*} | V}$ where $\bar{p^*} | V$ is the average probability mass per node in $V$.  As this is an inverse mean, it will converge to the true value $\frac{1}{\bar{p^*}}$, and therefore $\lim_{|V| \rightarrow |V^*|} PD_i(P) = \frac{1}{\bar{p^*} } \sum_{j \in V^*} p^*_{ij} = |V^*| \sum_{j \in V^*} p^*_{ij} = PD_i(P^*)$.

%

If we take the limit on the $PDD$ we obtain a similar result:

\vspace{-3mm}
\begin{small}
\begin{align} 
\lim_{|V| \rightarrow |V^*|} PDD(P_1,P_2) =  \sum_{bin_k \in Bins} &(\frac{1}{{|V^*|}} \sum_{i \in V^*} I[ {PD}_i(P^*_1) \in bin_k ] \nonumber \\ - \frac{1}{{|V^*|}} \sum_{i \in V^*} I[  {PD}_i(P_2) \in bin_k ] )^2 \nonumber &= PDD(P^*_1,P^*_2)
\end{align}
\end{small}
\vspace{-3mm}

\noindent If we take the expectation over $V$:

\vspace{-3mm}
\allowdisplaybreaks{
\begin{align} 
E[ PDD(P_1,&P_2) ] = E[ \sum_{bin_k \in Bins} (\frac{1}{{|V_1|}} \sum_{i \in V_1} I[ {PD}_i(P_1) \in bin_k ] \nonumber \\ 
&- \frac{1}{{|V_2|}} \sum_{i \in V_2} I[  {PD}_i(P_2) \in bin_k ] )^2 ] \nonumber \\
=& \sum_{bin_k \in Bins} E[  (\frac{1}{{|V_1|}} Bin(|V_1|, p_{k,1}) -  \frac{1}{{|V_2|}} Bin(|V_2|, p_{k,2}) )^2 ] \nonumber \\
=& \sum_{bin_k \in Bins} E[  (\frac{1}{{|V_1|}} Bin(|V_1|, p_{k,1}))^2 ] + E[ (\frac{1}{{|V_2|}} Bin(|V_2|, p_{k,2}) )^2 ] \nonumber \\ 
&- 2 E[\frac{1}{{|V_1|}} Bin(|V_1|, p_{k,1})] E[\frac{1}{{|V_2|}} Bin(|V_2|, p_{k,2})] \nonumber \\
\end{align}
}
\vspace{-3mm}

\noindent where $p_{bin k,x}$ is the probability of any node selected from $V_x$ belonging to bin $k$.  Using the same approach as with Mass Shift, we obtain a bias correction of $-\frac{ (p_{k,1})(1 - p_{k,1})  }{|V_2|} - \frac{ (p_{k,2})(1 - p_{k,2})  }{|V_2|}$.

\paragraph{Triangle Probability} 


\noindent As the name suggests, the triangle probability (TP) statistic is an approach to capturing the transitivity of the network and an alternative to traditional clustering coefficient measures.  Define the triangle probability as: 

\vspace{-3mm}
\begin{small}
\begin{align} \label{eq:tp}
TP(P^*) &= \frac{1}{Z^*} \sum_{ijk \in V^*} (\frac{1}{(\overline{p^*})})^3  p^*_{ij}  p^*_{ik}  p^*_{jk} \nonumber \\
&=\frac{1}{Z^*} \sum_{ijk \in V^*} | A^* |^3 p^*_{ij}  p^*_{ik}  p^*_{jk}
\end{align}
\end{small}
\vspace{-4mm}

\noindent The empirical version on $G$ is:

\vspace{-4mm}
\begin{small}
\begin{align}
\widehat{TP}(G) &= \frac{1}{Z} \sum_{ijk \in V} \widehat{| A |}^3 \widehat{p}_{ij}  \widehat{p}_{ik}  \widehat{p}_{jk}
\end{align}
\end{small}
\vspace{-2mm}

\noindent where $Z^* = {| V^* | \choose 3}$ and $Z = {| V | \choose 3}$.

\begin{theorem} $\widehat{TP}(G)$ is a size consistent statistic which converges to ${TP}(P^*)$. \end{theorem}

\noindent As before, if we take the limit with increasing $|W|$: 

\vspace{-3mm}
\begin{small}
\allowdisplaybreaks{
\begin{align}
\lim_{|W| \rightarrow \infty} \frac{1}{Z} \sum_{ijk \in V} \widehat{| A |}^3 \widehat{p}_{ij}  \widehat{p}_{ik}  \widehat{p}_{jk} 
&= \lim_{|W| \rightarrow \infty} \frac{1}{Z} \sum_{ijk \in V} \frac{\widehat{| A |}^3}{|W|^3} {w}_{ij}  {w}_{ik}  {w}_{jk} \nonumber \\
= \frac{1}{Z} \sum_{ijk \in V} {{| A |}^3} {p}_{ij}  {p}_{ik}  {p}_{jk} &= TP(P)
\end{align}
}
\end{small}
\vspace{-2mm}

\noindent Now if we take the limit as $|V| \rightarrow |V^*|$:

\vspace{-3mm}
\begin{small}
\allowdisplaybreaks{
\begin{align}
&\lim_{|V| \rightarrow |V^*|} \frac{1}{Z} \sum_{ijk \in V} {{| A |}^3} {p}_{ij}  {p}_{ik}  {p}_{jk} \nonumber \\
&= \lim_{|V| \rightarrow |V^*|} \frac{1}{Z} \sum_{ijk \in V} \frac{{| A |}^3}{(\sum_{ij \in V} p^*_{ij} )^3} p^*_{ij}  p^*_{ik}  p^*_{jk} \nonumber \\
&= \lim_{|V| \rightarrow |V^*|} \frac{1}{Z} \sum_{ijk \in V} \frac{1}{(\bar{p^*} | V )^3} p^*_{ij}  p^*_{ik}  p^*_{jk} 
\end{align}
}
\end{small}
\vspace{-2mm}

\noindent Similar to the approach before, $ \frac{1}{(\bar{p^*} | V )^3}$ converges to $\frac{1}{(\bar{p^*}  )^3}$ by Slutsky's Theorem, so the final limit is

\vspace{-3mm}
\begin{small}
\begin{align}
&= \frac{1}{Z^*} \sum_{ijk \in V^*} \frac{1}{(\bar{p^*}  )^3} p^*_{ij}  p^*_{ik}  p^*_{jk} = TP(P^*) \square
\end{align}
\end{small}
\vspace{-2mm}

As with the Mass Shift, let us take the expectation w.r.t. $|W|$ and see if we can improve the rate of convergence with a bias correction:

\vspace{-3mm}
\begin{small}
\begin{align}
E[ \widehat{TP}(G) ]  &= E \biggr[  \frac{1}{Z} \sum_{ijk \in V} (\widehat{| A |})^3  \widehat{p}_{ij}  \widehat{p}_{ik}  \widehat{p}_{jk} \biggr] \nonumber \\
&=   \frac{1}{Z} \sum_{ijk \in V} E \biggr[ (\frac{\widehat{| A |}}{| W |})^3  {w}_{ij}  {w}_{ik}  {w}_{jk} \biggr]
\end{align}
\end{small}
\vspace{-2mm}

\noindent As before, let us assume that we have enough edge samples so that $\widehat{| A |} = | A |$:

\vspace{-3mm}
\begin{small}
\begin{align}
  \frac{1}{Z} \sum_{ijk \in V}  E\biggr[ (\frac{| A |}{| W |})^3   {w}_{ij}  {w}_{ik}  {w}_{jk} \biggr]   
\end{align}
\end{small}
\vspace{-2mm}



\noindent The quantity $E[  \frac{1}{| W |^3} {w}_{ij}  {w}_{ik}  {w}_{jk} ]$ can be calculated with

\vspace{-3mm}
\begin{small}
\allowdisplaybreaks{
\begin{align*}
E\biggr[ \frac{1}{| W |^3}  &w_{ij} w_{ik} w_{jk} \biggr]  = \frac{1}{| W |^3} E[w_{ij} w_{ik} w_{jk}]  \nonumber \\
=& \frac{1}{| W |^3}  E[w_{ij} w_{ik}] E[w_{jk}] - cov(w_{ij}w_{ik}, w_{jk}) \\
=& \frac{1}{| W |^3}  E[w_{ij}] E[w_{ik}] E[w_{jk}] \nonumber \\ &- E[w_{jk}] cov(w_{ij},w_{ik}) - cov(w_{ij} w_{ik}, w_{jk}) \\
=& \frac{1}{| W |^3}  | W |^3 p_{ij} p_{ik} p_{jk} + | W |^2 p_{ij} p_{ik} p_{jk} - cov(w_{ij} w_{ik}, w_{jk}) 
\end{align*}
}
\end{small}
\vspace{-2mm}


\noindent The covariance term can be expanded with the formula for products of random variables \cite{bohrnstedt}:

\vspace{-3mm}
\begin{small}
\allowdisplaybreaks{
\begin{align*}
 cov(w_{ij} \cdot w_{ik},&w_{jk}) = \nonumber \\ & E[w_{ij}]cov(w_{ik},w_{jk}) + E[w_{ik}]cov(w_{ij}w_{jk}) \\ &+ E[(w_{ij} - E[w_{ij}])(w_{ik} - E[w_{ik}])(w_{jk} - E[w_{jk}]) ] \\
=& -2 | W |^2p_{ij}p_{ik}p_{jk} \\&+E\biggr[ w_{ij}w_{ik}w_{jk} - E[w_{ij}]w_{ik}w_{jk} \nonumber \\ &- E[w_{ik}]w_{ij}w_{jk}  - E[w_{jk}]w_{ik}w_{ij} \\ & + E[w_{ij}]E[w_{ik}]w_{jk} + E[w_{jk}]E[w_{ik}]w_{ij}  \\& + E[w_{ij}]E[w_{jk}]w_{ik}- E[w_{ij}]E[w_{jk}]E[w_{ik}] \biggr] \\
=& -2| W |^2p_{ij}p_{ik}p_{jk} + E[w_{ij} w_{ik} w_{jk}] \\&- | W |p_{ij}E[w_{ik}w_{jk}] - | W |p_{ik}E[w_{ij}w_{jk}] \nonumber \\ &- | W |p_{jk}E[w_{ik}w_{ij}] \\&+ 3| W |^3p_{ij}p_{ik}p_{jk} - | W |^3p_{ij}p_{ik}p_{jk} \\
=& -2| W |^2p_{ij}p_{ik}p_{jk} \nonumber \\ & + E[w_{ij}w_{ik}w_{jk}] - 3| W |^3p_{ij}p_{ik}p_{jk} \\&+ | W |p_{ij}cov(w_{ik},w_{jk}) \nonumber \\ & + | W |p_{ik}cov(w_{ij},w_{jk}) + | W |p_{jk}cov(w_{ik},w_{ij}) \\&+ 3| W |^3p_{ij}p_{ik}p_{jk} - | W |^3p_{ij}p_{ik}p_{jk} \\
=& -2| W |^2p_{ij}p_{ik}p_{jk} + E[w_{ij}w_{ik}w_{jk}]  \nonumber \\ &- 3| W |^2p_{ij}p_{ik}p_{jk} - | W |^3p_{ij}p_{ik}p_{jk} \\
=& -5| W |^2p_{ij}p_{ik}p_{jk} + E[w_{ij}w_{ik}w_{jk}] - | W |^3 p_{ij}p_{ik}p_{jk} 
\end{align*}
}
\vspace{-2mm}
\end{small}


\noindent By plugging the covariance into the original equation we obtain:


\vspace{-3mm}
\allowdisplaybreaks{
\begin{align*}
E\biggr[ \frac{1}{| W |^3}  w_{ij} w_{ik} w_{jk} \biggr]  
=&  \frac{1}{| W |^3} ( | W |^3p_{ij} p_{ik} p_{jk} + | W |^2 p_{ij} p_{ik} p_{jk}  \\
+ 5| W |^2p_{ij}p_{ik}p_{jk} &- E[w_{ij}w_{ik} w_{jk}]
+ | W |^3p_{ij}p_{ik}p_{jk}  )\\
2 E[ \frac{1}{| W |^3}  w_{ij} w_{ik} w_{jk} ] 
=&  \frac{1}{| W |^3} ( 2| W |^3p_{ij} p_{ik} p_{jk} + 6| W |^2 p_{ij} p_{ik} p_{jk} ) \\
E[ \frac{1}{| W |^3}  w_{ij} w_{ik} w_{jk} ] =&  \frac{1}{| W |^3}(  | W |^3p_{ij} p_{ik} p_{jk} + 3| W |^2 p_{ij} p_{ik} p_{jk} )\\
=&  p_{ij} p_{ik} p_{jk}  + \frac{3}{| W |}{p}_{ij}{p}_{ik}{p}_{jk}
\end{align*}
}
\vspace{-2mm}


\noindent So the bias term is $  \frac{3}{| W |}{p}_{ij}{p}_{ik}{p}_{jk}$.  If we subtract the empirical version of this term to compensate, we obtain the corrected empirical Triangle Probability:

\vspace{-3mm}
\begin{small}
\begin{align*}
\widehat{TP}(G) &= \frac{1}{Z} \sum_{ijk \in V} \widehat{| A |}^3 ( \widehat{p}_{ij}  \widehat{p}_{ik}  \widehat{p}_{jk} - \frac{3}{| W |}\widehat{p}_{ij}\widehat{p}_{ik}\widehat{p}_{jk} )
\end{align*}
\end{small}
\vspace{-2mm}

\section{Anomaly Detection Process}

In order to perform the anomaly detection on a dynamic network the collection of messages need to first be converted into a sequence of graph instances.  As each message consists of a pair of nodes and an associated timestamp, after picking a time step width $\Delta$ the graph at each sequential time step $t$ is created by adding all messages falling between $t$ and $t+\Delta$ to matrix $W_t$, producing a sequence of graphs.  The algorithm is described in Figure~\ref{alg:graphprocess}.   

Then, a statistic value needs to be calculated at every graph instance in the stream.  As the length of the stream is usually short compared to the size of the graphs, the computational complexity depends on the cost of calculating the network statistics on the largest graph instance.  In order to calculate our consistent statistics $\widehat{P}_t$ must be estimated, which is easily obtained by normalizing the observed messages $W$.  Then the network statistic scores are calculated at each time step.  This generates a set of standard time series which can be analyzed with traditional time series anomaly detection techniques.

Selection of a proper $\Delta$ time step width is crucial.  Due to the nature of size-consistent statistics larger values of $\Delta$ will reduce the error associated with statistical bias, but larger values also reduce the granularity of the detection algorithm making it harder to pinpoint the exact time that the anomaly occurred.

\begin{figure}[t!]
\begin{algorithmic}
\STATE $GraphProcess(messages, \Delta):$ 
\\ \hrulefill
\STATE $t_{start} = 0, t_{end} = \Delta$
\WHILE{$t_{start} <$ last timestamp in $messages$}
\STATE $W = [], V = \{ \}$
\FOR{$m_{ij,t}$ in $messages$}
\IF{$t_{start}<  t < t_{end}$}
\IF{$i$ not in $V$}
\STATE add $i$ to $V$, $w_{i,x} = 0$, $w_{x,i} = 0$
\ENDIF
\IF{$j$ not in $V$}
\STATE add $j$ to $V$, $w_{j,x} = 0$, $w_{x,j} = 0$
\ENDIF
\STATE $w_{ij}$ ++
\ENDIF
\ENDFOR
\STATE return $G_{t_{start}} = \{ W, V \}$
\STATE return $\widehat{P}_{t_{start}} = W / | W |$  
\STATE $t_{start} += \Delta, t_{end} += \Delta$
\ENDWHILE
\\ \hrulefill
\end{algorithmic}
\vspace{3mm}
\caption{Creation of the dynamic graph sequence from message stream using time step width $\Delta$.}
\label{alg:graphprocess}


\end{figure}%

Now that we have a stream of graphs we can perform the anomaly detection process.  For every time step $t$ the graph at $G_t$ becomes the test graph and the graphs $G_{t-1}, G_{t-2}... G_{t-k}$ become the null distribution examples (here we use $k=50$).  By applying $\widehat{S}_k$ to each graph we obtain both the test point and the null distribution.  Given a certain p-value $\alpha$, we then set the critical points to be the values which reject the most extreme $\alpha/2$ values from the null distribution on both sides.  If the test point $\widehat{S}_k(G_t)$ falls outside of these critical points we can reject the null hypothesis and raise an anomaly flag.  This detection algorithm is described in \ref{alg:anomalydetection}.


\subsection{Smoothing}

Rather than calculating delta statistics using a weighted matrix $W_{t-1}$ which contains only the communications of the immediately prior time step, an aggregate of prior time steps $W_{t-k} ... W_{t-1}$ can be used by simply calculating the average weighted matrix $\overline{W}$ from $W_{t-k} ... W_{t-1}$ and then calculating $S_k(W_t, \overline{W})$ as the delta statistic.  The advantage of this approach is that it measures the distance of the current behavior from the average behavior seen in a range of recent past timesteps, and as such is less susceptible to flagging time $t$ due to an outlier in $W_{t-1}$.  

\begin{figure}[h!]
\begin{algorithmic}
\STATE $AnomalyDetection(G_1, G_2, ... G_t, S_k, \alpha):$ 
\\ \hrulefill
\FOR{$i$ in $50...t$ }
\FOR{$j$ in $1...i-1$}
\STATE  Add $S_k(G_j)$ to $NullDistr$
\ENDFOR
\STATE $CriticalPoints = CalcCPs(NullDistr,\alpha)$
\IF{$S_k(G_i)$ outside $CriticalPoints$ }
\STATE Generate Anomaly Flag at time $i$
\ENDIF
\ENDFOR
\\ \hrulefill 
\\
\end{algorithmic}
\vspace*{2mm}
\caption{ Anomaly detection procedure for a graph stream $\{ G_1 ... G_t \}$, graph statistic $S_k$ and p-value $\alpha$.}
\label{alg:anomalydetection}


\end{figure}%

Another smoothing option is to use a moving window approach with overlapping time steps, i.e. calculate $W_t, W_{t+\delta}, W_{t+2*\delta}...$ where $W_{t+\delta}$ starts at time $t+\delta$ and ends at time $t+\delta+\Delta$.  This effectively allows for a larger time step without sacrificing granularity, as it should be straightforward to find which $\delta$-wide time span that an anomaly occurred in.

A prior edge weight value for the cells of $W_t$ is another option.  Instead of using $W_t$, one can use $W'_t$ where $w'_{ij,t} = w_{ij,t} + c$ for some value of $c$.  In general, $c$ should be small, usually less than 1, as this prior value adds $c |V_t|^2$ total weight to the matrix and $c |V_t|^2 << |W_t|$ in the ideal case.  Larger values of $c$ can easily wash out the actual network behavior leading all of the graph examples to seem uniform.


So far the $P$ matrices have been estimated with a frequentist approach using the observed message frequencies to estimate the probabilities.  If one desires to assign a prior distribution to the $P$ matrix, a Bayesian approach is easily implemented by choosing a Beta distribution for each cell in $P$ and using them as conjugate priors for normalized binomial distributions using the observed message frequencies as the evidence.  The reason we did not utilize this approach is because it is difficult to choose proper prior distributions: due to the sparsity of most networks the vast majority of cells in $P$ are zero.  Similar to the prior edge weight approach assigning a nonzero prior to all cells in $P$ tends to dilute the network, but deciding which cells to assign a zero prior probability is nontrivial.  Because 0 is the natural value for most pairs of nodes in the network trying to smooth by assigning a non-zero prior to all these node pairs is detrimental.

\subsection{Complexity Analysis}

Statistics like probability mass and probabilistic degree can be calculated at each step in $O( |A_t|)$ time, making their overall complexity $O( |A_t| t)$, where $|A_t|$ is the number of nonzero elements in $W_t$.  Triangle probability, on the other hand, is more expensive as some triangle counting algorithm must be applied.  The fastest counting algorithms typically run in $O(|V|^k)$ time where $2 < k < 3$, making the overall complexity $O(|V|^k  t)$ for the whole stream.  However, if we make the assumption that the maximum number of neighbors of any node is bounded by $n_{max}$, we can approximate the triangle count with $O(|V|  n_{max}^2  t)$.  Note that any other statistic-based approach such as Netsimile that utilizes triangle count or clustering coefficient must make the same approximations in order to run in linear time.  

\section{Experiments}

Now that we have established the properties of size-consistent and -inconsistent statistics we will show the tangible effects of these statistical properties using synthetic, semi-synthetic, and real-world datasets.  The objective for the synthetic and semi-synthetic experiments is to maximize the true positive detection rate (where a true positive is flagging a graph generated with anomalous parameters) and minimize the false positive rate (where a false positive is flagging a graph with unusual edge count or node count but generated with normal parameters).  The real-world experiments will be an exploratory analysis, demonstrating how to discover and explain events in a real-world dynamic graph.  

We will compare each of the consistent statistics to the conventional one they were intended to replace: graph edit distance for probability mass shift, degree distribution difference to probabilistic degree difference, and Barrat weighted clustering to triangle probability.  In addition we will also compare the performance of the consistent statistics to Netsimile and Deltacon.  Netsimile is an aggregate statistic which attempts to measure graph differences in a variety of dimensions and as such can be applied to find many types of anomalies.  Deltacon on the other hand measures graph differences through the distances between nodes in the graphs and attempts to find anomalies of an entirely different type than the consistent statistics.


\subsection{Synthetic Data Experiments} 

In order to create data with specific known properties we used generative graph models.  There are four types of graphs generated: 
\begin{enumerate}
\item normal graph examples which are used to create the null distribution for a hypothesis test.
\item edge false positive graphs which are generated using the same model parameters but with additional sampled edges. 
\item node false positive graphs which are generated using the same model parameters but with additional sampled nodes.
\item true positive graphs which are generated with a normal number of edges and nodes but different model parameters.  
\end{enumerate}
The first three sets of graphs are created with the same generative model but with varying edges and nodes in the output graphs while the last set uses a different generative model.  An illustration of the null distribution, false positive distribution with additional edges, and true positive distribution is shown in figure \ref{synthetic-generation}.

First a set of normal graph examples are created using the process described in 1. which will form the null distribution graphs.  A statistic is calculated for each graph example and given an $\alpha$ value the two critical points are found.  Then a false positive graph set is created using either 2. or 3. and a true positive graph set created using 4. and statistics calculated for each.  The percentage of false positive graphs outside the critical points becomes the \textit{false positive rate} while the percentage of true positive graphs becomes the \textit{true positive rate}.  By varying the value of $\alpha$ and plotting the true positive vs. false positive rate for each value we can create an ROC curve showing the tradeoff of true anomalous instances found versus falsely flagged instances.


The circle on the ROC curves represents selecting a p-value of 0.05.  The number of edges in the normal and true positive graphs ranges from 300k-400k while the edge false positive graphs have 400k-500k, and the number of nodes in the normal and true positive graphs is 25k while the node false positive graphs have 30k.  An equal number of graphs of each type were generated.  For a statistic that detects the model changes reasonably well we expect the false positive distribution to be very close to the null distribution, while the anomalous distribution is significantly different.  

Ideal performance on the ROC curve would be a horizontal line across the entire top of the plot: this would indicate perfect performance in detecting true positive graphs even at low p-value, and a false positive rate that is low until the p-value is increased.  For comparison a diagonal line with a slope of 1 would indicate random performance where each false positive and true positive graph is flagged as anomalous using an unbiased coin flip.  Any statistic which has a curve below this line has more sensitivity to the additional edges or nodes of the false positive graphs than to the model changes of the true positive graphs.  Some of the statistics evaluated even have a vertical line at the right of the plot: this indicates that no matter the p-value picked all false positive graphs are being flagged but not all true positive graphs are flagged; this is the worst possible space for the statistic to be in.

To evaluate delta statistics, graphs were generated in pairs, the first being from the normal/false positive/true positive model while the other always from the normal model, and the delta statistic calculated between them.

\begin{figure}[!h]
\begin{center}
\includegraphics[width=.79\columnwidth,natwidth=610,natheight=642,trim={0mm 10mm 0mm 0mm},clip]{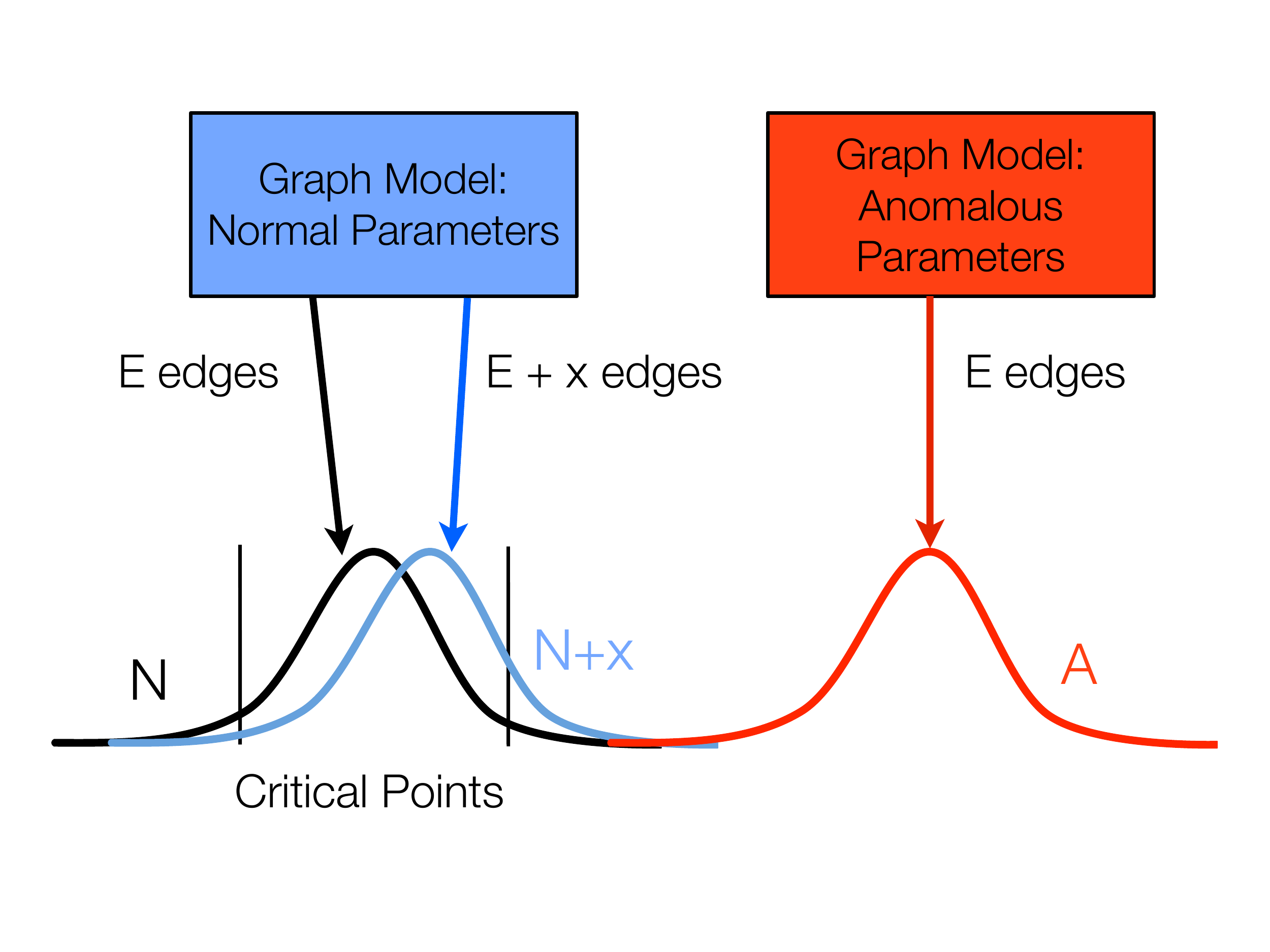}
\caption{Diagram of synthetic experiments and the three sets of generated graphs.  }
\label{synthetic-generation}
\end{center}
\end{figure}


\begin{figure}[h!]
\begin{algorithmic}
\STATE $Synthetic(P_N, P_A, | W |, | V |, \delta, S_k, \alpha):$ 
\\ \hrulefill
\FOR{$i$ in $1...50$ }
\STATE $NormalGraphs.add(GenerateGraph(|W|,|V|,P_N)$
\STATE $FalsePosGraphs.add(GenerateGraph(|W|+\delta,|V|,P_N)$
\STATE $AnomalousGraphs.add(GenerateGraph(|W|,|V|,P_A)$
\ENDFOR
\STATE $NullDistr = \{ S_k(G), G \in NormalGraphs \}$
\STATE $CriticalPoints = CalcCPs(NullDistr,\alpha)$
\FOR{$G$ in $False Positive Graphs$}
\IF{$S_k(G)$ outside $CriticalPoints$ }
\STATE False Positives ++
\ENDIF
\ENDFOR
\FOR{$G$ in $Anomalous Graphs$}
\IF{$S_k(G)$ outside $CriticalPoints$ }
\STATE True Positives ++
\ENDIF
\ENDFOR
\\ \hrulefill
\end{algorithmic}
\vspace{3mm}
\caption{Synthetic data experimental procedure for statistic $S_k$ using normal probability matrix $P_N$, anomalous probability matrix $P_A$, and $\Delta$ additional edges in False Positive graphs.}
\label{fig:synthgeneration}


\end{figure}%



\begin{figure*}[h!]
\centering
\begin{tabular}{c c c}
\vspace{-3mm}
 \subfigure[]{\includegraphics[trim={2mm 2mm 2mm 2mm},clip,width=.30\columnwidth,natwidth=610,natheight=642]{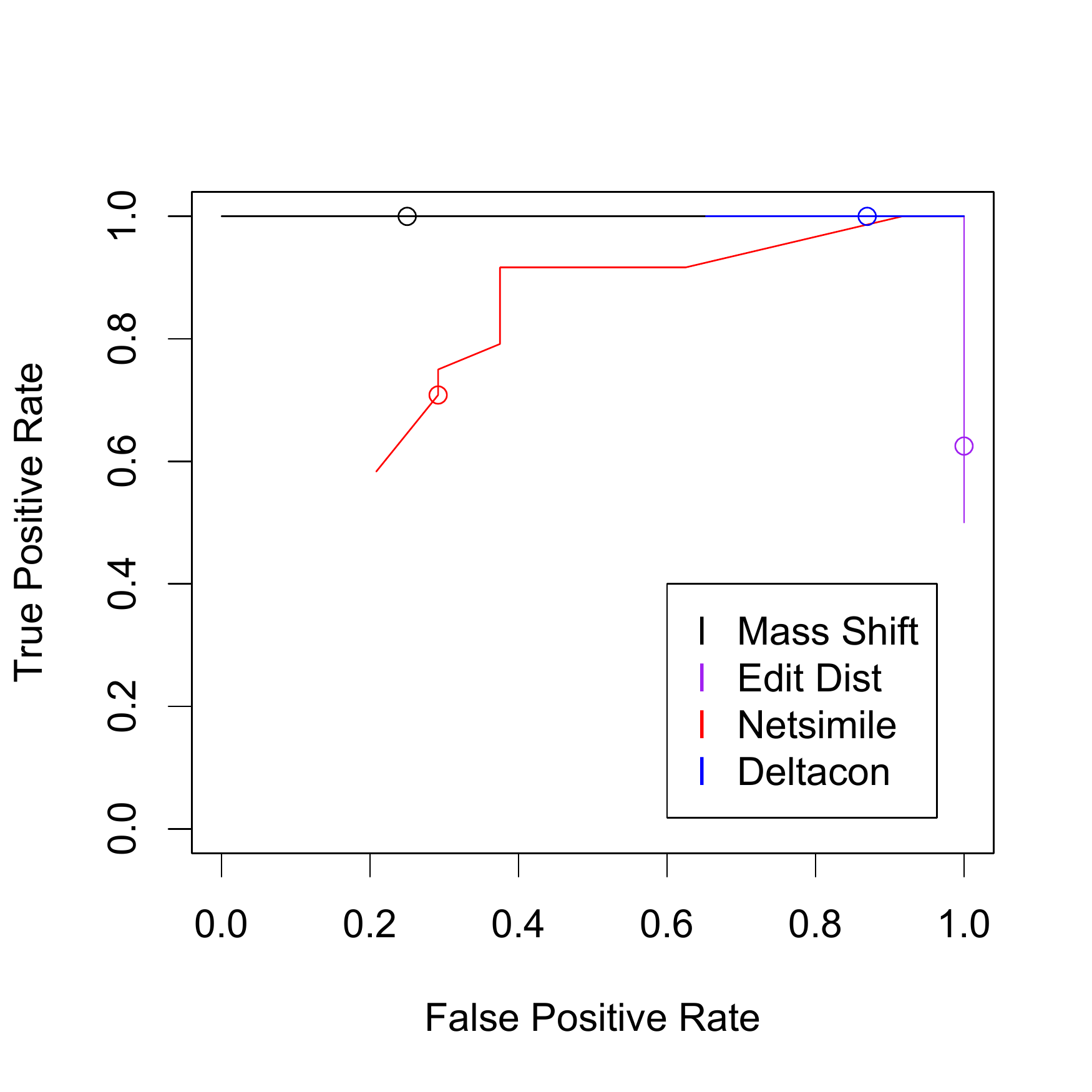}} & \subfigure[]{\includegraphics[width=.30\columnwidth,natwidth=610,natheight=642]{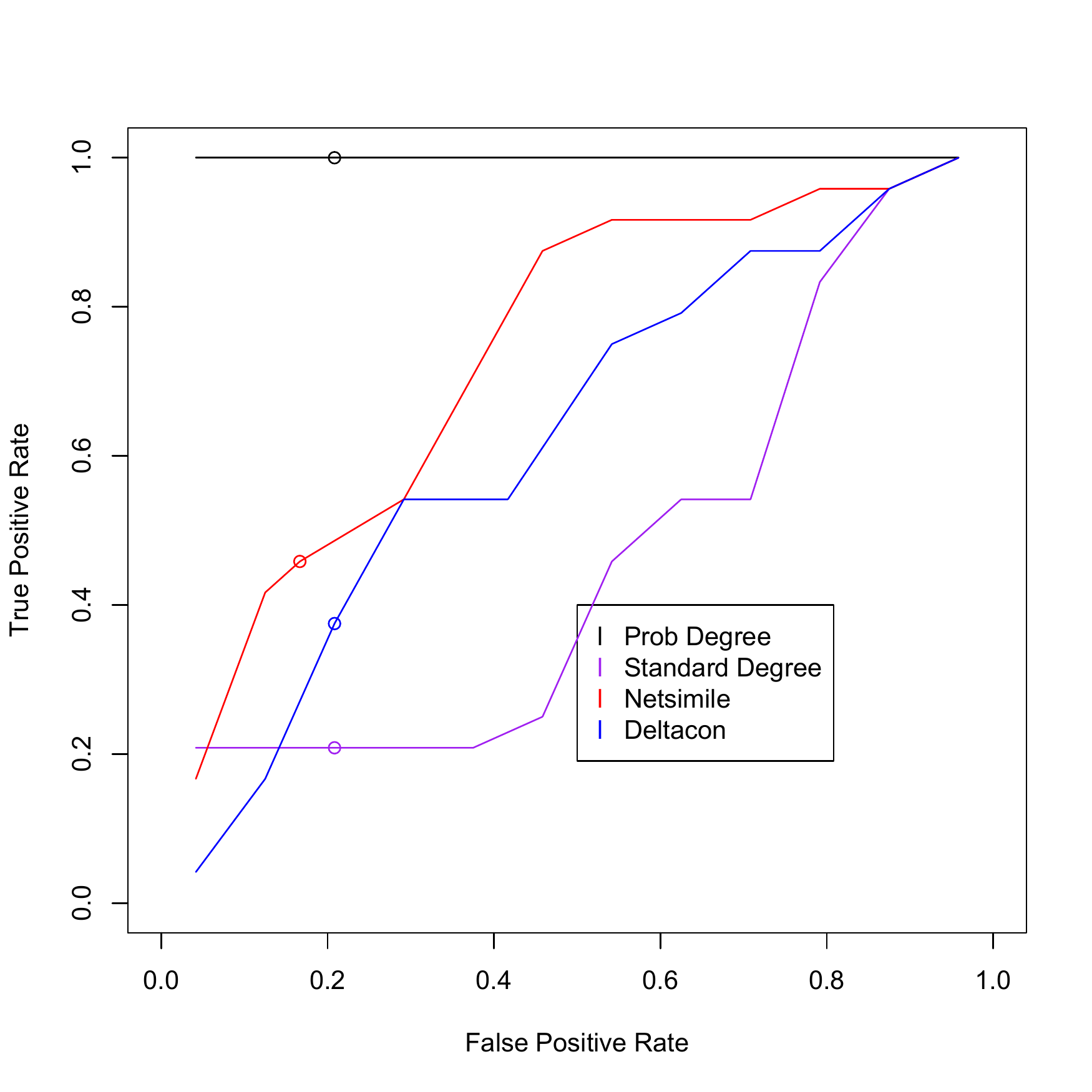}} & 
 \subfigure[]{\includegraphics[width=.30\columnwidth,natwidth=610,natheight=642]{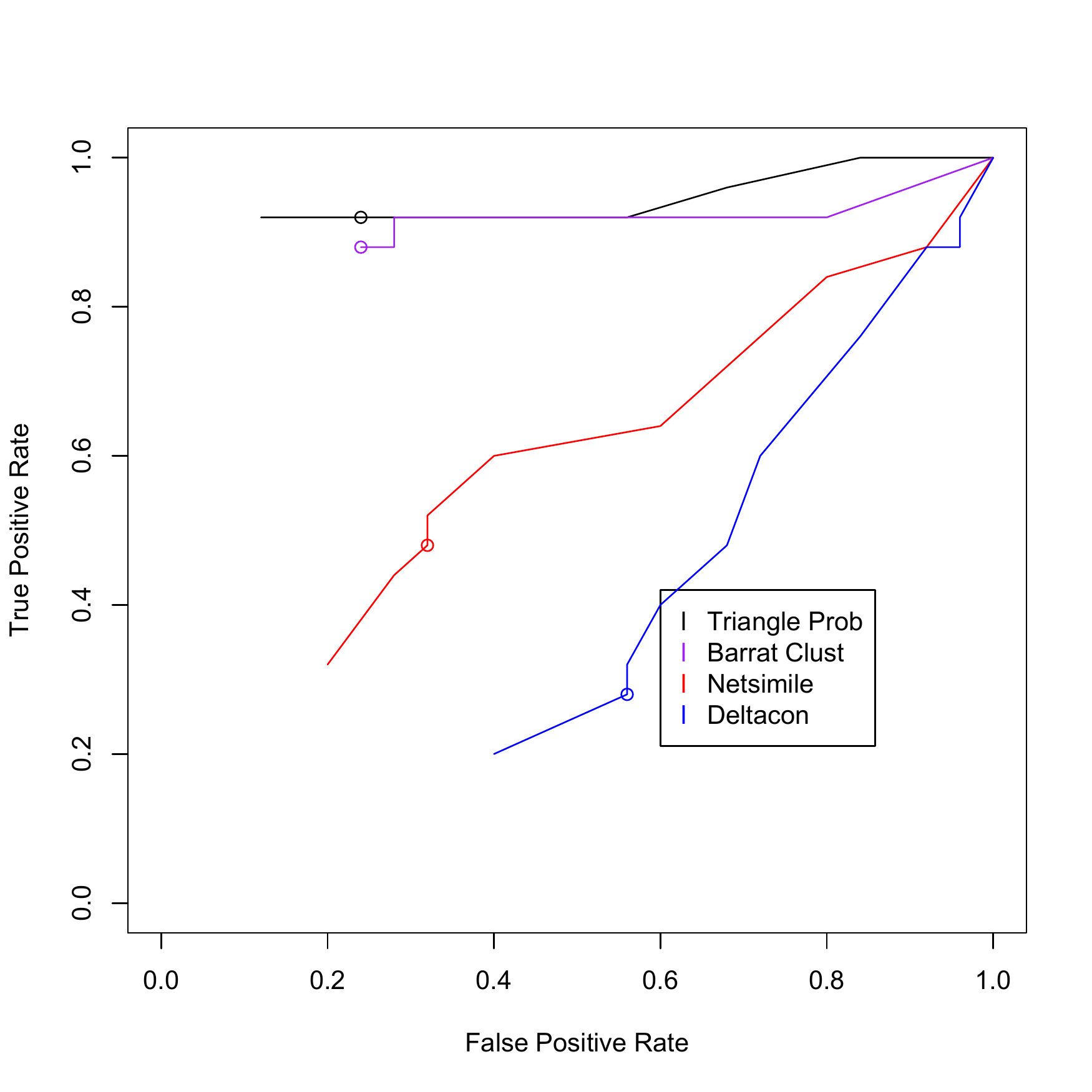}} \\ 
\subfigure[]{\includegraphics[width=.30\columnwidth,natwidth=610,natheight=642]{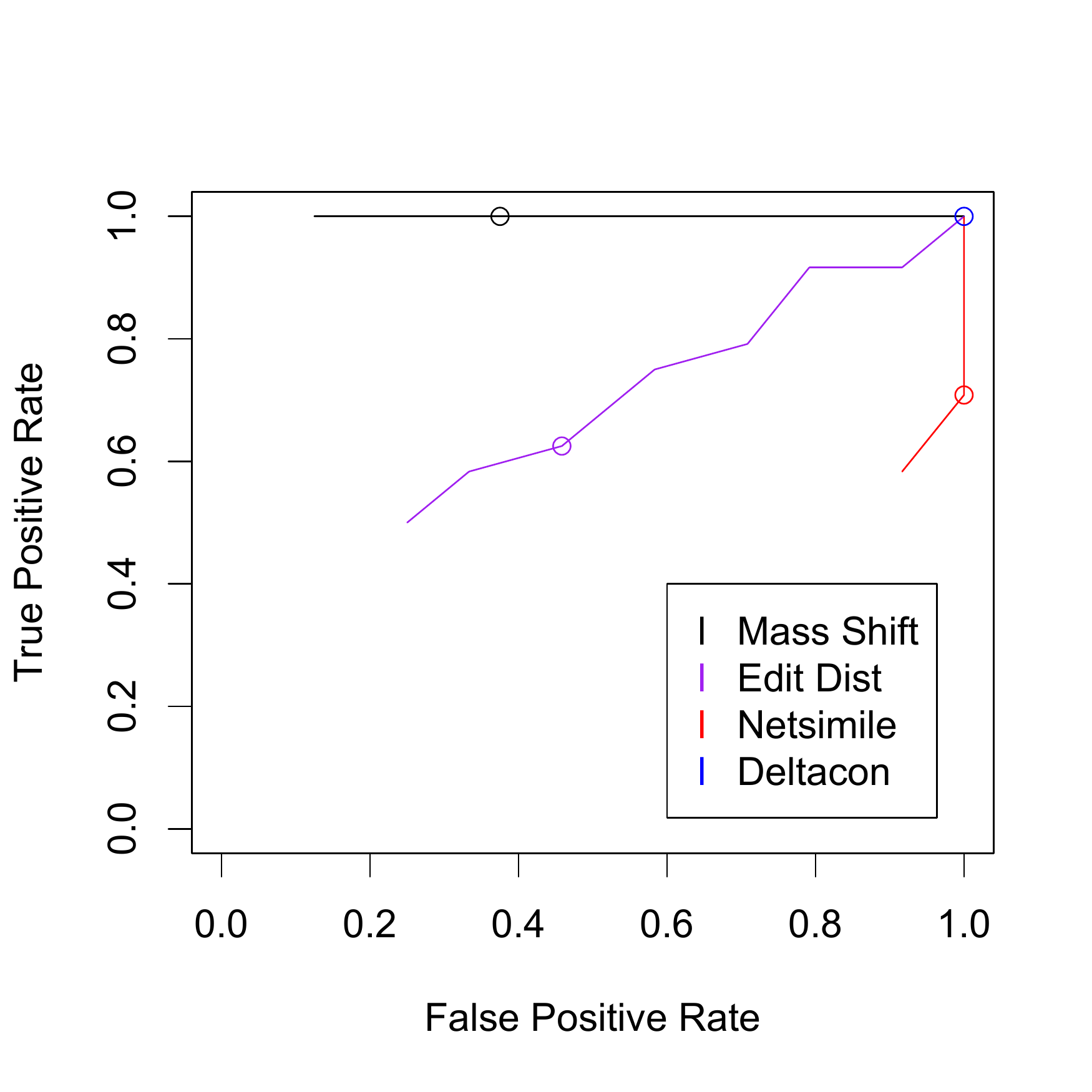}} & \subfigure[]{\includegraphics[width=.30\columnwidth,natwidth=610,natheight=642]{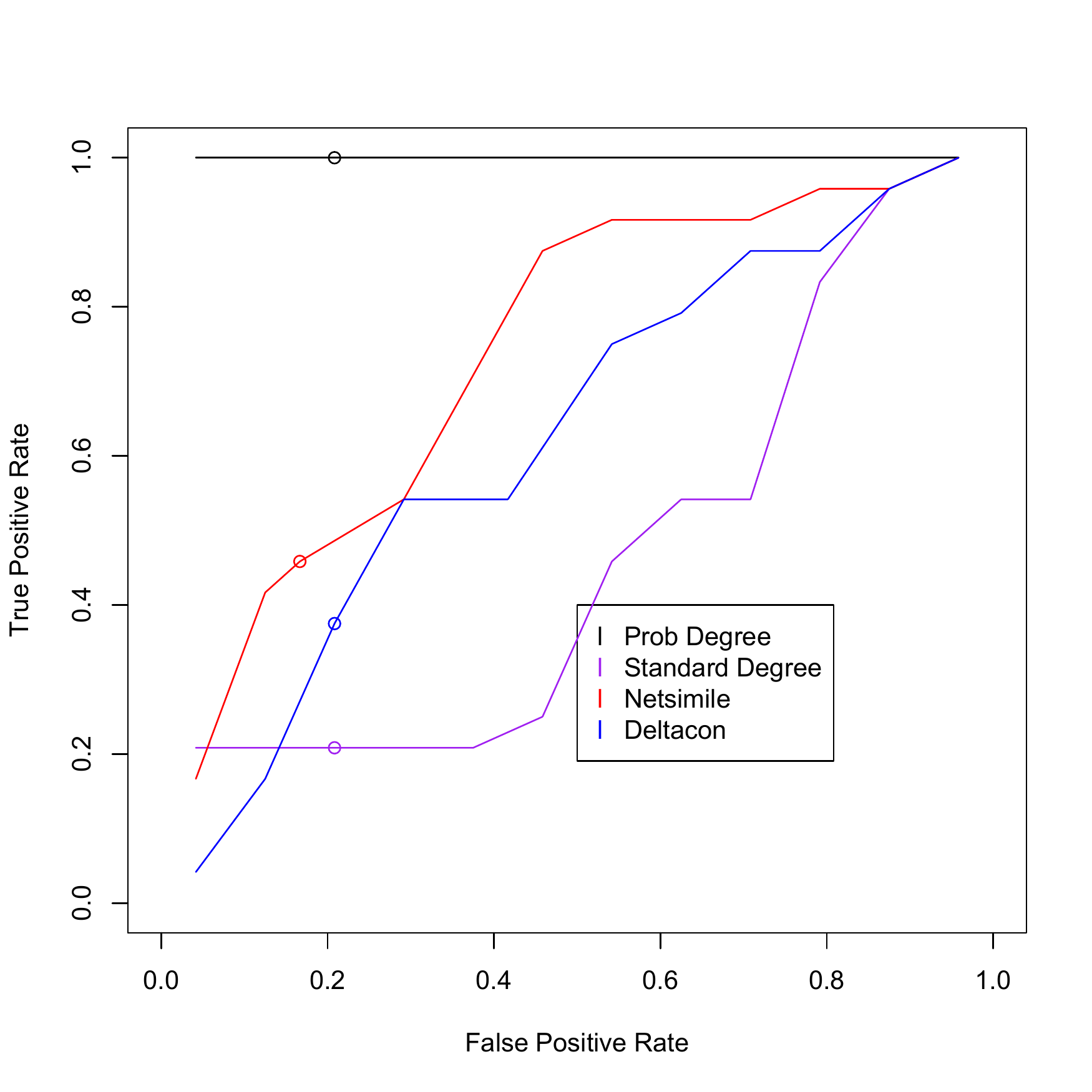}} & 
 \subfigure[]{\includegraphics[width=.30\columnwidth,natwidth=610,natheight=642]{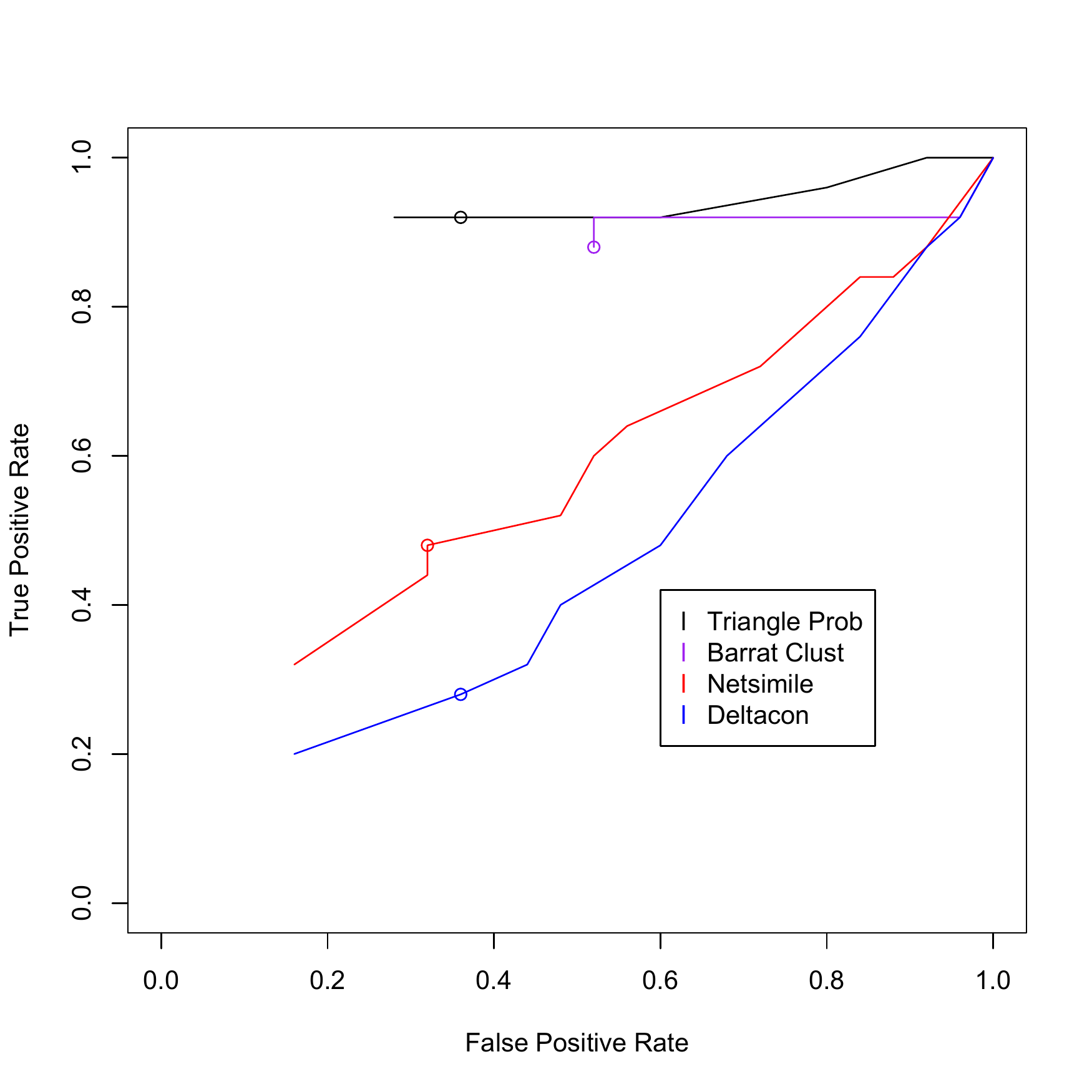}} \\ 
\subfigure[]{\includegraphics[width=.30\columnwidth,natwidth=610,natheight=642]{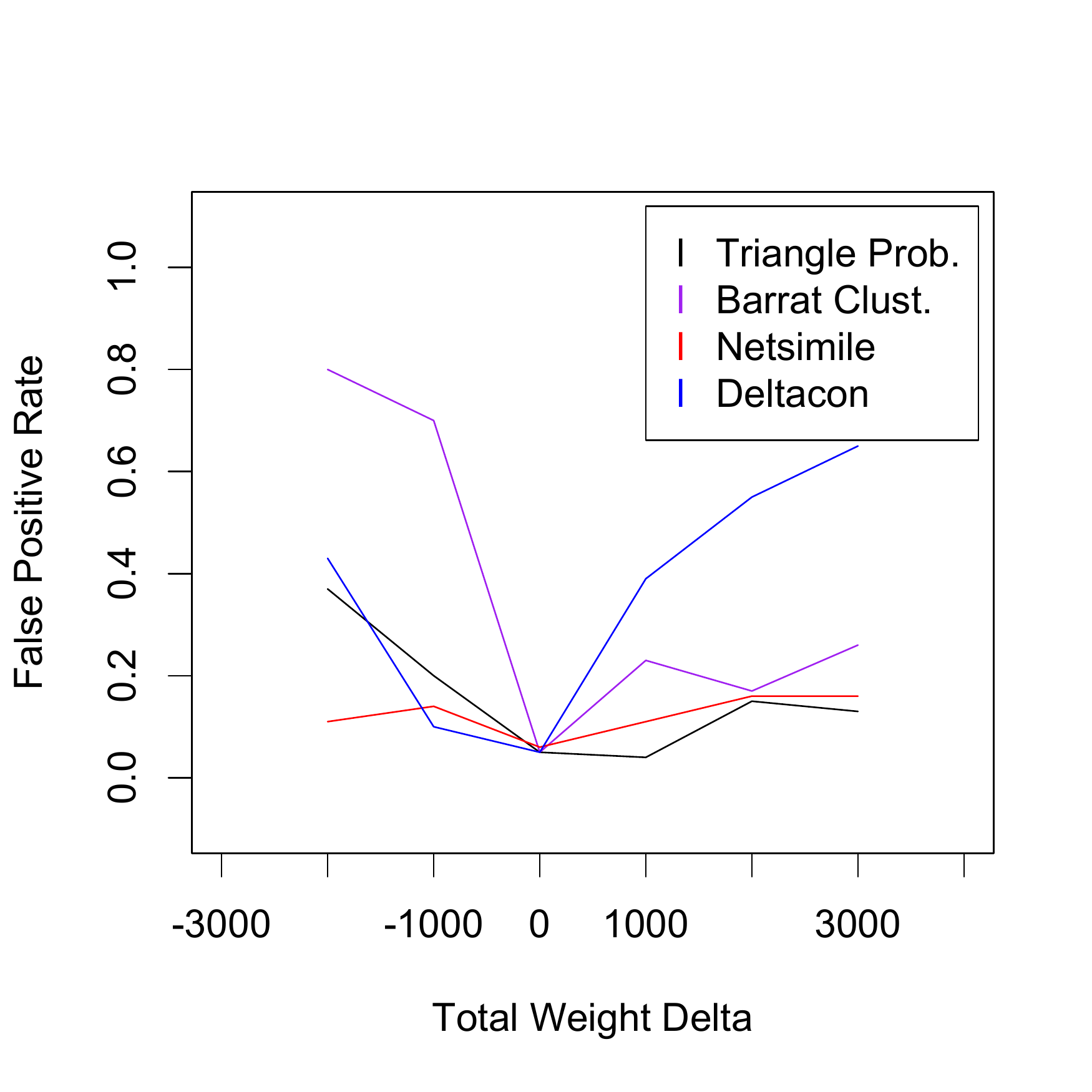}} & \subfigure[]{\includegraphics[width=.30\columnwidth,natwidth=610,natheight=642]{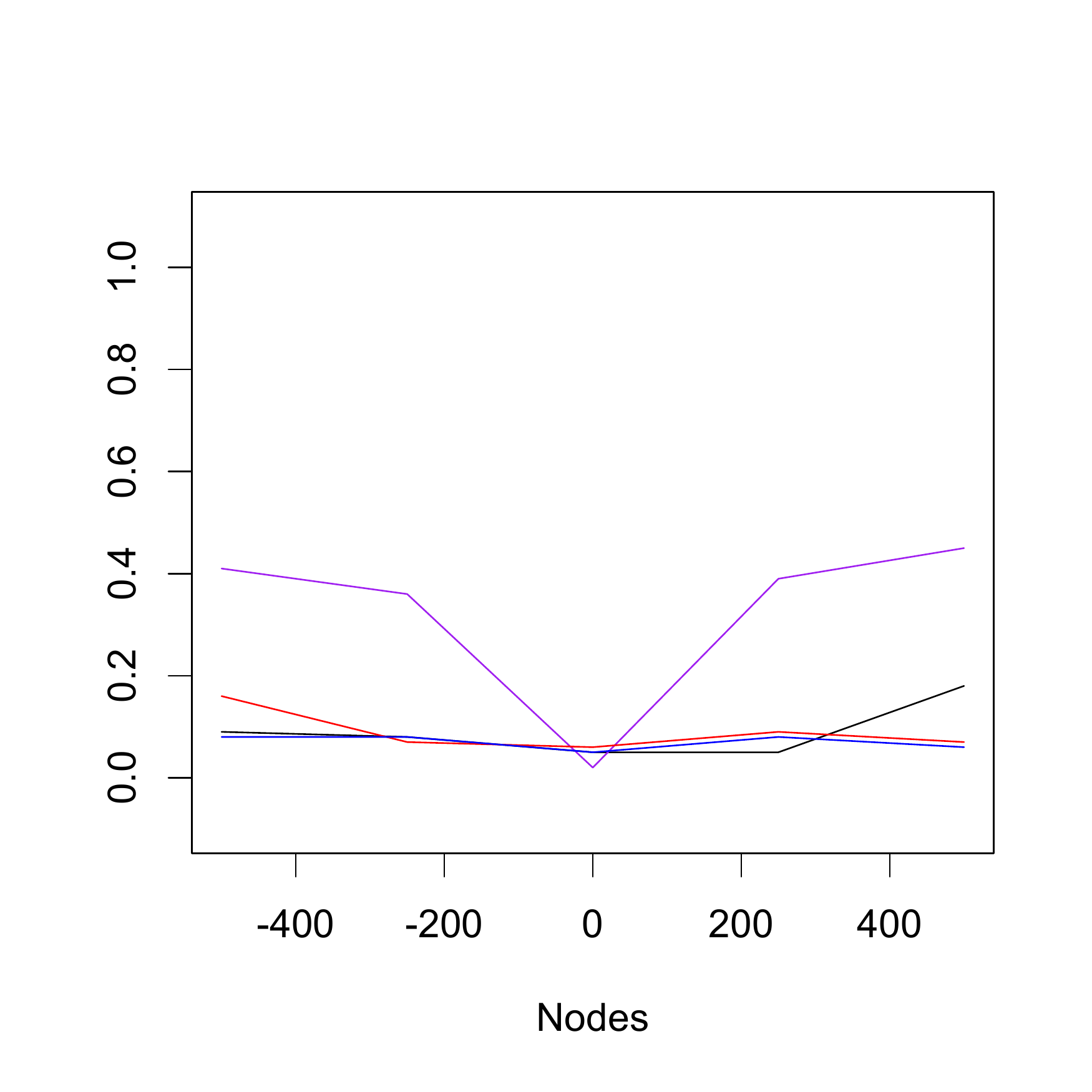}} &  \subfigure[]{\includegraphics[width=.30\columnwidth,natwidth=610,natheight=642]{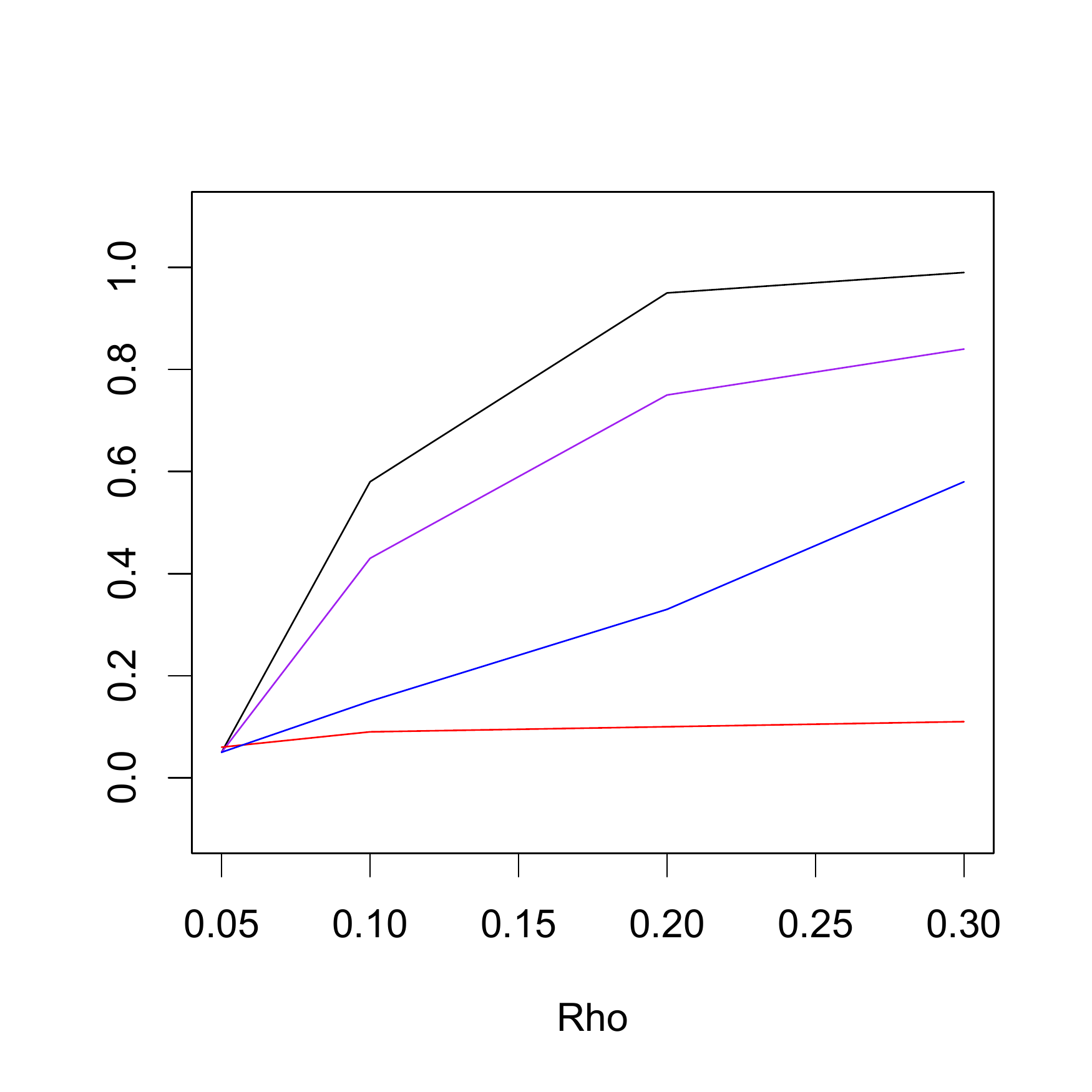}} \\  \vspace{-2mm}
\end{tabular}
\caption{(a),(b),(c): ROC curves with false positives due to extra edges.  (d),(e),(f): ROC curves with false positives due to extra nodes.  (g) false positive rate vs. edges, (h) false positive rate vs. nodes, (i) true positive rate.}
\label{fig:ROC}
\end{figure*}

To test the performance of graph change statistics like graph edit distance and probability mass shift, synthetic data was generated using a mixture model that either samples edges from a static normal graph instance from 1. or from an anomalous graph from 4.  The initial graph has a power-law degree distribution with an exponent of 2.0 and was generated using a Chung-Lu sampling process while the alternative graph was generated with an Erdos-Renyii graph model.  The normal model draws edge samples only from the initial graph, while the alternative model draws 5\% of the edges from the alternative graph.  The performance of these statistics is shown in \ref{fig:ROC} (a) and (d).  Mass Shift strictly dominates the other statistics as either the edges or nodes changes.  

To determine ability to detect degree distribution changes synthetic graphs were also generated using a Chung-Lu process, however anomalous graph instances were generated by altering the exponent parameter of the power law determining degree distribution rather than using a mixture model.  The normal graph instances have a power-law degree distribution with an exponent of 2.0 while the true positive graph have an exponent of 1.8.  The performance is shown in figure \ref{fig:ROC} (b) and (e).





The transitivity experiments were done by creating graphs with a varying amount of triangle closures.  To create each graph, a KPGM model with a seed of 
\begin{small}
$\left[ \begin{tabular}{c c} 0.4 & 0.3 \\ 0.3 & 0 \end{tabular}  \right]$
\end{small}
is used to sample an initial edge set.  These parameters were selected to create a graph with a branching pattern with few natural triangles.  Then, with probability $\rho$ each edge is removed and replaced with a triangle closure by performing a random walk (identical to the technique used in the Transitive Chung Lu model \cite{pfeiffer}).  The normal graphs were generated with a rho of 0.05 while the alternative graphs had a rho of 0.055.  The results are in figures (c) and (f).  

Figures \ref{fig:ROC} (g)-(i) shows the effect of changing (g) edges, (h) nodes, or (i) model parameter on transitivity statistics.  The zero point on the false positive plots compares graphs of the same size and model which will produce false positives at the p-value rate, while deviating in either direction introduces more false positives.  The power in figure (i) depends on the deviation in the model parameter.

\subsection{Semi-Synthetic Data Experiments}



\begin{figure}[!h]
\begin{center}
\includegraphics[width=.79\columnwidth,natwidth=610,natheight=642]{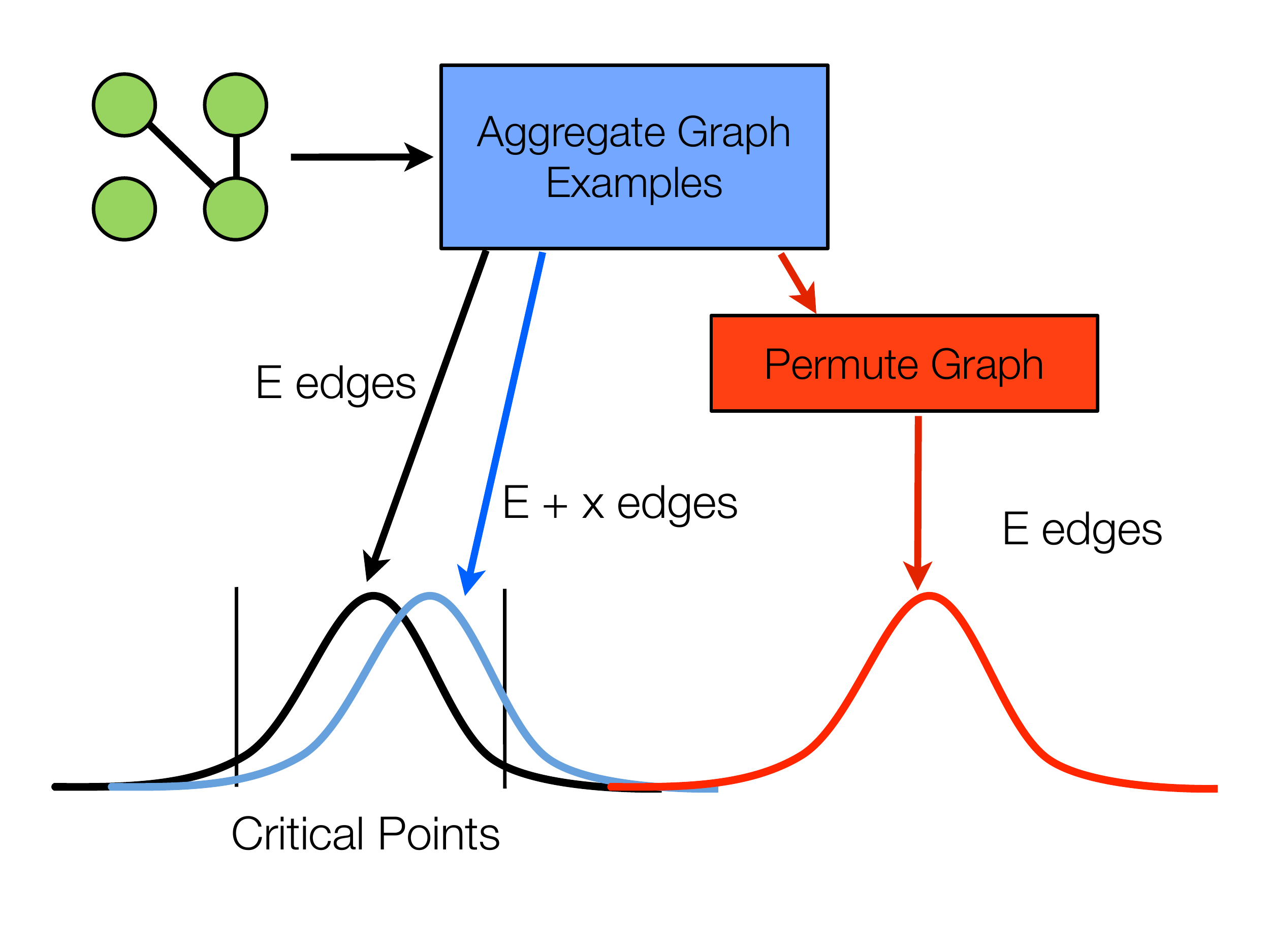}
\caption{Diagram of semi-synthetic experiments and the three sets of generated graphs.  }
\label{semi-synthetic-generation}
\end{center}
\end{figure}
Although synthetically driven experiments have the advantage of complete control over the network properties of the generated graphs, these experiments give inherently artificial results and the utility of any conclusions drawn from those experiments depends heavily on the comprehensiveness of the experiments.  To ensure that these results generalize to more realistic scenarios I've also evaluated them using a set of semi-synthetic experiments where the normal and false positive graph examples of 1. -- 3. are sampled from real-world networks and the true positive anomaly examples of 4. are artificially inserted.  These experiments show that the proposed consistent statistics are superior at discovering anomalies inserted into real-world data.

\begin{figure}[h!]
\begin{algorithmic}
\STATE $SemiSynthetic(P_N, P_A, | W |, | V |, \delta, S_k, \alpha):$ 
\\ \hrulefill
\FOR{$i$ in $1...50$ }
\STATE $NormalGraphs.add(SampleGraph(|W|,|V|,P_N)$
\STATE $FalsePosGraphs.add(SampleGraph(|W|+\delta,|V|,P_N)$
\STATE $AnomalousGraphs.add(SampleGraph(|W|,|V|,P_A)$
\ENDFOR
\STATE $NullDistr = \{ S_k(G), G \in NormalGraphs \}$
\STATE $CriticalPoints = CalcCPs(NullDistr,\alpha)$
\FOR{$G$ in $False Positive Graphs$}
\IF{$S_k(G)$ outside $CriticalPoints$ }
\STATE False Positives ++
\ENDIF
\ENDFOR
\FOR{$G$ in $Anomalous Graphs$}
\IF{$S_k(G)$ outside $CriticalPoints$ }
\STATE True Positives ++
\ENDIF
\ENDFOR
\\ \hrulefill
\end{algorithmic}
\vspace{3mm}
\caption{Semi-synthetic data experimental procedure for statistic $S_k$ using normal probability matrix $P_N$, anomalous probability matrix $P_A$, and $\Delta$ additional edges in False Positive graphs.}
\label{fig:semisynthgeneration}


\end{figure}%


The first step in generating the graph sets is to aggregate all graph instances from a dynamic network source into a single graph example which will become our normal graph source.  All normal graph examples are generated from this source graph by first sampling an active node set, obtaining the subgraph over those active nodes, then sampling edges with replacement to create the sample graph.  By aggregating all instances over time we smooth out any variations that occur over the lifespan of the network and obtain the ``average'' behavior of the network to use as our normal examples.  False positive examples are creates by sampling additional nodes or edges from the same source network.  

True positive examples are sampled from a separate, alternate source instance which is created by permuting the original source graph in some way.  To generate network change anomalies the alternate source has 5\% of its edges selected uniformly at random compared to the source; degree distribution anomalies are generated by taking 30\% of the edges of high degree nodes (high degree meaning in the top 50\% of nodes) and assigning them uniformly at random; and transitivity anomalies are generated by performing triangle closures by selecting an initial node, randomly walking two steps, then linking the endpoints of the walk to form a triangle.  The semi-synthetic data generation process is shown in figure \ref{semi-synthetic-generation} and the exact algorithm for generating the data is described in Figure \ref{fig:semisynthgeneration}.  The input $P_N$ is created by dividing the aggregated normal graph described above by $|W|$ and the input $P_A$ is created by modifying the aggregated normal graph in one of the ways described above and then dividing by $|W|$.  


The dataset used for the underlying graph was the University E-mail dataset described in Section \ref{real-data-experiments} used in the real data experiments; when aggregated this data forms a graph with 54102 total nodes and 5236591 total messages.  Edge false positives are generated by creating graphs with 20k nodes and either 400k or 600k edges while node false positives are generated by sampling either 20k or 30k nodes and sampling edges equal to 20 times the number of nodes.  Sampling edges as a ratio of nodes in the node false positive experiment is to hold the density of the graphs constant.


We analyze the performance of the statistics using the same ROC approach as with the synthetic data.  Figure \ref{fig:ROCsemiSynth} shows the resulting ROC curves.   As with the synthetic experiments (a)-(c) show mass shift statistics, degree distribution statistics, and transitivity statistics respectively when the false positives are generated with additional edges, while (d)-(f) have false positives generated via additional nodes.  The proposed consistent statistics have superior performance in most cases, and none of the competing statistics perform well in both the additional edges and additional nodes scenarios.

\vspace{-3mm}
\begin{figure}[h!]
\centering

\subfigure[]{\includegraphics[width=.32\columnwidth,natwidth=610,natheight=642]{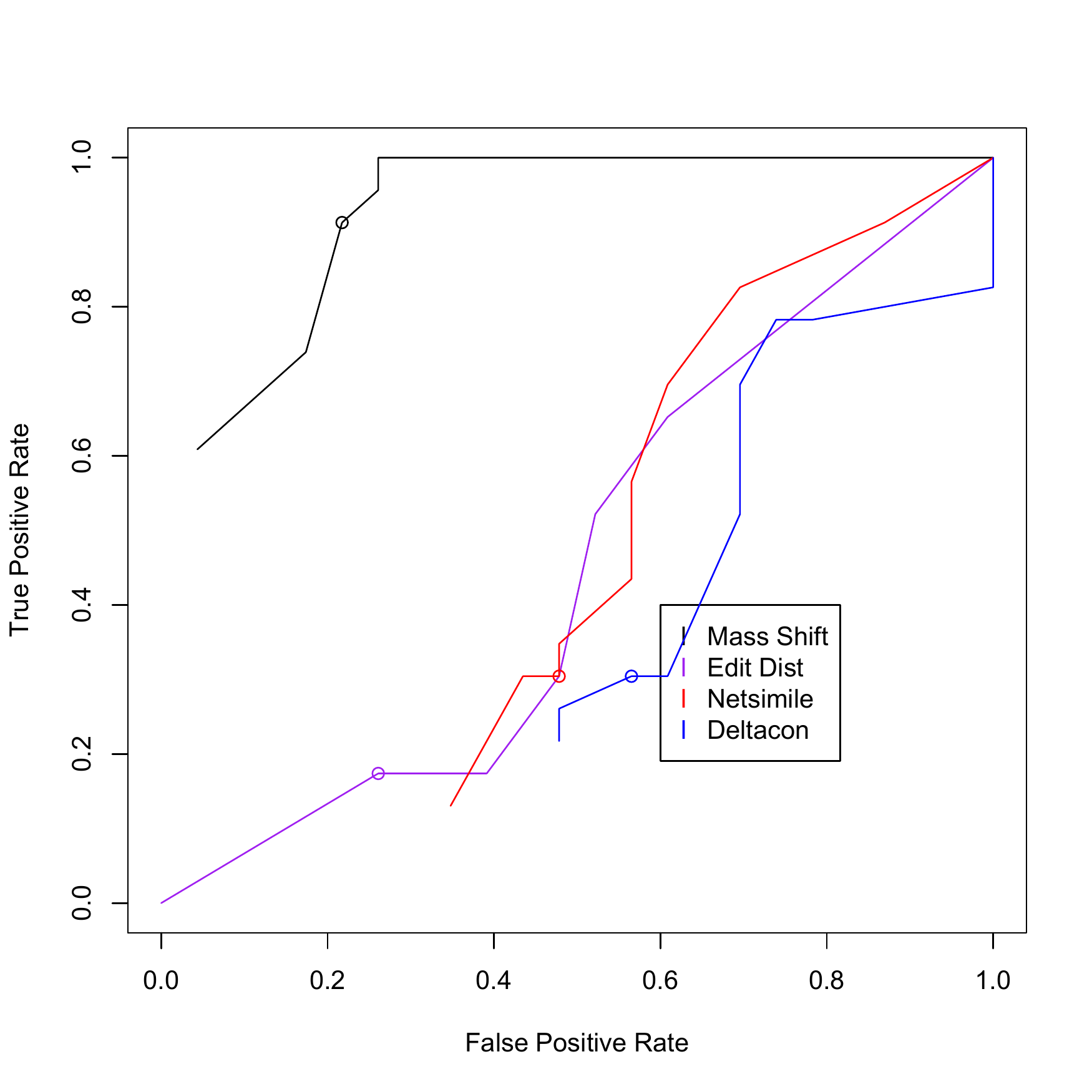}}
\subfigure[]{\includegraphics[width=.32\columnwidth,natwidth=610,natheight=642]{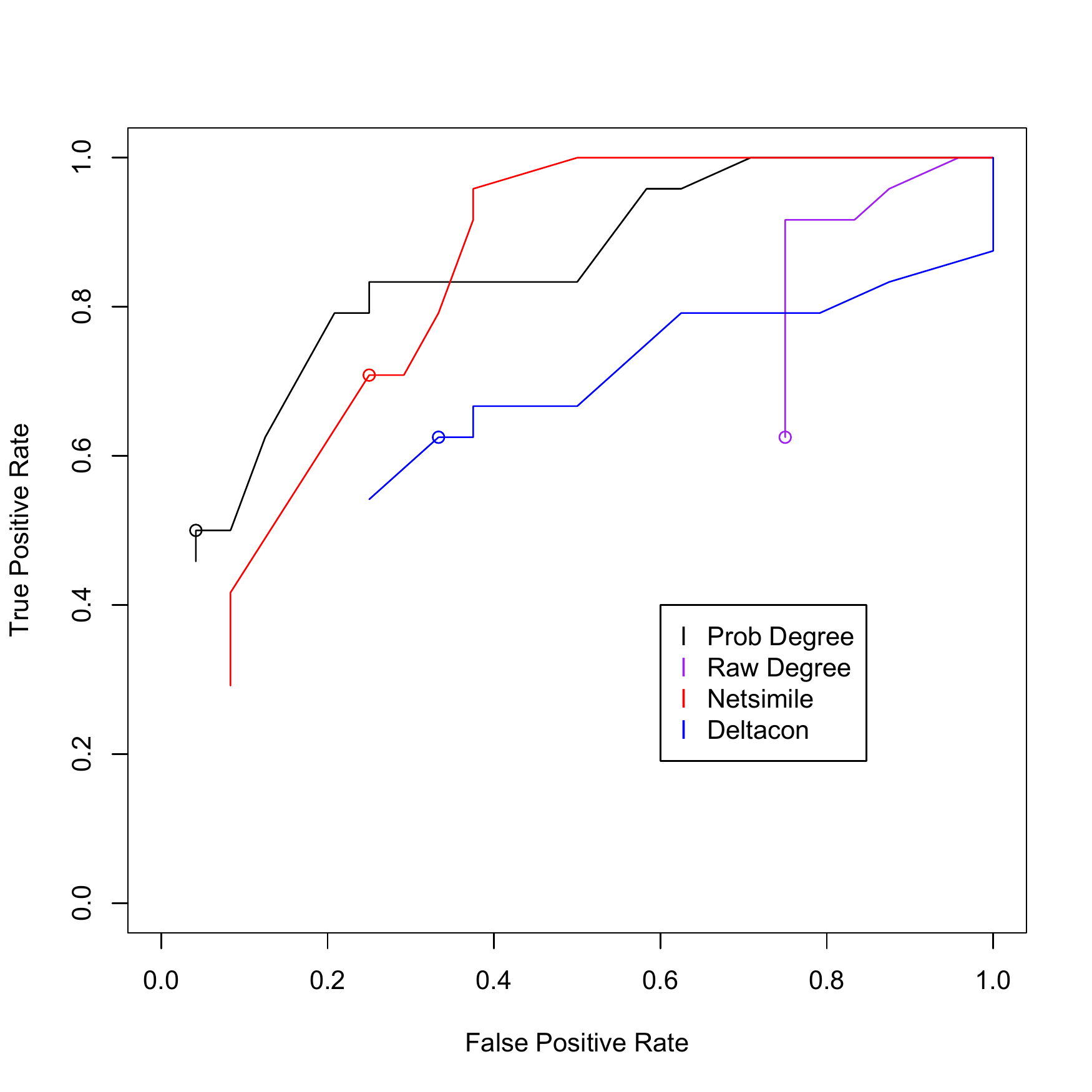}}
\subfigure[]{\includegraphics[width=.32\columnwidth,natwidth=610,natheight=642]{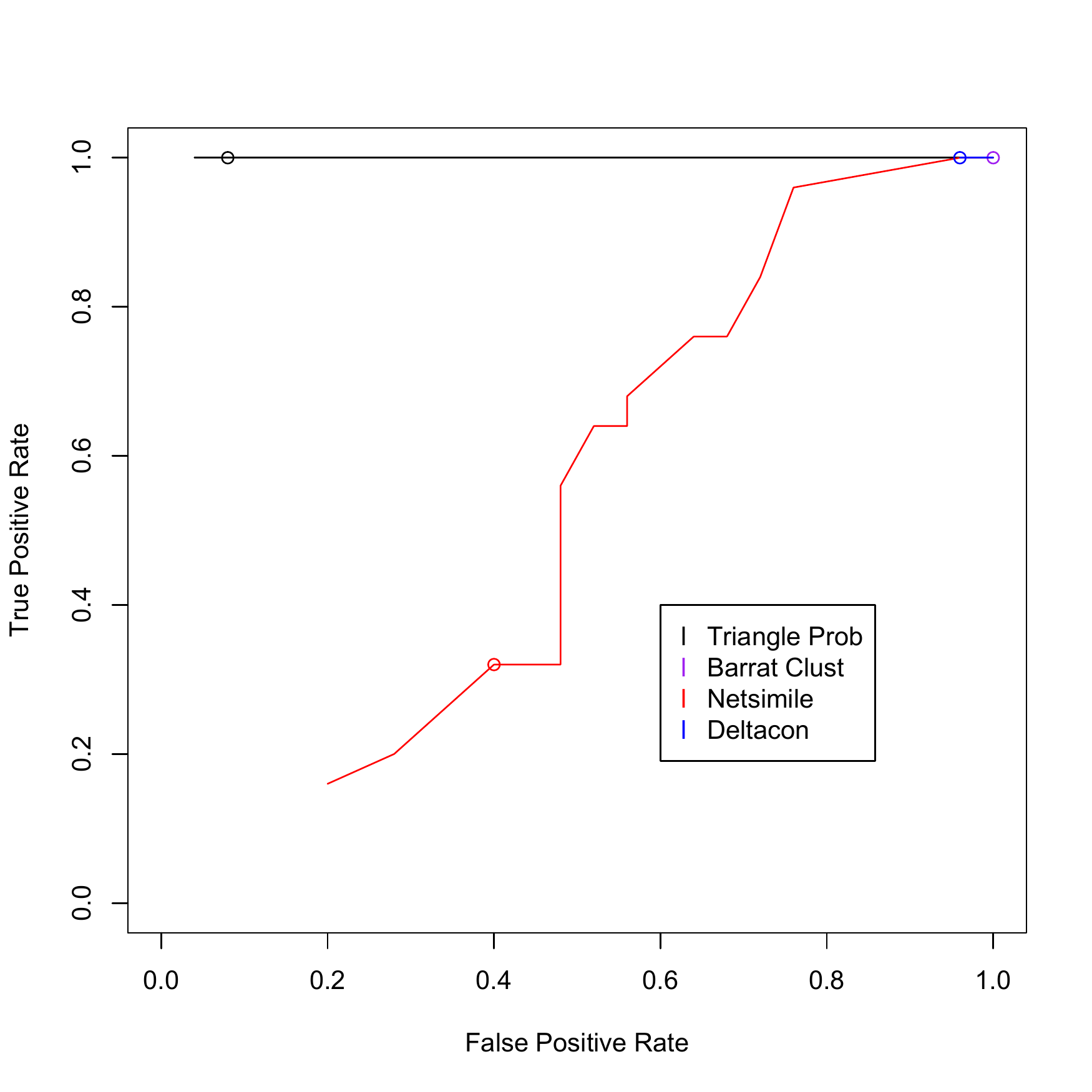}}
\subfigure[]{\includegraphics[width=.32\columnwidth,natwidth=610,natheight=642]{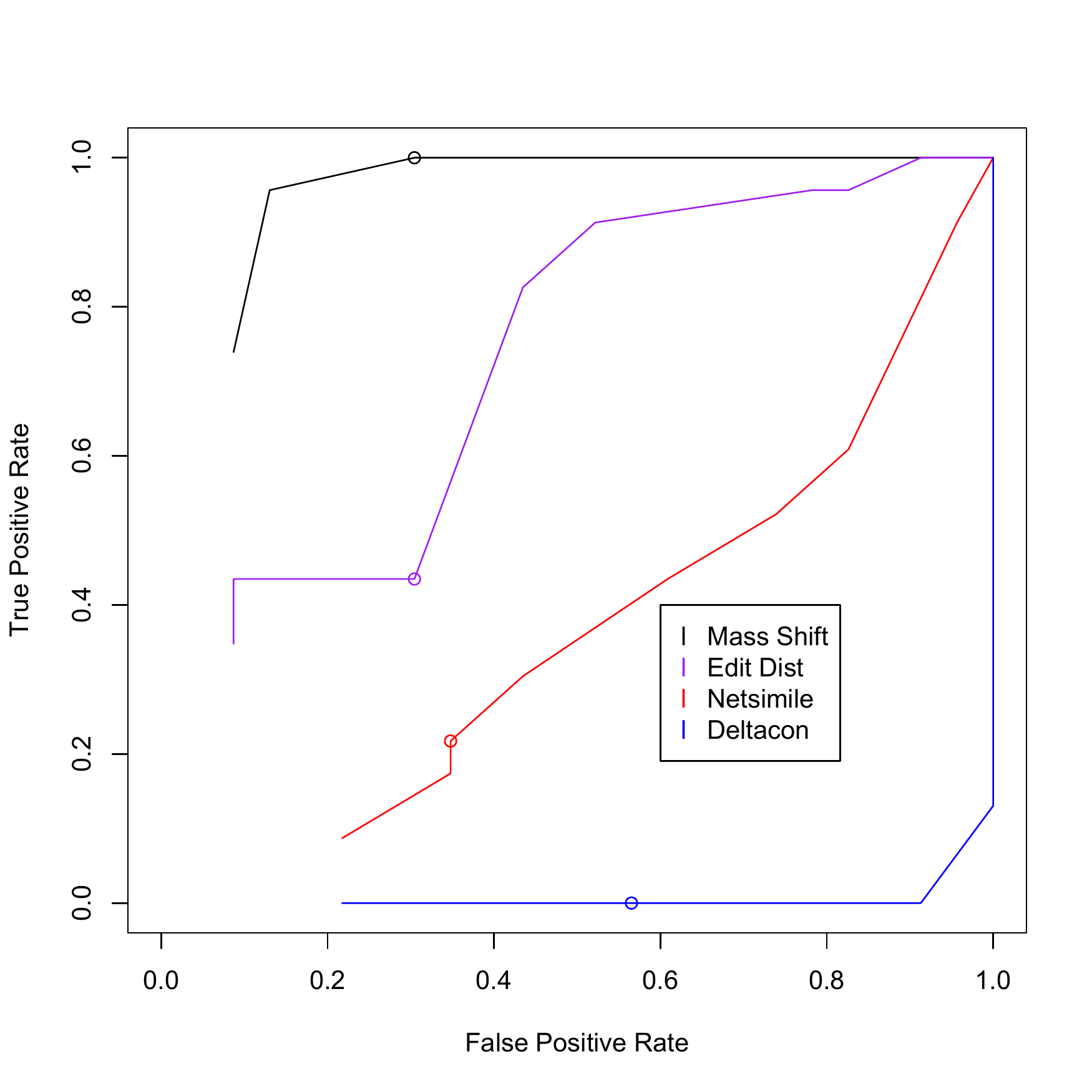}}
\subfigure[]{\includegraphics[width=.32\columnwidth,natwidth=610,natheight=642]{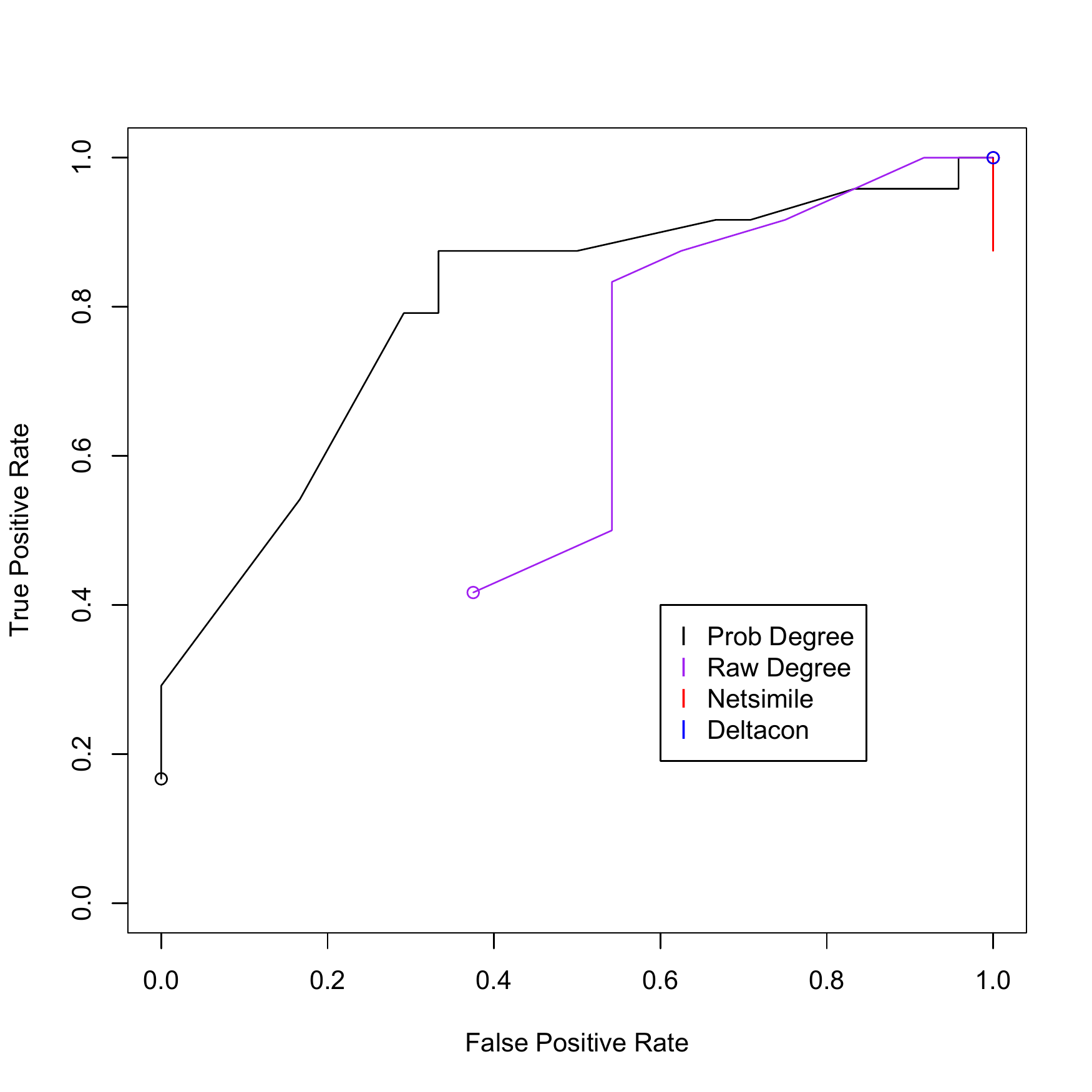}}
\subfigure[]{\includegraphics[width=.32\columnwidth,natwidth=610,natheight=642]{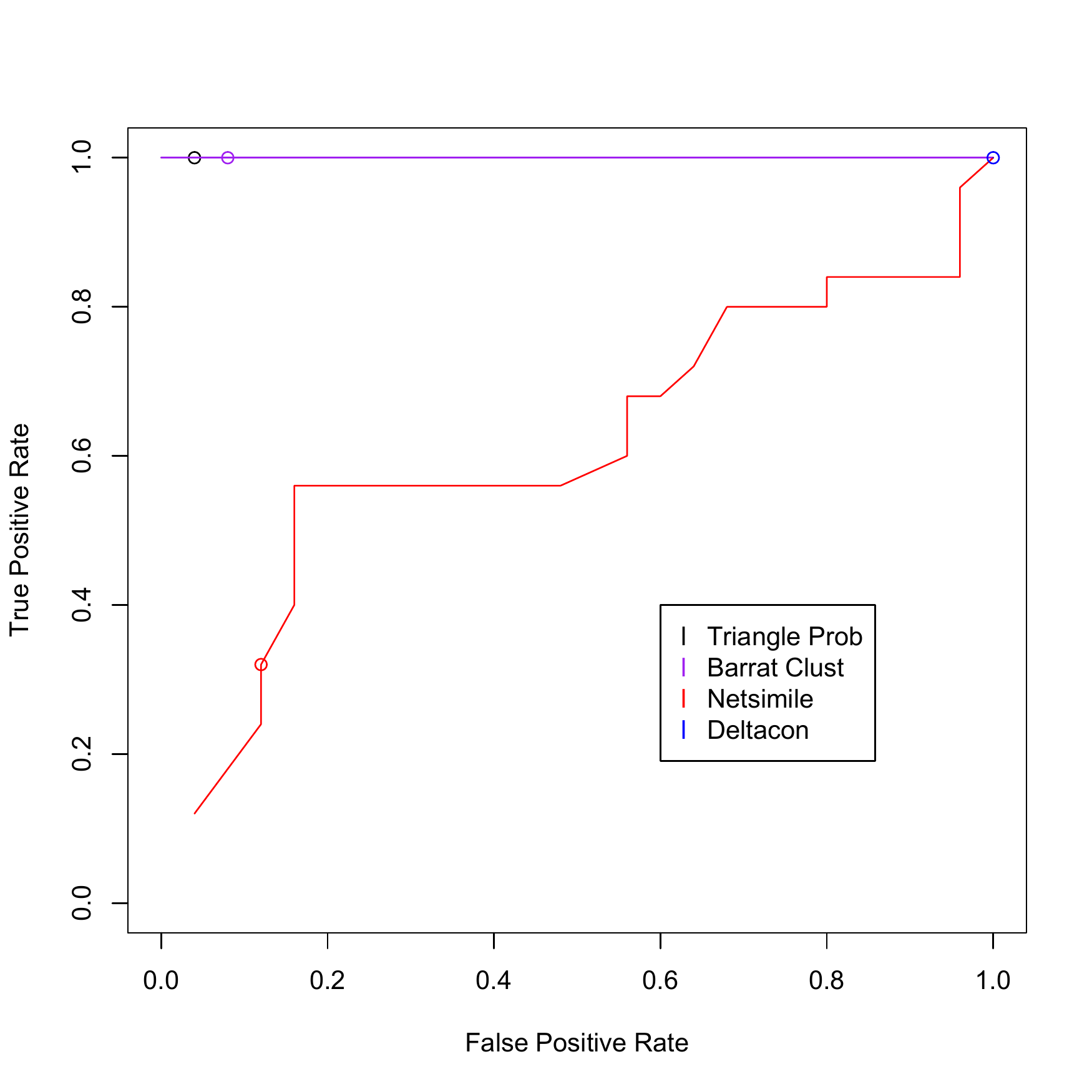}}
\caption{ROC curves on the semi-synthetic dataset with varying alphas.}
\label{fig:ROCsemiSynth}
\end{figure}
\vspace{-2mm}





\subsection{Real Data Experiments} 
\label{real-data-experiments}



Now let us investigate the types of anomalies found when these statistics are applied to three real-world networks and contrast these events to those found by other detectors.  The first dataset is the Enron communication data, a subset of e-mail communications from prominent figures of the Enron corporation (150 individuals, 47088 total messages) with a time step width of one week used in papers such as Priebe et al \cite{priebe}.  The second is the University E-mail data, e-mail communications of students from one university in the 2011-2012 school year (54102 individuals, 5236591 total messages), sampled daily and described in detail in the paper by LaFond et al \cite{tlafond}.  The third is a Facebook network subset made up of postings to the walls of students in the 2007-2008 school year (444829 individuals, 4171383 total messages), also from the same university and sampled daily.  The Facebook dataset was also used in a paper by LaFond \cite{tlafond2} and is described there in more detail. 

\begin{figure*}[h!]
\centering
\includegraphics[width=.65\columnwidth,natwidth=610,natheight=642]{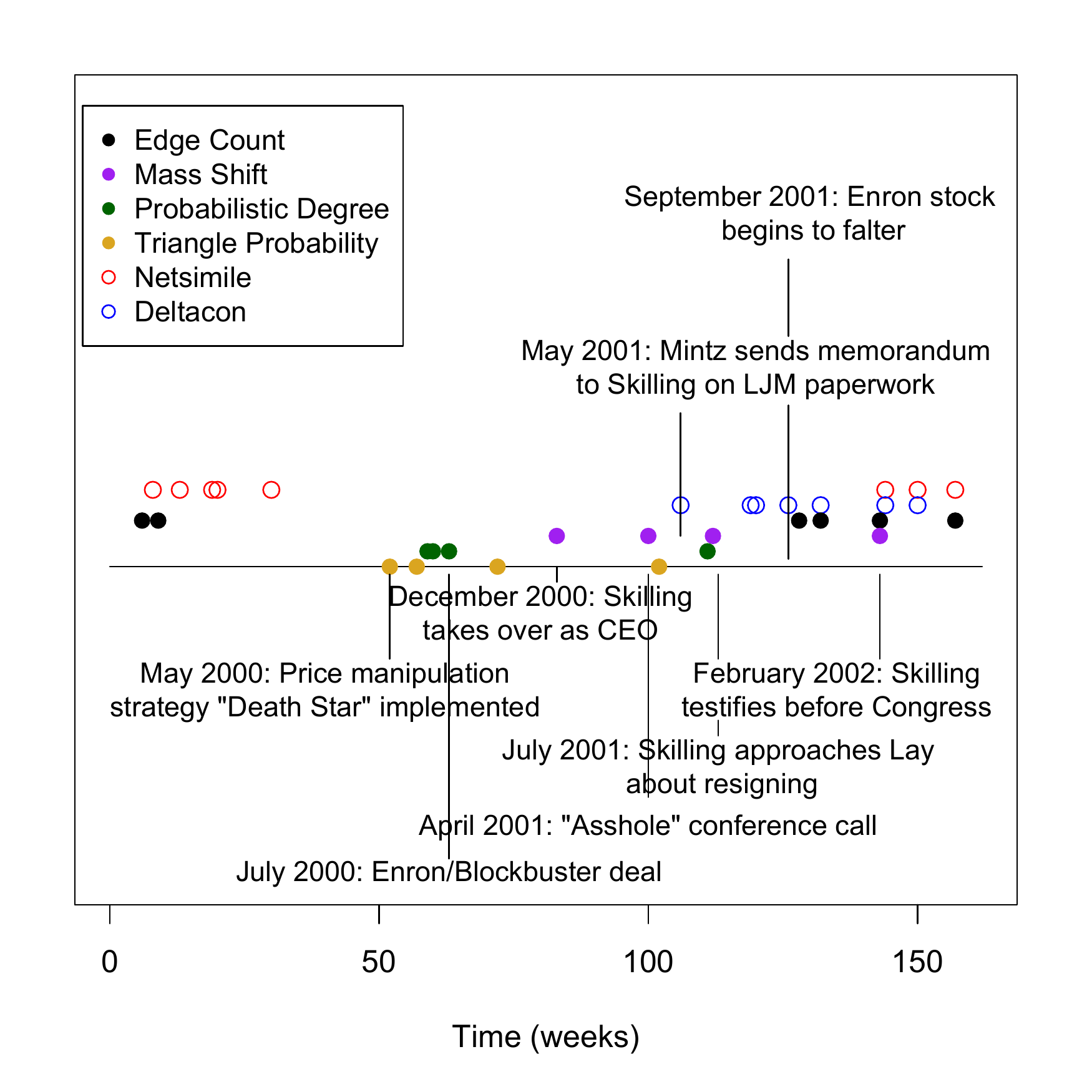}
\caption{Detected anomalies in Enron corporation e-mail dataset.  Filled circles are detections from our proposed statistics, open circles are other methods.}
\label{fig:enrontimeline}
\end{figure*}

\begin{figure*}[h!]
\centering
\includegraphics[width=.65\columnwidth,natwidth=610,natheight=642]{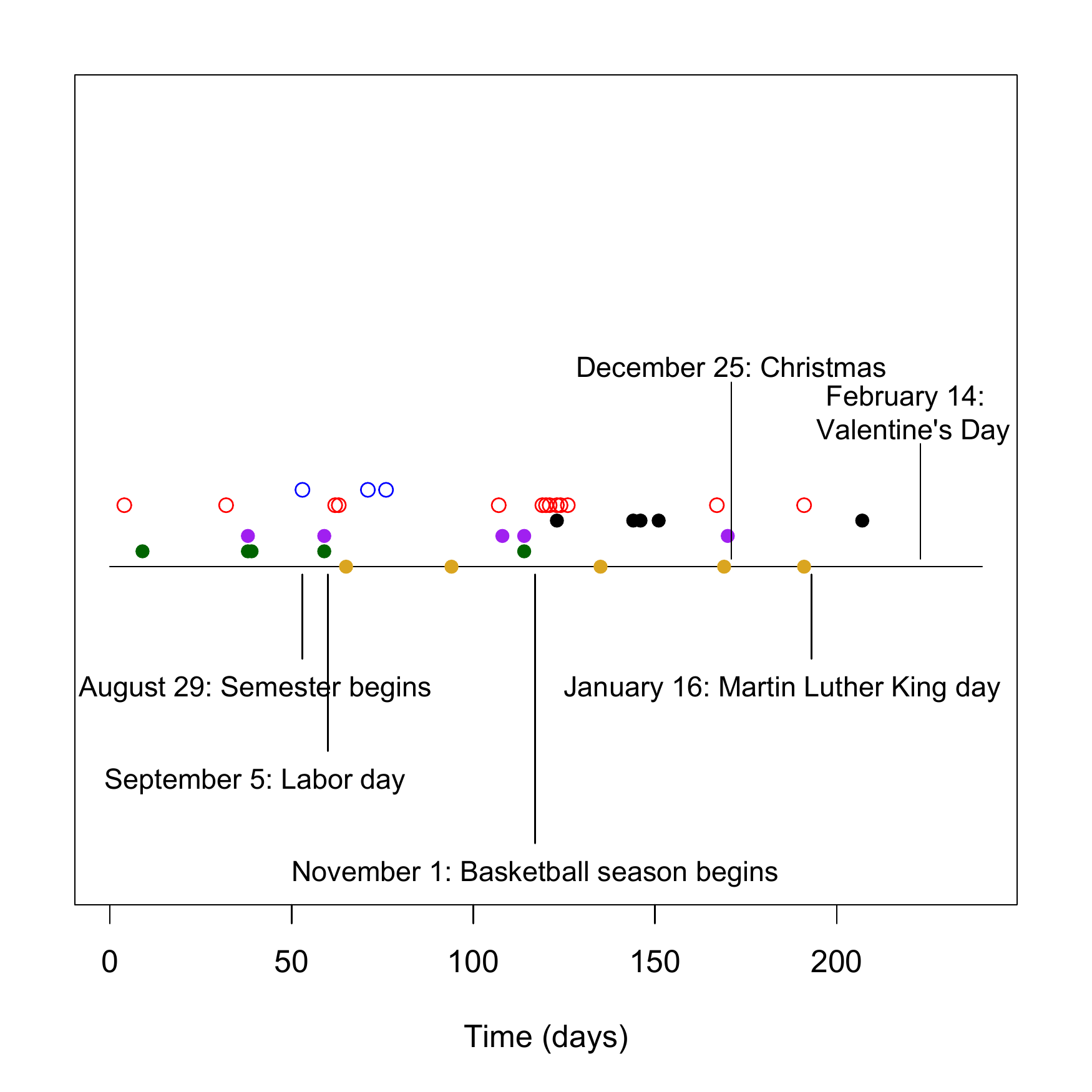}
\caption{Detected anomalies in university student e-mail dataset. }
\label{fig:purduetimeline}
\end{figure*}

\begin{figure*}[h!]
\centering
\includegraphics[width=.65\columnwidth,natwidth=610,natheight=642]{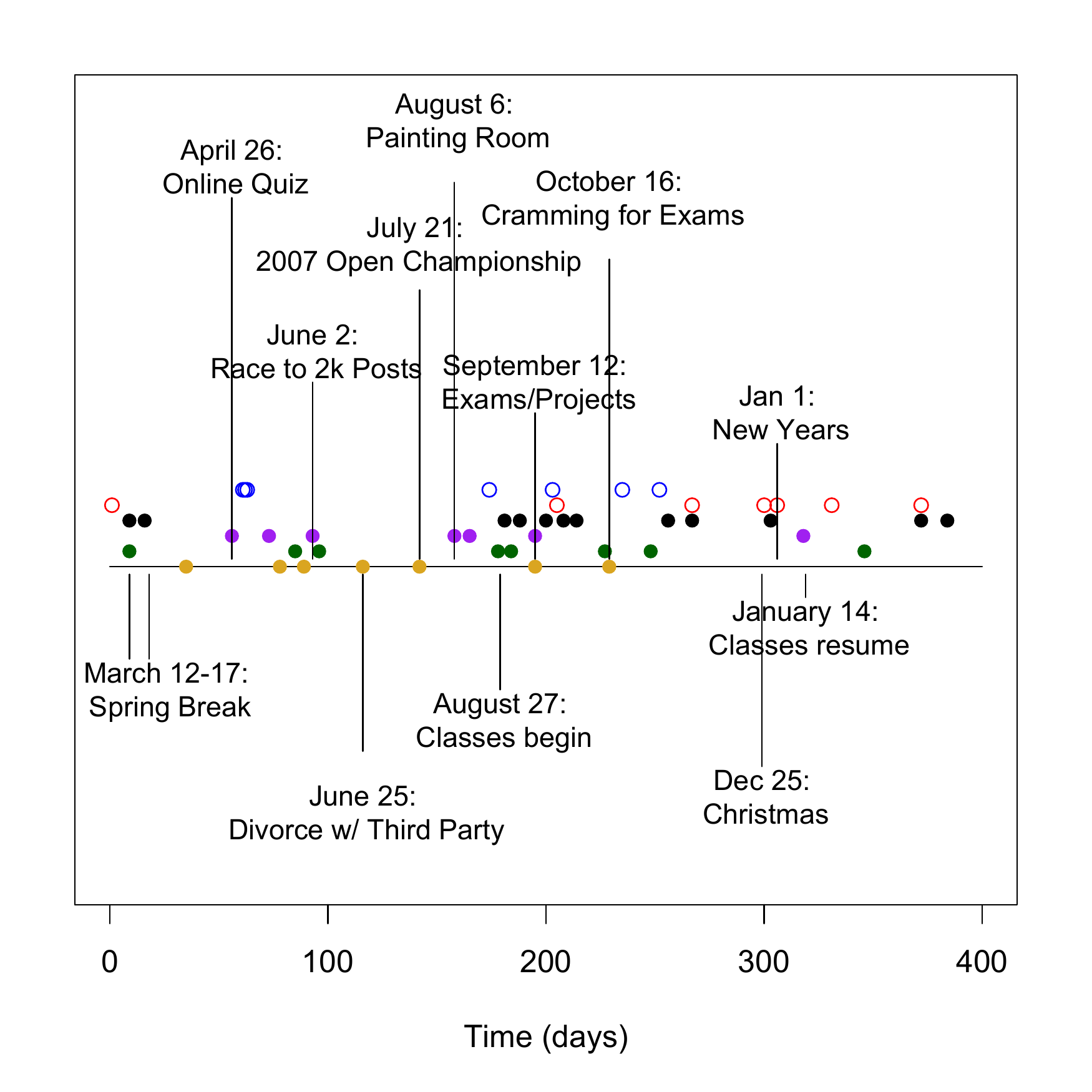}
\caption{Detected anomalies in Facebook wall postings dataset. }
\label{fig:facebooktimeline}
\end{figure*}

Figure \ref{fig:enrontimeline} shows the results of multiple statistics detectors when applied to the set of e-mail data from the Enron corporation, including our three proposed statistics, the raw message count, Netsimile, and Deltacon.  Time step 143 represents the most significant event in the stream, Jeffrey Skilling's testimony before congress on February 6, 2002.  The detected triangle anomalies at time steps 50-60 coincide with Enron's price manipulation strategy known as ``Death Star'' which was put into action in May 2000.  Other events include The CEO transition from Lay to Skilling in December 2000, the ``asshole'' conference call featured prominently in the book ``The Smartest Guys in the Room,'' and Lay approaching Skilling about resigning.  

Netsimile has difficulty detecting most of important events in the Enron timeline.  Although it accurately flags the time of the Congressional hearings, the other points flagged, particularly early on, do not correspond to any notable events and are probably false positives due to the artificial sensitivity of the algorithm in very sparse network slices.  

Deltacon detects a greater range of events than Netsimile but still fails to detect several important events such as the price manipulation and Skilling's attempted resignation.  In general it generates detections more frequently in the region between May and December 2001 which is also the region of highest message activity, and fails to generate detections in times with fewer messages.  


Figure \ref{fig:purduetimeline} shows the detected time steps of the University E-mail dataset.  Several major events from the academic school year like the start of the school year and Christmas break are shown.  It seems that the consistent statistics flag times closer to holidays and other events compared to other statistics.  Unfortunately, as the text content of the messages was unavailable it is impossible to determine if the detected conversations correspond to specific events based on the dialogues of users.  


Figure \ref{fig:facebooktimeline} shows the detected events of the Facebook wall data and the explanations for the detected events.  Some of the listed events are holidays while others were obtained by investigating the time steps flagged as anomalies; see Section \ref{local-anomaly-decomposition} for an explanation of this process.  Some events of interest are: the ``Race to 2k Posts'' where a pair of individuals noticed they were nearing two thousand posts on one of their walls and decided to reach that mark in one night, generating much more traffic between them than usual (over 160 posts); the ``Divorce w/ Third Party'' where a pair of individuals were going through a messy breakup and a mutual friend was cracking jokes and egging them on; and a discussion about Tiger Woods' odds in the 2007 Open Championship.

\section{Local Anomaly Decomposition}
\label{local-anomaly-decomposition}

After flagging a time step as anomalous it is useful to have some indication as to what is happening in the network at that time that generated the flag.  One tool for investigating the flagged time step is local anomaly decomposition, where the network is broken down into subgraphs that contribute the most value to the total statistic score at that time step.  For many statistics like mass shift or triangle probability which are summations over the edges, nodes, or triplets of the graph this process is trivial: each component of the summation has an associated anomaly score and the components that provide the most anomaly score are the ones investigated.  For others such as PDD which cannot be easily decomposed into node and edge contributions this approach is nearly impossible.  Anomaly score decomposition is more useful when the score is skewed rather than uniformly distributed as it is easier to highlight a concise region that contributes the most towards the anomaly.

To demonstrate the decomposition, we applied the statistics to the real-world networks and sorted all of the nodes (for Barrat clustering) or edges (all other statistics) from highest to lowest contribution to the anomaly score sum.  From there we selected the components with the highest anomaly score contribution totaling at least 20\% of the log of the anomaly score to be part of the visualized anomaly.  We then plotted all of the selected components as well as any adjacent edges and nodes.  We investigated the Enron and Facebook datasets as these have names/message content associated with the graphs; the e-mail dataset has neither so these graphs are omitted.  


Figures \ref{fig:enronmasslocal} - \ref{fig:enronbarratlocal} show the local subgraphs reported by the mass shift, triangle probability, graph edit distance, and Barrat clustering respectively.  The left subgraph shows activity in the time step immediately prior to the anomaly while the right shows the subgraph during the anomaly.  Red nodes and edges are part of the top anomaly contributors while black edges and nodes are merely adjacent; the thickness of the edges corresponds to the edge weight in that time step.

Figure \ref{fig:enronmasslocal} shows an unusually large amount of communication between Senior Vice President Richard Shapiro and Government Relations Executive Jeff Dasovich immediately before Lay approaches Skilling about resigning as CEO.  Figure \ref{fig:enrontrilocal} shows the triangular communications occurring between members of the Enron legal department which was occurring during the price-fixing strategy in California.  Both of these methods find succinct subgraphs to represent the anomalies occurring at these times.

\ref{fig:enrongedlocal}, on the other hand, shows graph edit distance reporting nearly the entirety of the network at that time.  While this does represent an event (the Congressional hearings) there is no interpretation of the event other than that there were many messages being sent at that time.  Barrat clustering identifies the legal department in \ref{fig:enronbarratlocal} but does so at a time with relatively low communication.  Barrat clustering normalizes by node degree which makes it more likely to report triangles with less weight as long as the participating nodes don't communicate with anyone else.

\clearpage

\begin{figure*}[h!]
\begin{center}
\subfigure[]{\includegraphics[width=.39\columnwidth,natwidth=610,natheight=642,trim=0 300 0 0,clip=true]{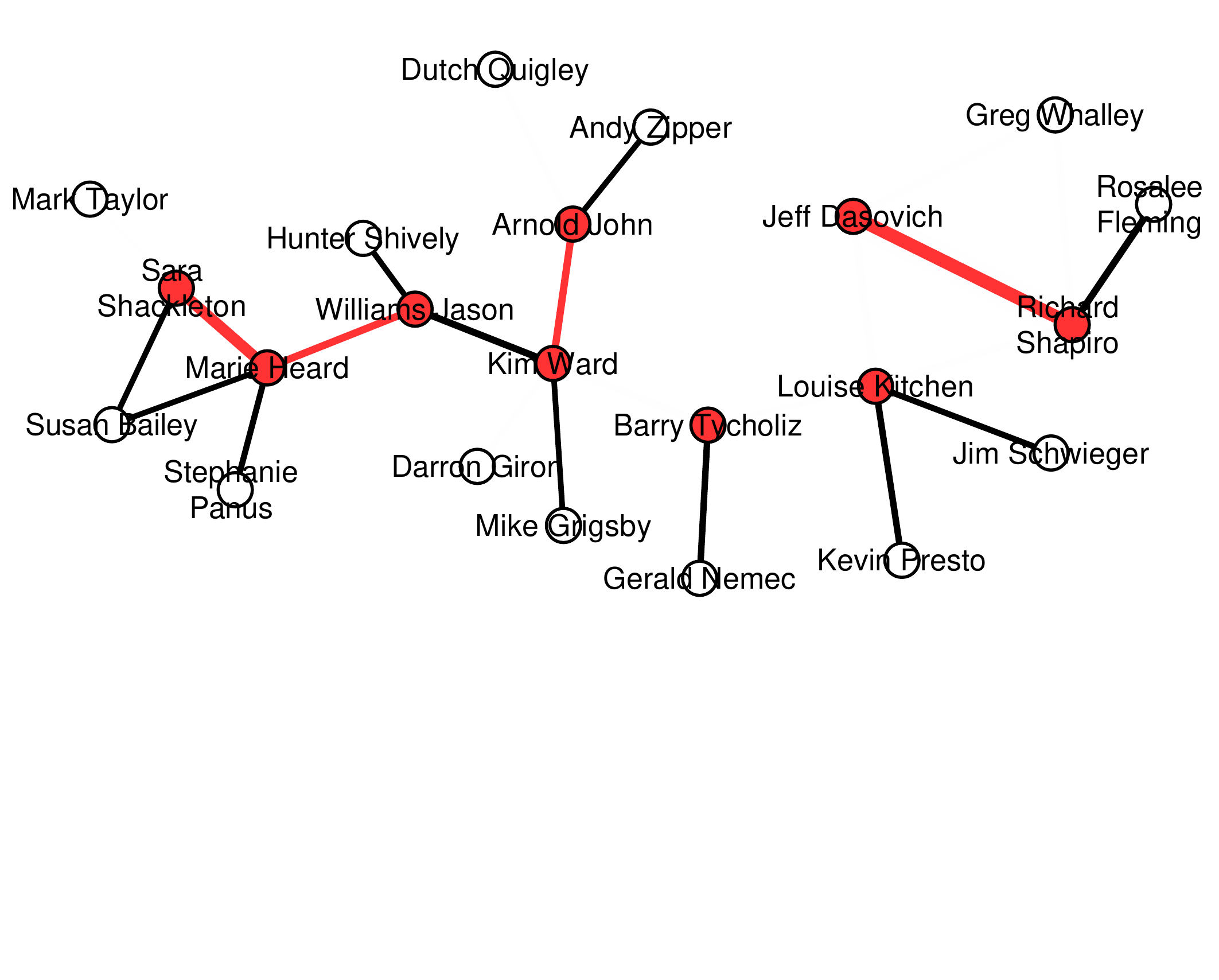}}
\subfigure[]{\includegraphics[width=.39\columnwidth,natwidth=610,natheight=642,trim=0 300 0 0,clip=true]{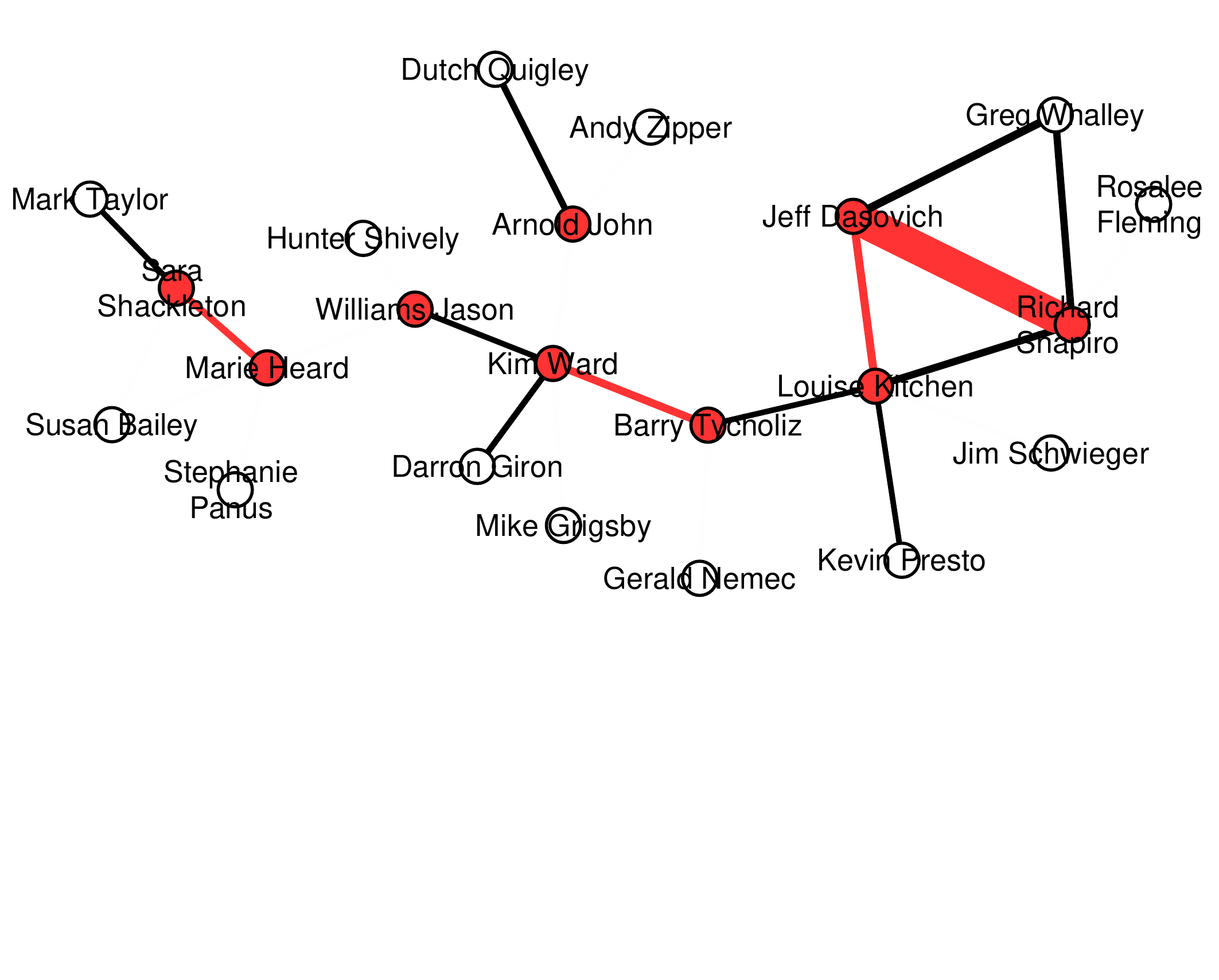}}
\end{center}

\caption{Subgraph responsible for most of the mass shift anomaly in the Enron network at the weeks of June 25 (before anomaly) and July 2 (during anomaly), 2001 respectively.  }
\label{fig:enronmasslocal}
\end{figure*}

\begin{figure*}[h!]
\begin{center}
\subfigure[]{\includegraphics[width=.39\columnwidth,natwidth=610,natheight=642,trim=0 300 0 0,clip=true]{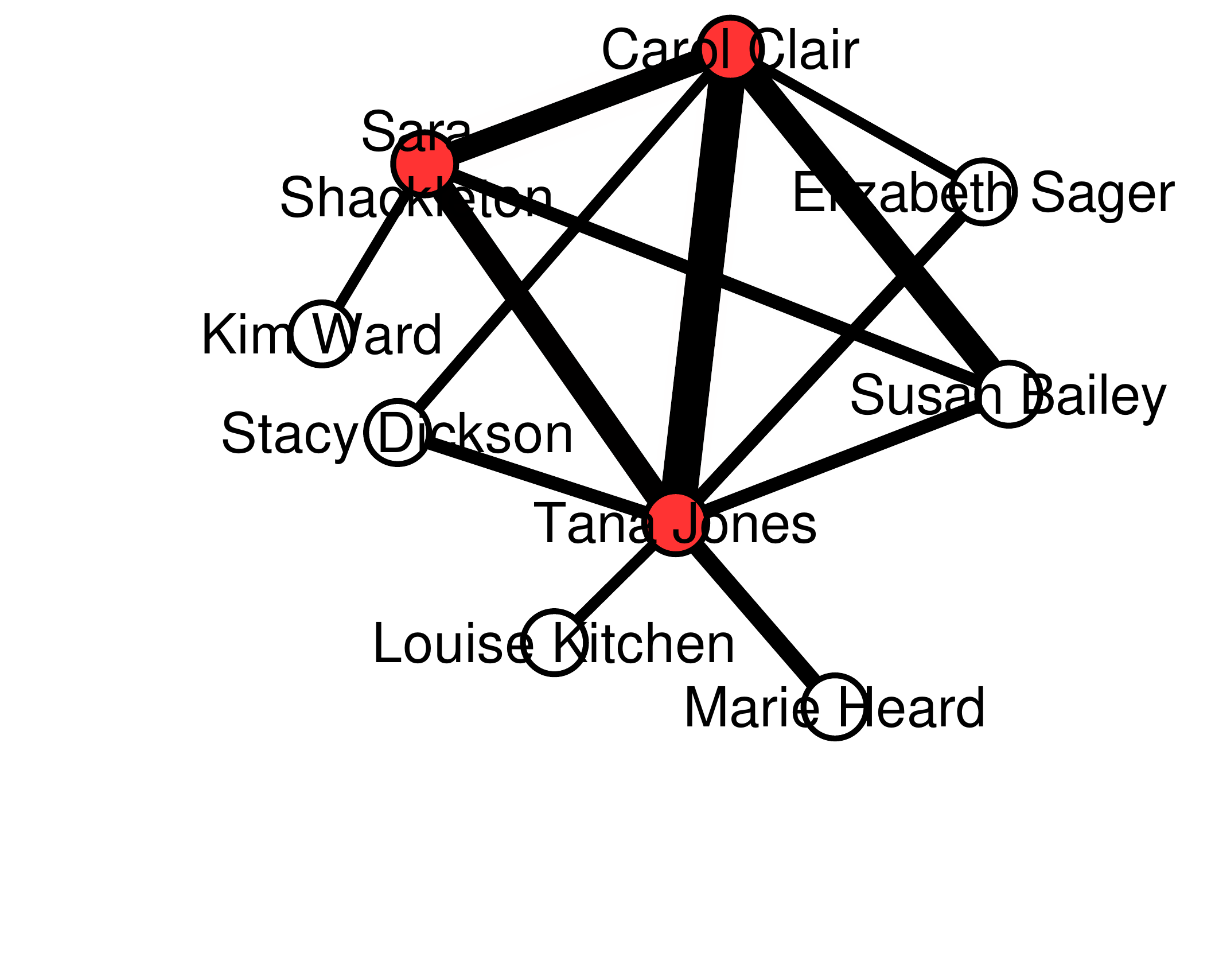}}
\subfigure[]{\includegraphics[width=.39\columnwidth,natwidth=610,natheight=642,trim=0 300 0 0,clip=true]{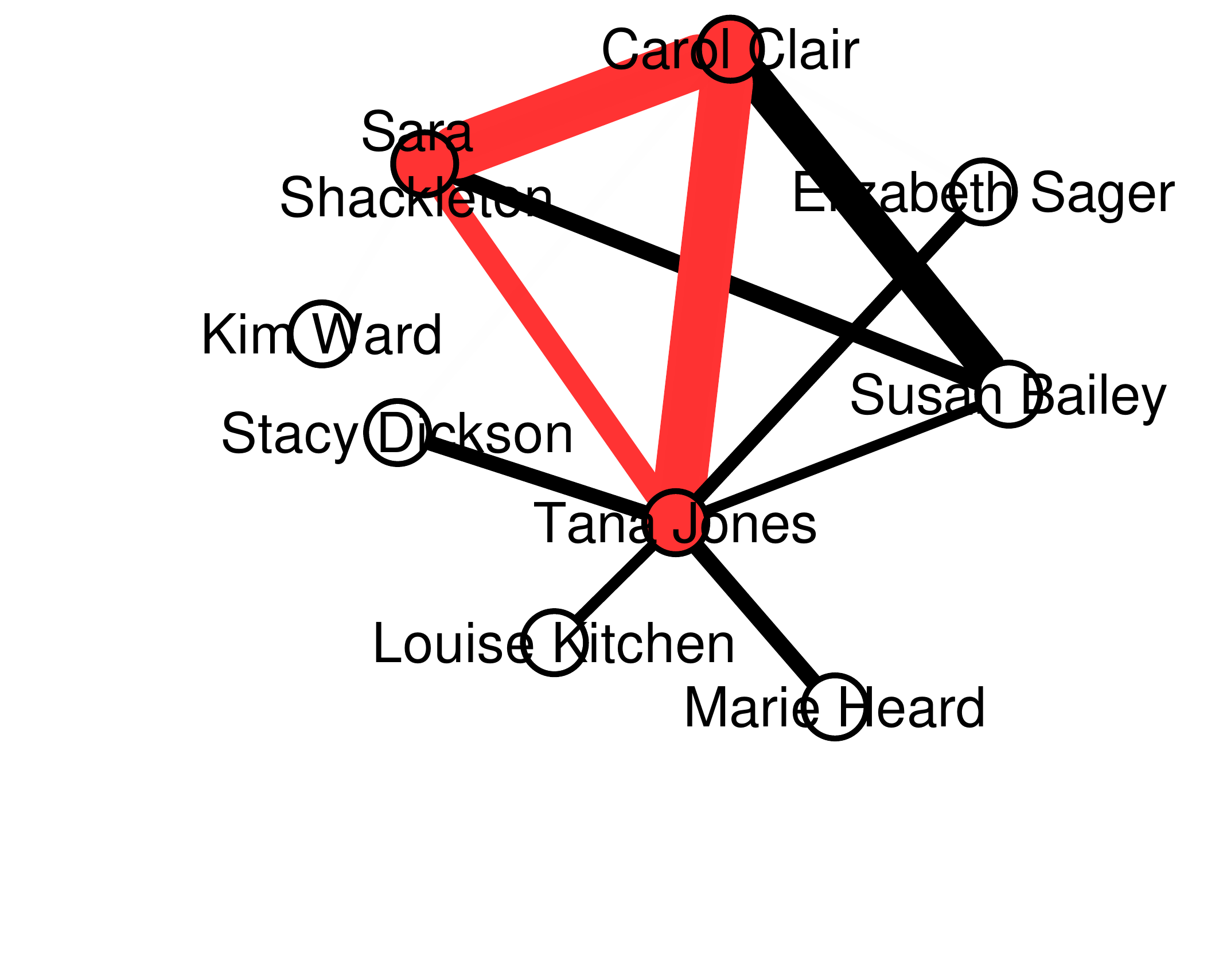}}
\end{center}

\caption{Subgraph responsible for most of the triangle probability anomaly in the Enron network at the weeks of May 1 (before anomaly) and May 8 (during anomaly), 2000 respectively.  }
\label{fig:enrontrilocal}
\end{figure*}

\begin{figure*}[h!]
\begin{center}
\subfigure[]{\includegraphics[width=.39\columnwidth,natwidth=610,natheight=642,trim=0 300 0 0,clip=true]{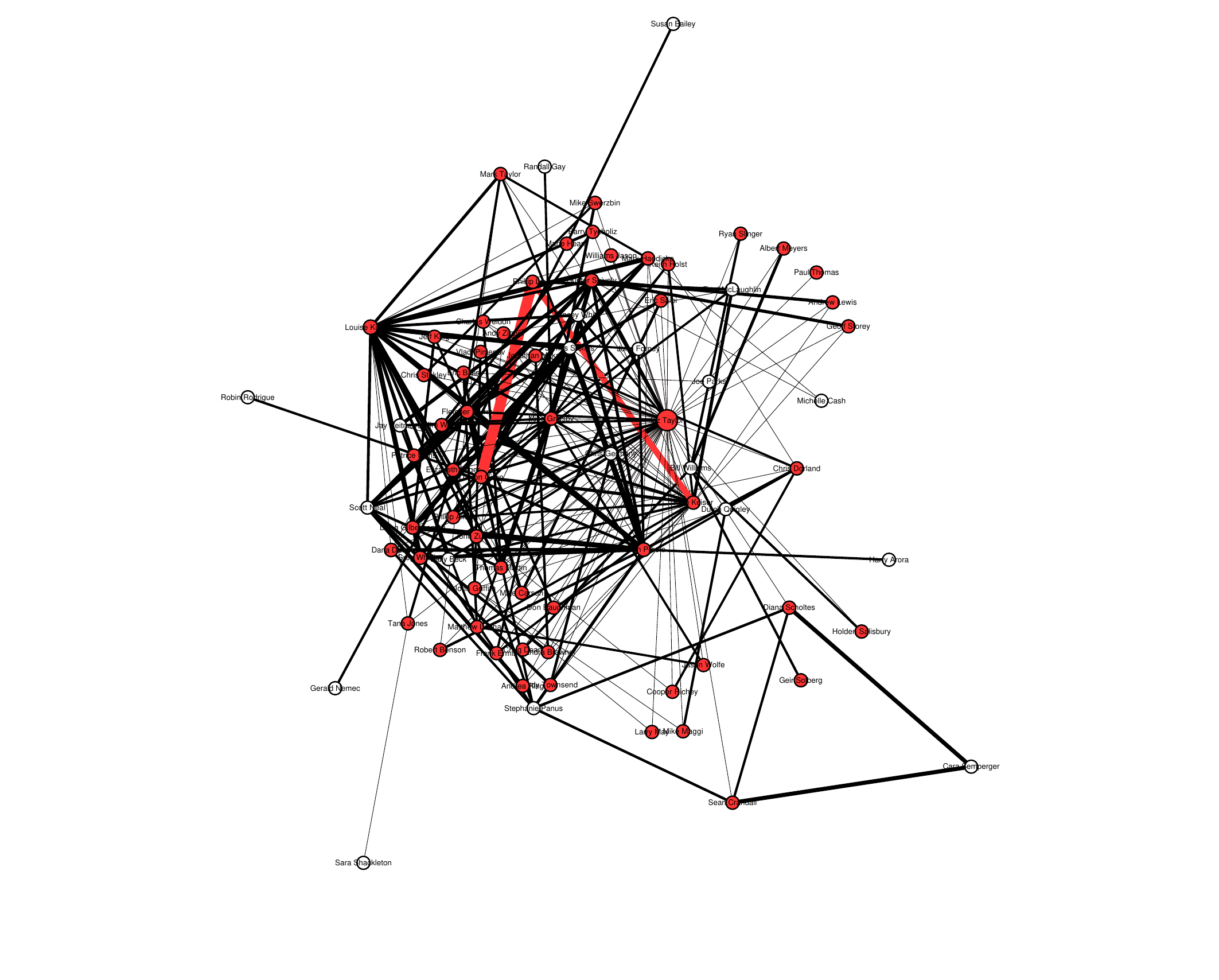}}
\subfigure[]{\includegraphics[width=.39\columnwidth,natwidth=610,natheight=642,trim=0 300 0 0,clip=true]{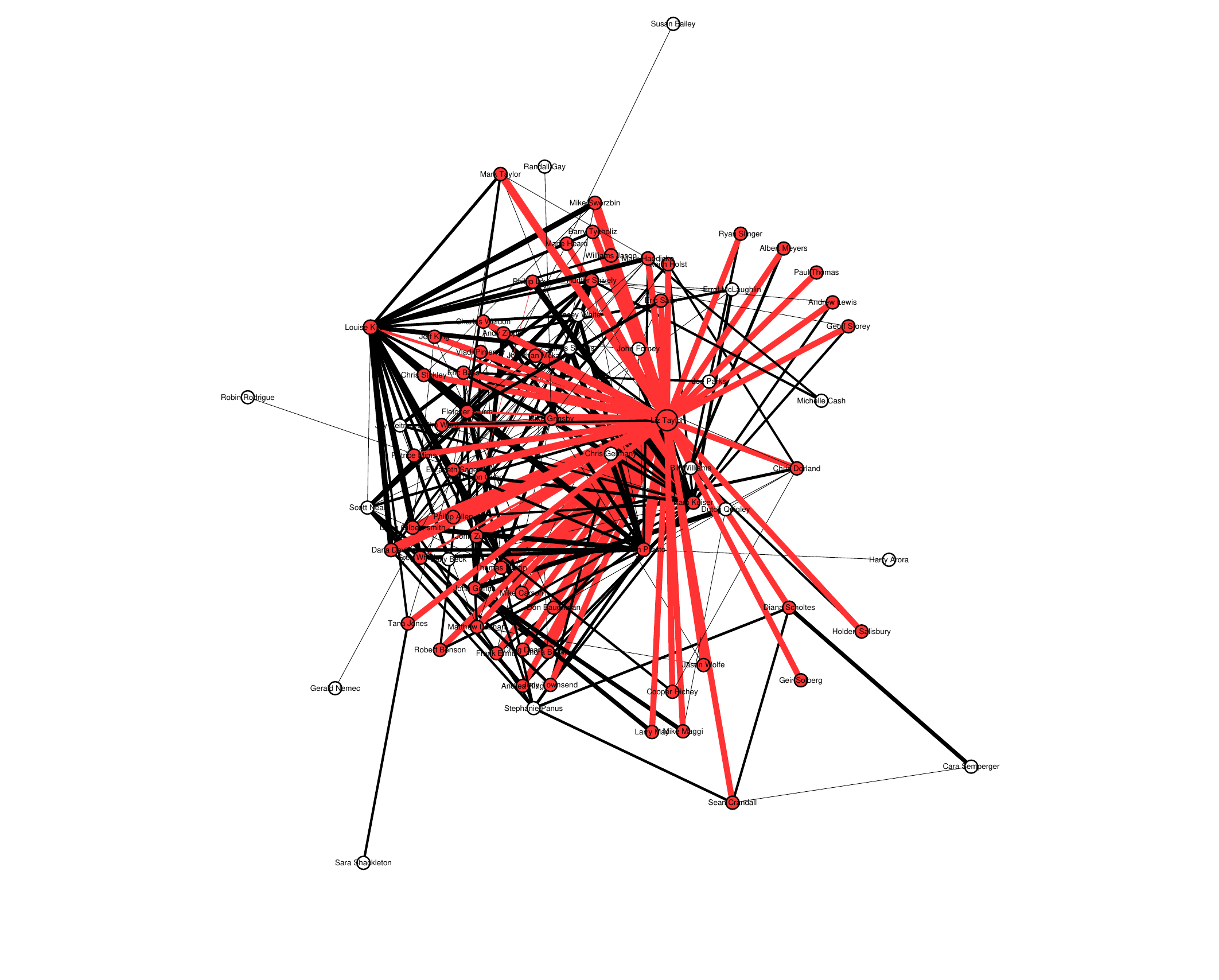}}
\end{center}

\caption{Subgraph responsible for most of the graph edit distance anomaly in the Enron network at the weeks of January 28 (before anomaly) and Feburary 4 (during anomaly), 2002 respectively.  }
\label{fig:enrongedlocal}
\end{figure*}

\begin{figure*}[h!]
\begin{center}
\subfigure[]{\includegraphics[width=.39\columnwidth,natwidth=610,natheight=642,trim=0 300 0 0,clip=true]{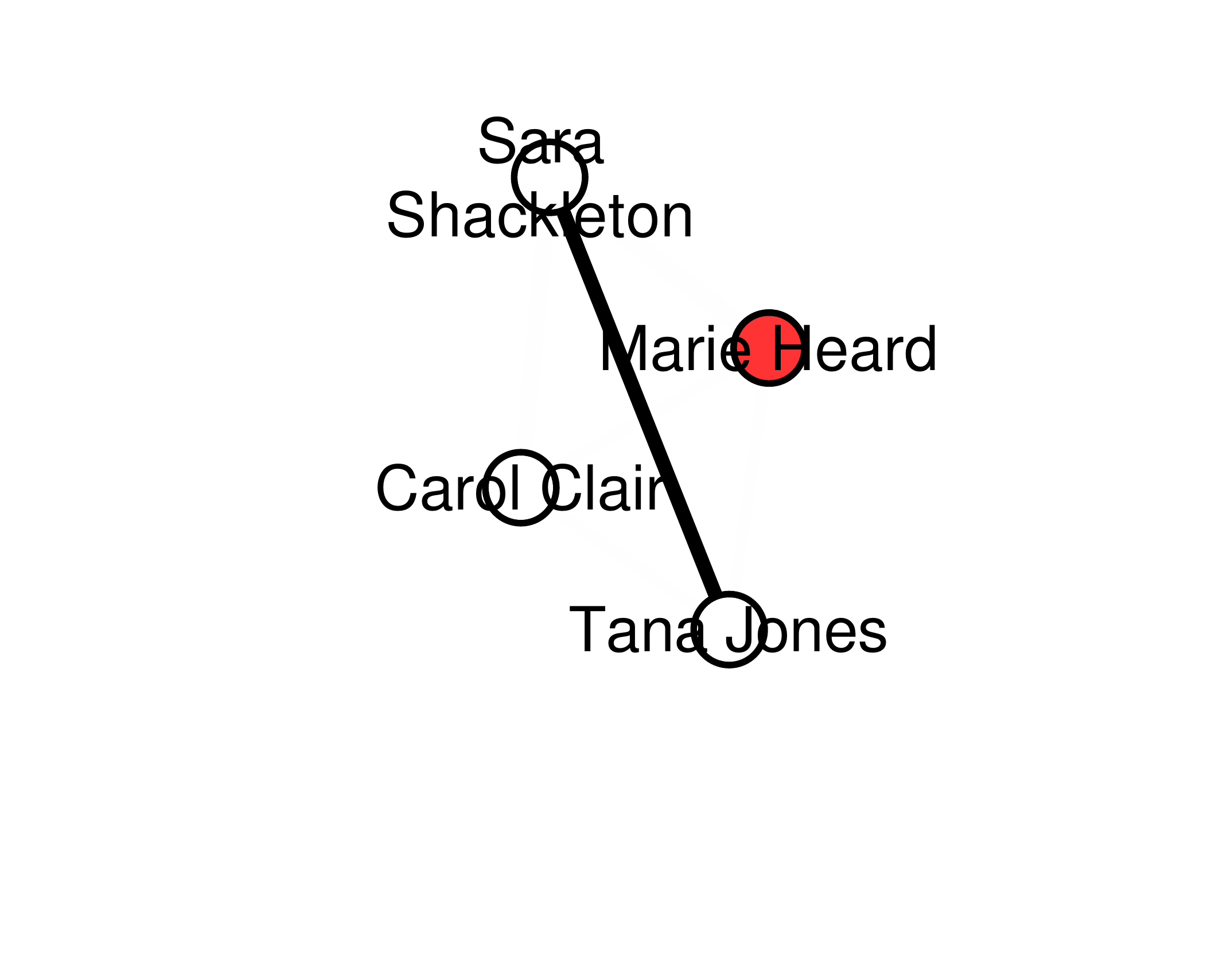}}
\subfigure[]{\includegraphics[width=.39\columnwidth,natwidth=610,natheight=642,trim=0 300 0 0,clip=true]{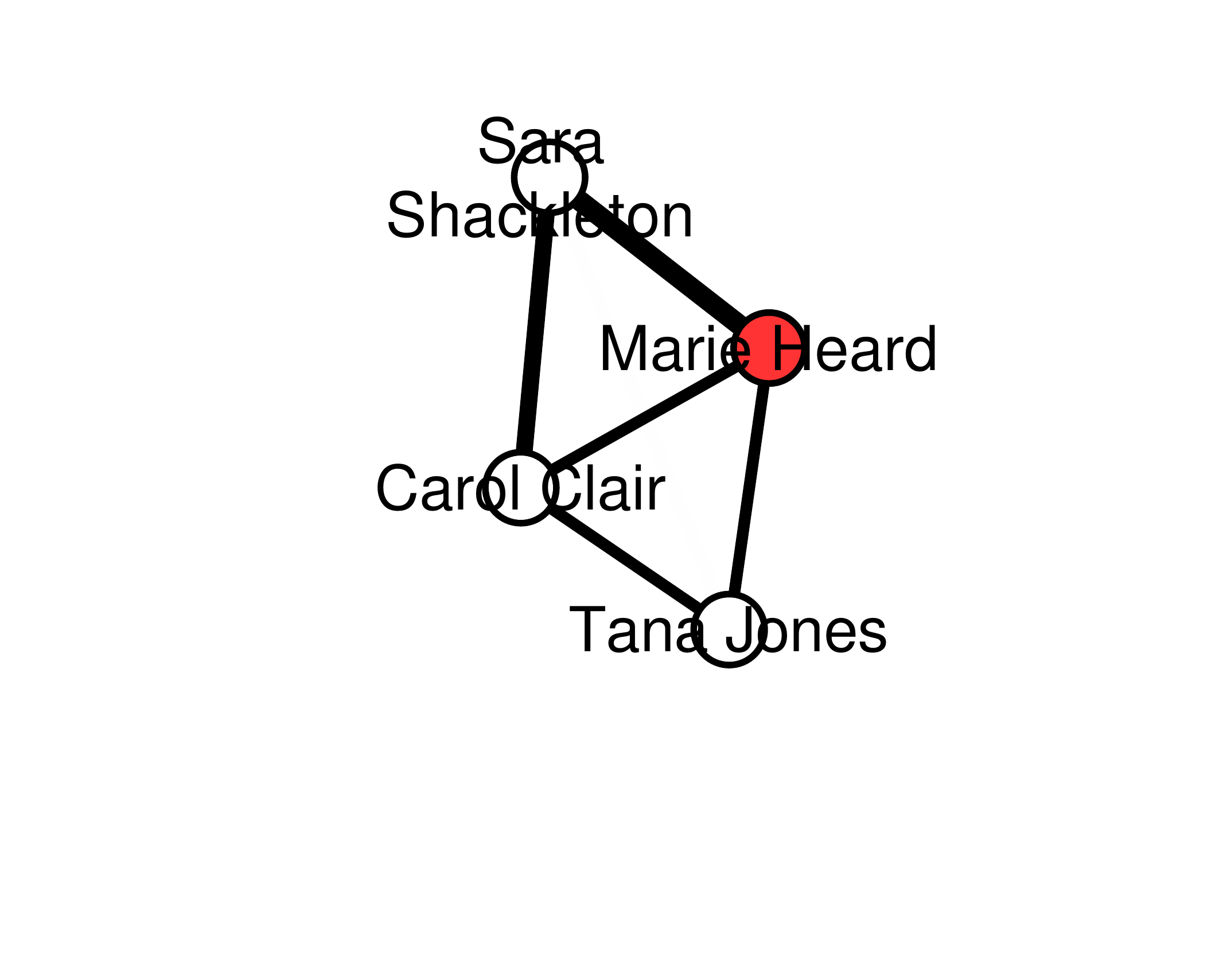}}
\end{center}

\caption{Subgraph responsible for most of the Barrat clustering anomaly in the Enron network at the weeks of November 1 (before anomaly) and November 8 (during anomaly), 1999 respectively.  }
\label{fig:enronbarratlocal}
\end{figure*}

\begin{figure*}[h!]
\begin{center}
\subfigure[]{\includegraphics[width=.39\columnwidth,natwidth=610,natheight=642,trim=0 300 0 0,clip=true]{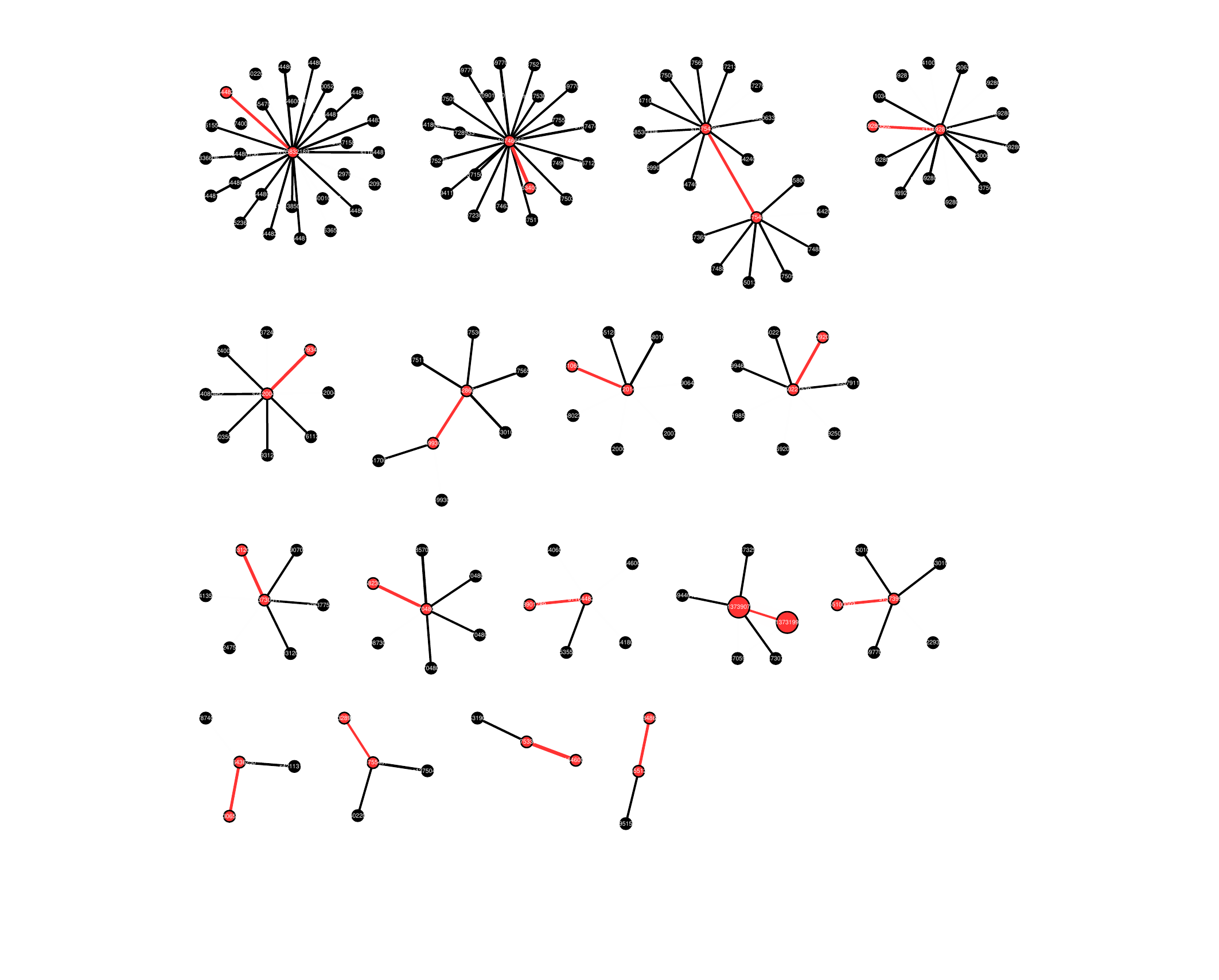}}
\subfigure[]{\includegraphics[width=.39\columnwidth,natwidth=610,natheight=642,trim=0 300 0 0,clip=true]{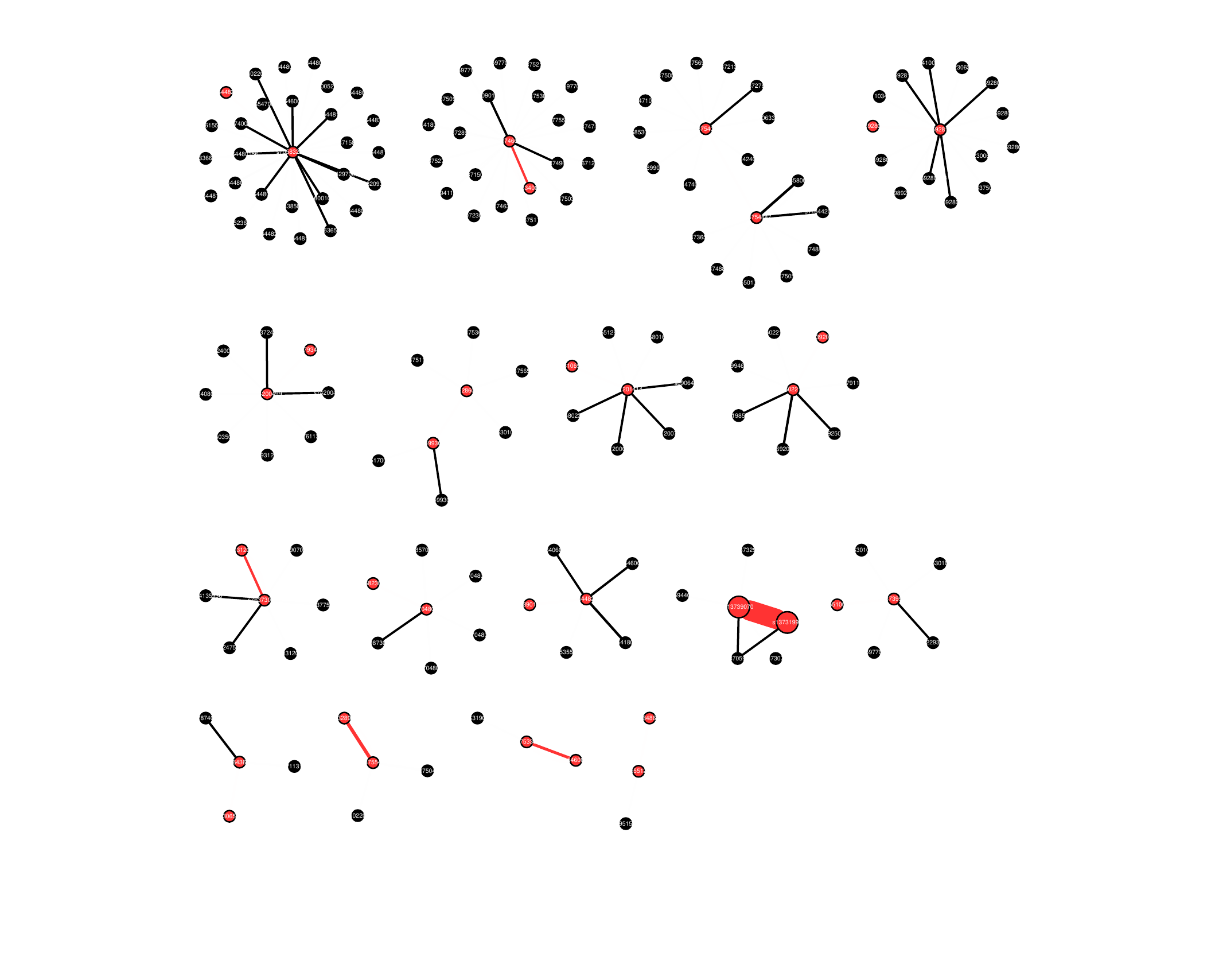}}
\end{center}

\caption{Subgraph responsible for most of the mass shift anomaly in the Facebook network at June 1 (before anomaly) and June 2 (during anomaly), 2007 respectively.  }
\label{fig:facebookmasslocal}
\end{figure*}

\begin{figure*}[h!]
\begin{center}
\subfigure[]{\includegraphics[width=.39\columnwidth,natwidth=610,natheight=642,trim=0 300 0 0,clip=true]{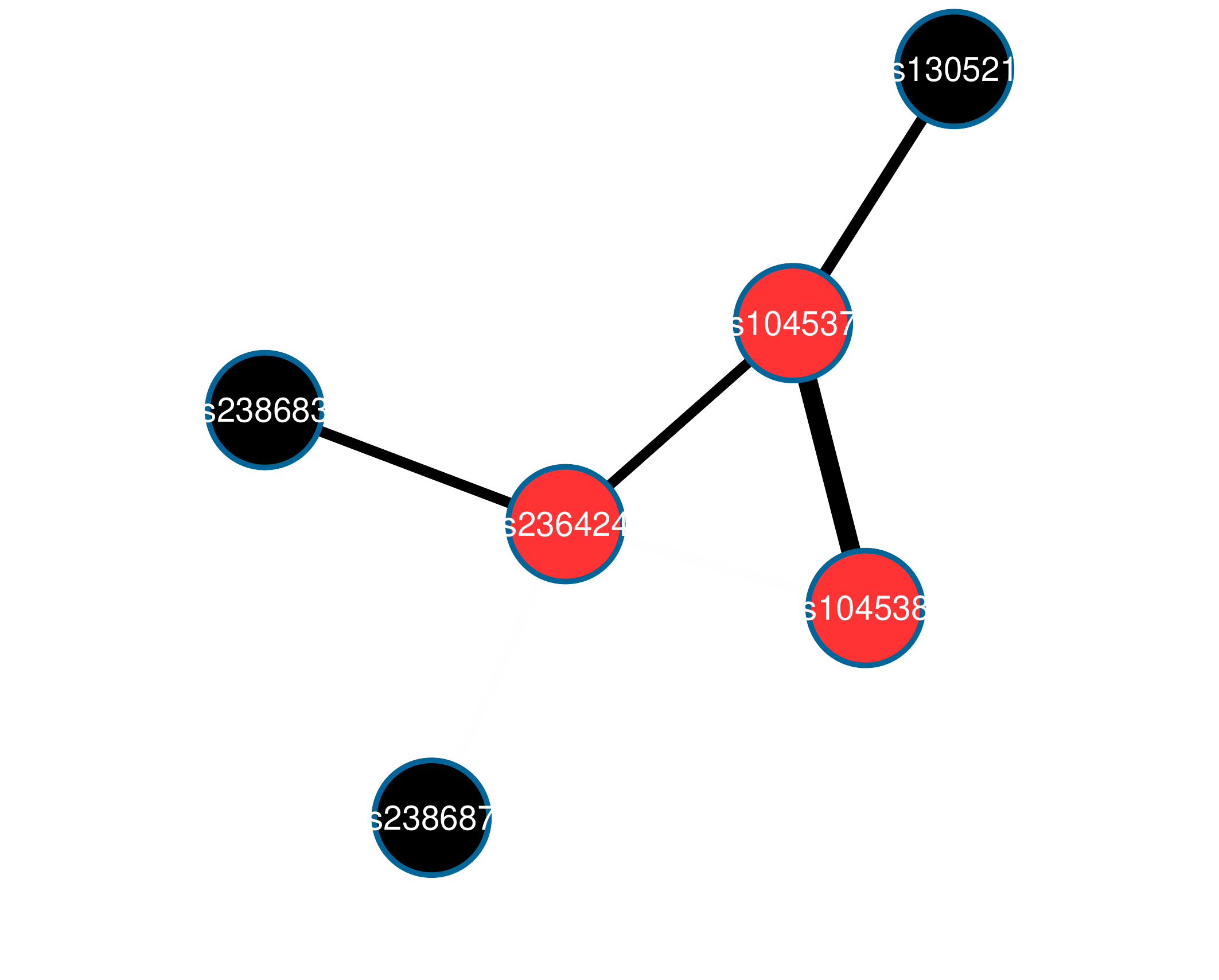}}
\subfigure[]{\includegraphics[width=.39\columnwidth,natwidth=610,natheight=642,trim=0 300 0 0,clip=true]{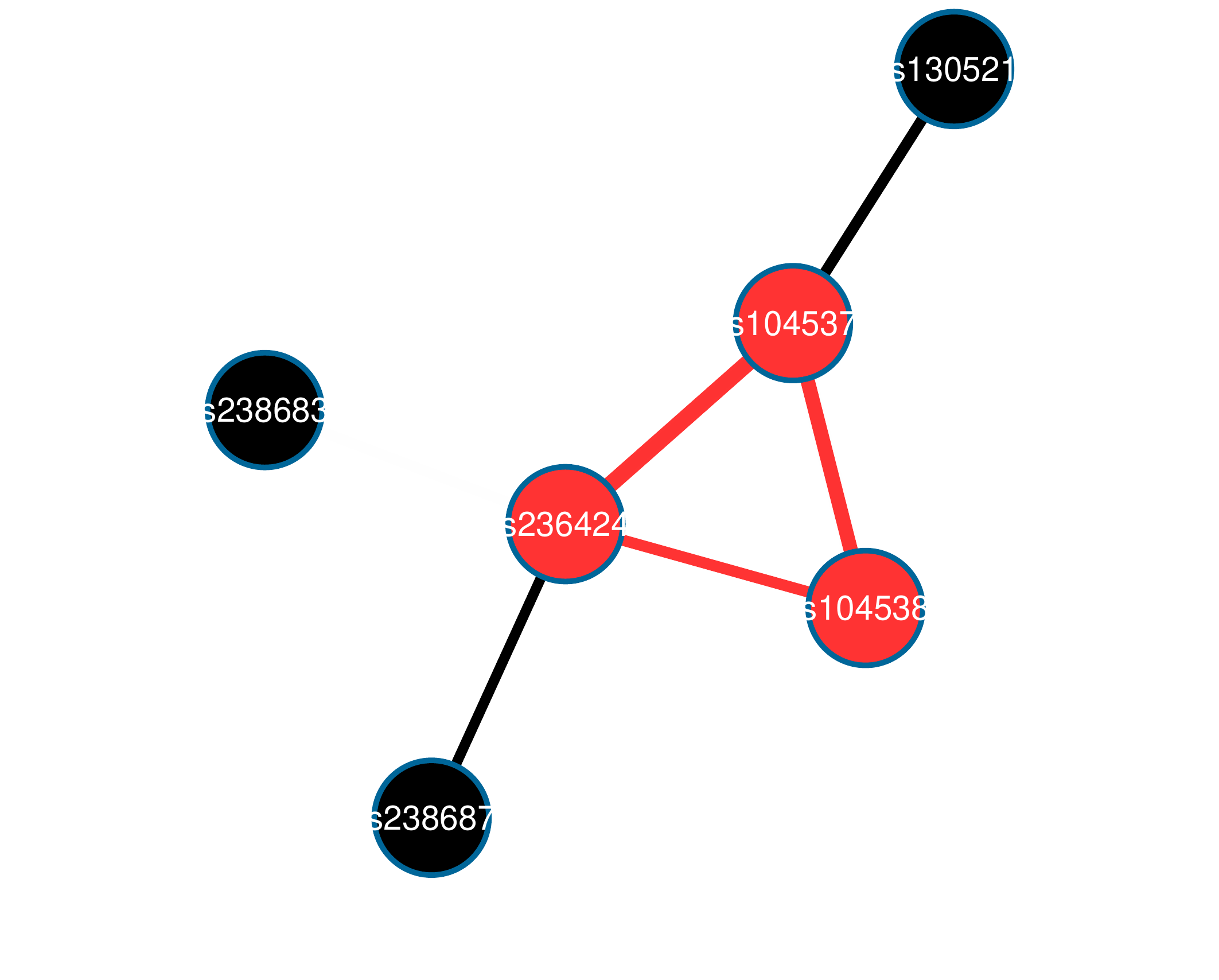}}
\end{center}

\caption{Subgraph responsible for most of the triangle probability anomaly in the Facebook network at July 20 (before anomaly) and July 21 (during anomaly), 2007 respectively.  }
\label{fig:facebooktrilocal}
\end{figure*}

\begin{figure*}[h!]
\begin{center}
\subfigure[]{\includegraphics[width=.39\columnwidth,natwidth=610,natheight=642,trim=0 300 0 0,clip=true]{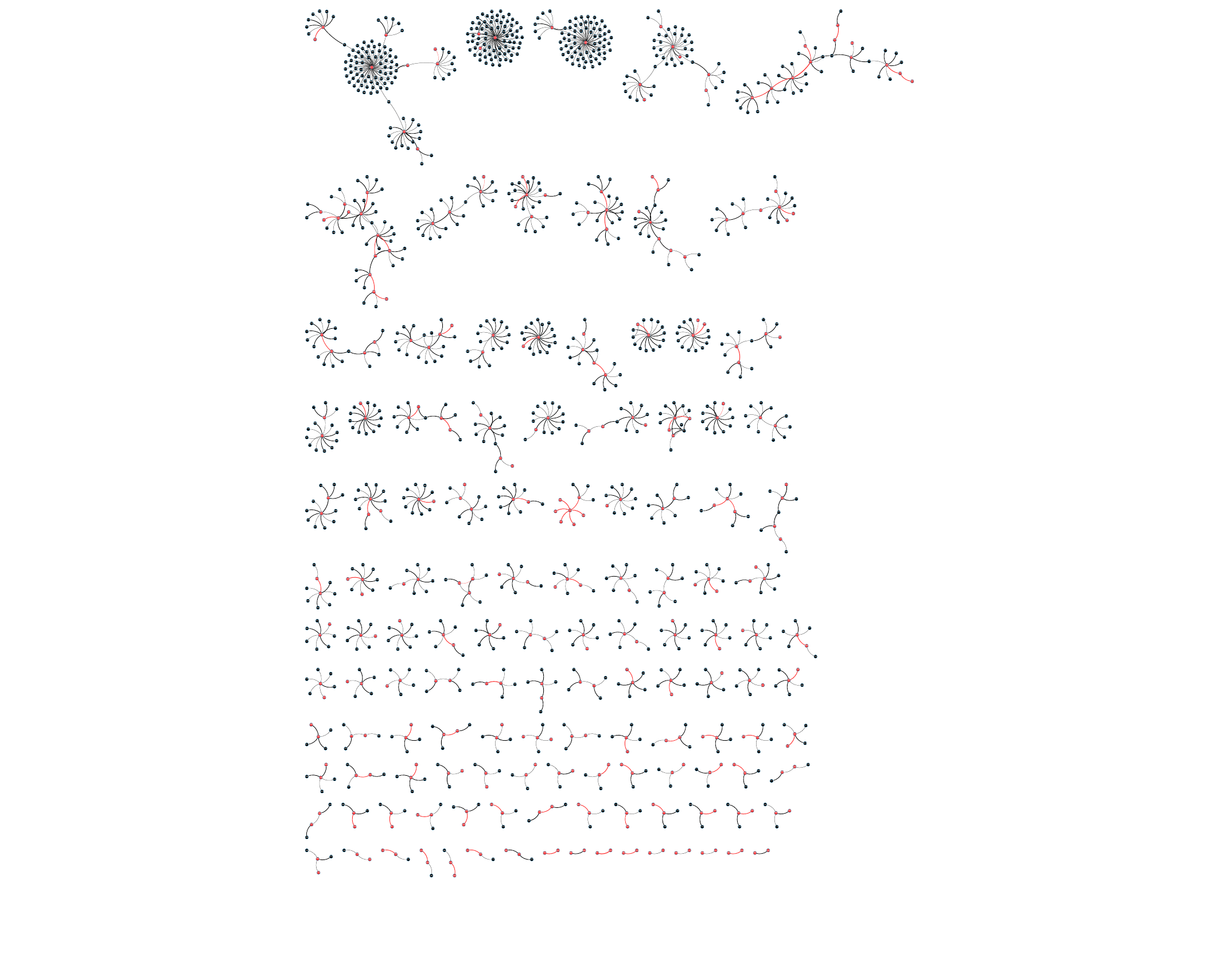}}
\subfigure[]{\includegraphics[width=.39\columnwidth,natwidth=610,natheight=642,trim=0 300 0 0,clip=true]{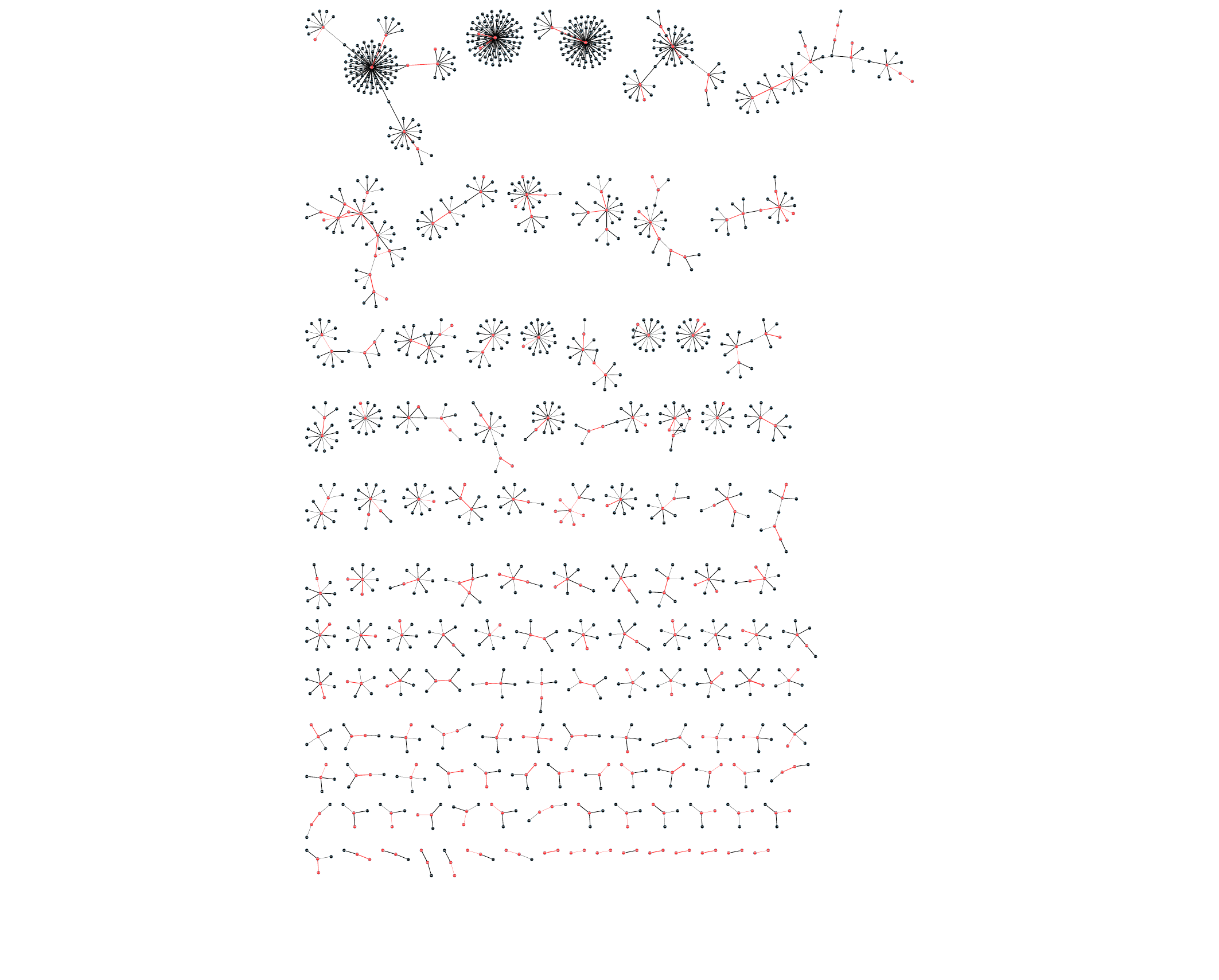}}
\end{center}

\caption{Subgraph responsible for most of the graph edit distance anomaly in the Facebook network at October 15 (before anomaly) and October 16 (during anomaly), 2007 respectively.  }
\label{fig:facebookgedlocal}
\end{figure*}

\begin{figure*}[h!]
\begin{center}
\subfigure[]{\includegraphics[width=.39\columnwidth,natwidth=610,natheight=642,trim=0 300 0 0,clip=true]{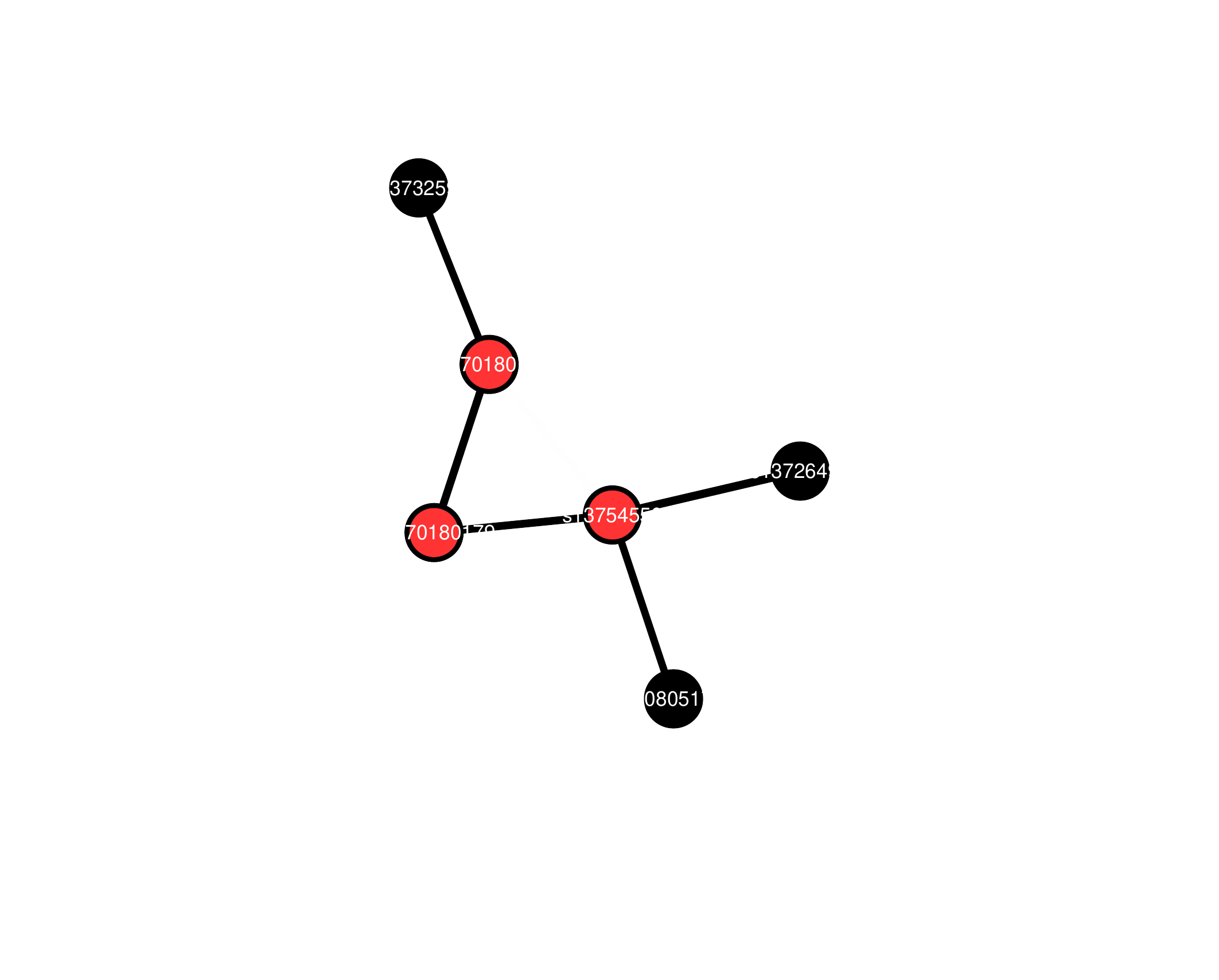}}
\subfigure[]{\includegraphics[width=.39\columnwidth,natwidth=610,natheight=642,trim=0 300 0 0,clip=true]{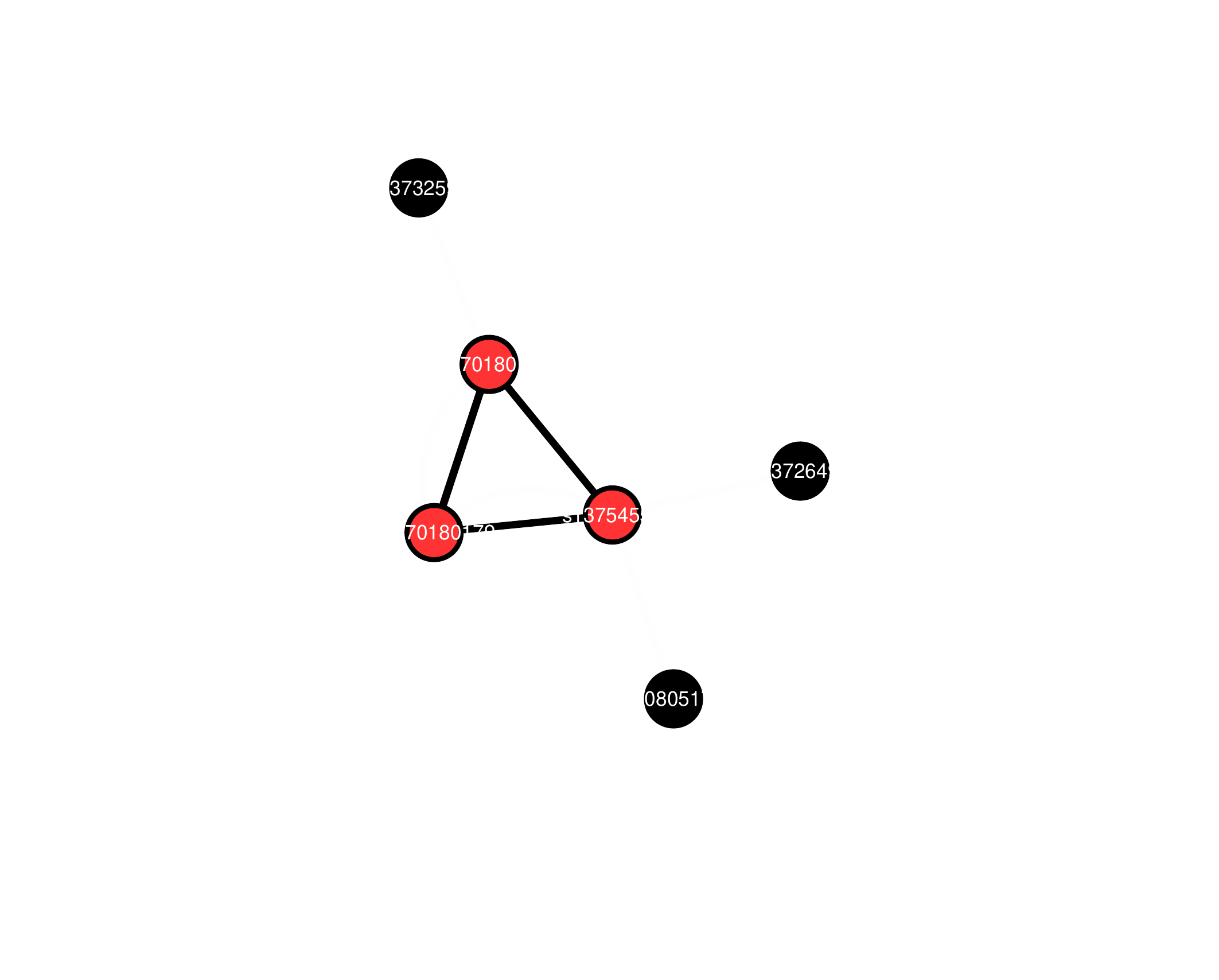}}
\end{center}

\caption{Subgraph responsible for most of the Barrat clustering anomaly in the Facebook network at May 17 (before anomaly) and May 18 (during anomaly), 2007 respectively.  }
\label{fig:facebookbarratlocal}
\end{figure*}

\clearpage

Figures \ref{fig:facebookmasslocal} - \ref{fig:facebookbarratlocal} show the local subgraphs found in the Facebook dataset.  \ref{fig:facebookmasslocal} shows the event we named ``race to 2k posts;'' at this time a pair of individuals noticed they were closing in on two thousand posts on their walls and decided to reach that goal in one night.  The result is a massively higher amount of communication than was typical between the two in prior time steps.  \ref{fig:facebooktrilocal} shows the communications occurring during the 2007 Open Championship golf tournament.  The three individuals with the most communication were arguing about the odds that Tiger Woods would win the tournament.

Graph edit distance, by contrast, identifies no coherent local structure in \ref{fig:facebookgedlocal}.  It is likely that this event signifies a global increase in communication rather than a change in the distribution of messages.  As the additional edges were distributed throughout the network, when looking for subgraphs that generated the most anomaly score the majority of the network has similar scores so a random chunk of the network is found.  \ref{fig:facebookbarratlocal} is the structure found by Barrat clustering; as before it finds a set of triangular communication with relatively low weights, around 2 -- 4, while the anomaly found by Triangle Probability has about 12 messages per edge. 

\section{Conclusions}

In this paper we have demonstrated that dependence on network edge count hinders the ability of statistics to detect certain changes in dynamic networks.  To remedy this we have introduced the concept of Size Consistency and shown that statistics with this property are less affected by edge count variation.


We proposed three Size Consistent network statistics, Mass Shift, Degree Shift, and Triangle Probability to replace the Graph Edit Distance, Degree Distribution and Clustering Coefficient statistics.  These statistics are provably Size Consistent and we demonstrated using synthetic trials that anomaly detectors using our statistics have superior performance on variable sized networks.  



The framework for developing Size Consistent network statistics can be applied to new statistics in the future.  We hope that researchers who propose network statistics in the future will make sure to analyze the effects that changing network size have on their proposed statistics and ensure that those statistics meet the Size Consistency requirements.

\bibliographystyle{abbrv}
\bibliography{thesis-j}

\begin{thebibliography}{1}

\bibitem{chunglu}
W.~Aiello, F.~Chung, and L.~Lu.
\newblock A random graph model for massive graphs.
\newblock {\em Symposium on Theory of Computing}, 2000.

\bibitem{bohrnstedt}
G.~Bohrnstedt and A.~Goldberger.
\newblock On the exact covariance of products of random variables.
\newblock {\em Journal of the American Statistical Association}, 1969.

\bibitem{tlafond2}
T.~L. Fond and J.~Neville.
\newblock Randomization tests for distinguishing social influence and homophily
  effects.
\newblock {\em World Wide Web Conference}, 2010.

\bibitem{holme}
P.~Holme, S.~Park, B.~Kim, and C.~Edling.
\newblock Korean university life in a network perspective: Dynamics of a large
  affiliation network.
\newblock {\em Physica A: Statistical Mechanics and its Applications},
  373:821--830, 2007.

\bibitem{tlafond}
T.~LaFond, J.~Neville, and B.~Gallagher.
\newblock Anomaly detection in networks with changing trends.
\newblock {\em Outlier Detection and Description under Data Diversity at the
  International Conference on Knowledge Discovery and Data Mining}, 2014.

\bibitem{onnela}
J.~Onnela, J.~Saram{\"a}ki, J.~Kert{\'e}sz, and K.~Kaski.
\newblock Intensity and coherence of motifs in weighted complex networks.
\newblock {\em Physical Review}, 71(6):065103, 2005.

\bibitem{pfeiffer}
J.~Pfeiffer, T.~LaFond, S.~Moreno, and J.~Neville.
\newblock Fast generation of large scale social networks while incorporating
  transitive closures.
\newblock In {\em IEEE SocialCom}, 2012.

\bibitem{priebe}
C.~Priebe et~al.
\newblock Scan statistics on enron graphs.
\newblock {\em Computational \& Mathematical Organization Theory}, 2005.

\bibitem{saramaki}
J.~Saramaki et~al.
\newblock Generalizations of the clustering coefficient to weighted complex
  networks.
\newblock {\em Physical Review}, 2007.

\end{thebibliography}



\end{document}